\documentclass[acmsmall,natbib,screen=true]{acmart}

\PassOptionsToPackage{dvipsnames}{xcolor}

\usepackage{dsfont}
\usepackage{acronym}
\usepackage{amsmath}
\usepackage[noend]{algorithmic}
\usepackage{algorithm}
\usepackage{bbm}
\usepackage{beramono}
\usepackage{bm}
\usepackage{booktabs}
\usepackage[skip=0pt]{caption}
\usepackage{cleveref}
\usepackage[T1]{fontenc}
\usepackage{graphicx}
\usepackage{hyperref}
\usepackage{listings}
\usepackage{multirow}
\usepackage{natbib}
\usepackage{subcaption}
\usepackage{tabularx}
\usepackage[dvipsnames]{xcolor}
\usepackage{enumerate}
\usepackage[inline]{enumitem}
\usepackage{numprint}
\usepackage[normalem]{ulem}
\usepackage{mleftright}
\usepackage{balance}
\usepackage[all]{nowidow}
\usepackage{amsthm}

\acmJournal{TOIS}
\acmYear{2023}
\acmVolume{0}
\acmNumber{0}
\acmArticle{0}
\acmMonth{0}
\acmPrice{}
\acmDOI{10.1145/3569453}

\setcopyright{CC}



\definecolor{rj}{RGB}{0, 150, 0}
\definecolor{mdr}{RGB}{200, 0, 0}
\definecolor{ho}{RGB}{0, 50, 150}

\acrodef{IR}{information retrieval}
\acrodef{LTR}{learning-to-rank}
\acrodef{ARP}{average relevance position}
\acrodef{DCG}{discounted cumulative gain}
\acrodef{EM}{expectation-maximization}
\acrodef{CTR}{click-through-rate}
\acrodef{RCTR}{relevant-click-through-rate}
\acrodef{COLTR}{counterfactual online learning to rank}
\acrodef{PDGD}{pairwise differentiable gradient descent}
\acrodef{NRCTR}{normalized RCTR}
\acrodef{NDCG}{normalized DCG}

\acrodef{ECP}{expected number of clicked and preferred items}

\acrodef{PL}{Plackett-Luce}

\acrodef{IPS}{inverse-propensity-scoring}
\acrodef{DR}{doubly-robust}
\acrodef{DM}{direct-method}
\acrodef{RPS}{ratio-propensity-scoring}
\acrodef{CV}{covariate}

\newcommand{\smid}{\,|\,}

\allowdisplaybreaks

\author{Harrie Oosterhuis}
\affiliation{%
	\institution{Radboud University}
	\city{Nijmegen}
	\country{The Netherlands}
}
\email{harrie.oosterhuis@ru.nl}

\title[Doubly-Robust Estimation for Correcting Position-Bias in Click Feedback for Unbiased Learning to Rank]{Doubly-Robust Estimation for Correcting Position-Bias\\ in Click Feedback for Unbiased Learning to Rank}

\allowdisplaybreaks

\begin{document}

\begin{abstract}
Clicks on rankings suffer from position-bias: generally items on lower ranks are less likely to be examined -- and thus clicked -- by users, in spite of their actual preferences between items.
The prevalent approach to unbiased click-based \ac{LTR} is based on counterfactual \ac{IPS} estimation.
In contrast with general reinforcement learning, counterfactual \ac{DR} estimation has not been applied to click-based \ac{LTR} in previous literature.

In this paper, we introduce a novel \ac{DR} estimator that is the first \ac{DR} approach specifically designed for position-bias.
The difficulty with position-bias is that the treatment -- user examination -- is not directly observable in click data.
As a solution, our estimator uses the expected treatment per rank, instead of the actual treatment that existing \ac{DR} estimators use.
Our novel \ac{DR} estimator has more robust unbiasedness conditions than the existing \ac{IPS} approach, and in addition, provides enormous decreases in variance: our experimental results indicate it requires several orders of magnitude fewer datapoints to converge at optimal performance.
For the unbiased \ac{LTR} field, our \ac{DR} estimator contributes both increases in state-of-the-art performance and the most robust theoretical guarantees of all known \ac{LTR} estimators.
\end{abstract}

\maketitle

\acresetall

\section{Introduction}

The basis of recommender systems and search engines are ranking models that aim to provide users with rankings that meet their preferences or help in their search task~\citep{liu2009learning}.
The performance of a ranking model is vitally important to the quality of the user experience with a search or recommendation system.
Accordingly, the field of \acf{LTR} concerns methods that optimize ranking models~\citep{liu2009learning}; click-based \ac{LTR} uses logged user interactions to supervise its optimization~\citep{joachims2002optimizing}.
However, clicks are biased indicators of user preference~\citep{joachims2017accurately, radlinski2008does} because there are many factors beside user preference that influence click behavior.
Most importantly, the rank at which an item is displayed is known to have an enormous effect on whether it will be clicked or not~\citep{craswell2008experimental}.
Generally, users do not consider all the items that are presented in a ranking, and instead, are more likely to examine items at the top of the ranking.
Consequently, lower-ranked items are less likely to be clicked by users, regardless of whether users actually prefer these items~\citep{joachims2017unbiased}.
Therefore, clicks can be more reflective of where an item was displayed during the gathering of data than whether  users prefer it.
This form of bias is referred to as \emph{position-bias}~\citep{wang2018position, craswell2008experimental, agarwal2019estimating}; it is extremely prevalent in user clicks on rankings.
Correspondingly, this has lead to the introduction of unbiased \ac{LTR}: methods for click-based optimization that mitigate the effects of position-bias.
\citet{wang2016learning} and \citet{joachims2017unbiased} proposed using \ac{IPS} estimators~\citep{horvitz1952generalization} to correct for position-bias.
By treating the examination probabilities as propensities, \ac{IPS} can estimate ranking metrics unbiasedly w.r.t.\ position-bias.
This has lead to the inception of the unbiased \ac{LTR} field, in which \ac{IPS} estimation has remained the basis for most state-of-the-art methods~\citep{oosterhuis2020topkrankings, oosterhuis2021onlinecounterltr, agarwal2019counterfactual, agarwal2019addressing, vardasbi2020trust}.
However, variance is a large issue with \ac{IPS}-based LTR and remains an obstacle for its adoption in real-world applications~\citep{oosterhuis2021onlinecounterltr}.

Outside of \ac{LTR}, \ac{DR} estimators are a widely used alternative for \ac{IPS} estimation~\citep{kang2007demystifying, robins1994estimation}, for instance, for optimization in contextual bandit problems~\citep{dudik2014doubly}.
The \ac{DR} estimator combines an \ac{IPS} estimate with the predictions of a regression model, such that it is unbiased when per treatment either: the estimated propensity or the regression model is accurate~\citep{kang2007demystifying}.
Additionally, the \ac{DR} estimator can also bring large decreases in variance if the regression model is adequately accurate~\citep{dudik2014doubly}.
Unfortunately, existing \ac{DR} estimators are not directly applicable to the unbiased \ac{LTR} problem, since the treatment variable -- that indicates whether an item was examined or not -- cannot be observed in the data.
This is the characteristic problem of position-biased clicks: when an item is not clicked, we cannot determine whether the user chose not to interact or the user did not examine it in the first place.
Consequently, the unbiased \ac{LTR} field has not progressed beyond the usage of \ac{IPS} estimation.
 
Our main contribution is the first \ac{DR} estimator that is specifically designed to perform unbiased \ac{LTR} from position-biased click data.
Instead of using the actual treatment: user examination, which is unobservable in click data, our novel estimator uses the expectation of treatment per rank to construct a covariate instead.
Similar to \ac{DR} estimators for other tasks, it combines the preference predictions of a regression model with \ac{IPS} estimation.
Unlike \ac{IPS} estimators which are only unbiased with accurate knowledge of the logging policy, our \ac{DR} estimator requires \emph{either} the correct logging policy propensity \emph{or} an accurate regression estimate per item.
As a result, our \ac{DR} estimator has less strict requirements for unbiasedness than \ac{IPS} estimation. 
Moreover, it can also provide enormous decreases in variance compared to \ac{IPS}: our experimental results indicate that the \ac{DR} estimator requires several orders of magnitude fewer datapoints to converge at optimal performance.
In all tested top-5 ranking scenarios, it needs less than $10^6$ logged interactions to reach the performance that \ac{IPS} reaches at $10^9$ logged interactions.
Additionally, when compared to other state-of-the-art methods \ac{DR} also provides better performance across all tested scenarios.
Therefore, the introduction of \ac{DR} estimation for unbiased \ac{LTR} contributes
 the first unbiased \ac{LTR} estimator that is provenly more robust than \ac{IPS}, while also improving state-of-the-art performance on benchmark unbiased \ac{LTR} datasets.

\subsection{Structure of the Paper}
The remained of this work is structured as follows:

Section~\ref{sec:relatedwork} discusses relevant existing work on click-based and unbiased \ac{LTR} and earlier methods that have applied \ac{DR} estimation to clicks.
Then Section~\ref{sec:problemdefinition} explains our \ac{LTR} problem setting by describing our assumptions about user behavior, how click data is logged and our \ac{LTR} goal.
Our background section is divided in two parts:
Section~\ref{background:generic} provides background on counterfactual estimation methods in general, not specific to \ac{LTR}.
These generic methods are introduced so that we can later contrast them with \ac{LTR} specific methods and illustrate the adaptations that are required to deal with position-bias specifically.
Subsequently, Section~\ref{sec:background:ltr} describes the existing \ac{IPS} method that is specifically designed for \ac{LTR} and position-bias.
Furthermore, it discusses existing regression loss estimation and earlier \ac{DR} estimation methods that have been applied to clicks.

Our novel methods, the novel \ac{DR} estimator designed to correct position-bias and a novel cross-entropy loss estimator, are introduced in Section~\ref{sec:method}.
Then Section~\ref{sec:experimentalsetup} details the experiments that were performed to evaluate our novel method, the results of these experiments are presented and discussed in Section~\ref{sec:results}.
Finally, Section~\ref{sec:conclusion} provides a conclusion of this work, followed by the appendices that provide extended proofs for the theoretical claims of the work.

\section{Related Work}
\label{sec:relatedwork}

This section provides a brief overview of the existing literature in the unbiased \ac{LTR} field, in addition, relevant work on dealing with position-biased clicks and existing methods that apply \ac{DR} estimation to click data outside of the \ac{LTR} field are also discussed.

Optimizing ranking models based on click-data is a well-established concept~\citep{joachims2002optimizing}.
Early methods took an online dueling-bandit approach~\citep{yue2009interactively, schuth2016mgd} and later an online pairwise approach~\citep{oosterhuis2018differentiable}.
The first \ac{LTR} method with theoretical guarantees of unbiasedness was introduced by \citet{wang2016learning} and then generalized by \citet{joachims2017unbiased}.
They assume the probability that a user examines an item only depends on the rank at which it is displayed and that clicks only occur on examined items~\citep{wang2018position, craswell2008experimental}.
Then using counterfactual \ac{IPS} estimation they correct for the selection bias imposed by the examination probabilities.
The introduction of this approach launched the unbiased \ac{LTR} field:
\citet{agarwal2019counterfactual} expanded the approach for optimizing neural networks.
\citet{oosterhuis2020topkrankings} generalized the approach to also correct for the \emph{item-selection-bias} in top-$k$ ranking settings by basing the propensities on a stochastic logging policy.
\citet{agarwal2019addressing} showed that user behavior shows an additional \emph{trust-bias}: increased incorrect clicks at higher ranks~\citep{joachims2017accurately}, \citet{vardasbi2020trust} extended the \ac{IPS} estimator with affine corrections to correct for this trust-bias.
\citet{singh2019policy} use \ac{IPS} to optimize for a fair distribution of exposure over items.
\citet{jagerman2020safe} consider safe model deployments by bounding model performance. %
\citet{oosterhuis2021onlinecounterltr} introduced a generalization of the top-$k$ and affine estimators that considers the possibility that the logging policy is updated during the gathering of data.
\citet{wang2021non} proposed a \ac{RPS} estimator that weights pairs of clicked and non-clicked items by their ratio between the propensities.
\ac{RPS} is an extension of \ac{IPS} that introduces bias but also reduces variance.

In contrast with the rest of the field, recent work has proposed some methods that do not rely on \ac{IPS} to the field.
\citet{zhuang2021cross} and \citet{yan2022twotowers} fit predictive models to observed click data that explicitly factorizes relevance and bias factors, while they report promising real-world results, their methods do not provide strong theoretical guarantees w.r.t.\ unbiasedness.
\citet{ovaisi2020correcting} propose an adaptation of Heckman’s two-stage method.
Besides these exceptions and to the best of our knowledge,
all methods in the unbiased \ac{LTR} field are based on \ac{IPS}.

Interestingly, methods for dealing with position-biased clicks have also been developed outside of the unbiased \ac{LTR} field.
\citet{Komiyama2015} and \citet{lagree2016multiple} propose bandit algorithms that use similar \ac{IPS} estimators  for serving ads in multiple on-screen positions at once.
Furthermore, \citet{li2018offline} also propose \ac{IPS} estimators for the unbiased click-based evaluation of ranking models.
This further evidences the widespread usage of \ac{IPS} estimation for correcting position-biased clicks.

Nevertheless, there is previous work that has applied \ac{DR} estimators to position-biased clicks:
\citet{saito2020doubly} proposed a \ac{DR} estimator for post-click conversions that estimates how users treat an item after clicking it.
\citet{kiyohara2022doubly} designed a \ac{DR} estimator for policy-evaluation under cascading click behavior~\citep{vardasbi2020cascade}.
Lastly, \citet{yuan2020unbiased} introduced a \ac{DR} estimator for \ac{CTR} prediction on advertisements placements.
Section~\ref{sec:doublyrobust:prev} will discuss these methods in a bit more depth, and how they differ from the prevalent \ac{IPS} approach to unbiased \ac{LTR} and our proposed \ac{DR} estimator.
Important for our current discussion is that each of these methods tackles a different problem setting than the \ac{LTR} problem setting in this work.
Moreover, the latter two \ac{DR} estimators use corrections based on action propensities, similar to generic counterfactual estimation, and in stark contrast with the examination propensities of unbiased \ac{LTR}.
Finally, these works focus on policy evaluation instead of \ac{LTR}, thus it appears that the effectiveness of \ac{DR} for ranking model optimization is currently still unknown to the field.

\section{Problem Definition}
\label{sec:problemdefinition}
This section describes the assumptions underlying the theory of this paper and details our exact problem setting.
Specifically, it introduces the assumed mathematical model by which clicks are generated, the metric we aim to optimize, and notation to describe logged data.

\subsection{User Behavior Assumptions}
This paper assumes that the probability of a click depends on the user preference w.r.t. item $d$ and the position (also called rank) $k \in \{1,2,\ldots,K\}$ at which $d$ is displayed.
Let $R_d = P(R = 1 \mid d)$ be the probability that a user prefers item $d$ and for each $k$ let $\alpha_k \in [0,1]$ and $\beta_k \in [0,1]$ such that $\alpha_k + \beta_k \in [0, 1]$, the probability of $d$ receiving a click $C \in \{0, 1\}$ when displayed at position $k$ is:
\begin{align}
P(C = 1 \mid d, k) &= \alpha_k P(R = 1 \mid d) + \beta_k = \alpha_k R_d + \beta_k. \label{eq:clickprob}
\end{align}
This assumption has been derived~\citep{vardasbi2020trust} from a more interpretable user-model proposed by \citet{agarwal2019addressing}.
Their model is based on the examination assumption~\citep{richardson2007predicting}: users first examine an item before they interact with it, i.e.\  let $O \in \{0,1\}$ indicate examination then $O = 0 \rightarrow C = 0$.
Additionally, they also incorporate the concept of trust-bias: users are more likely to click against their preferences on higher ranks because of their \emph{trust} in the ranking model. 
This can be modelled by having the probability of a click conditioned on examination vary over $k$:
\begin{equation}
\begin{split}
\epsilon_k^+ &= P(C = 1 \mid R = 1, O = 1, k), \\
\epsilon_k^-  &= P(C = 1 \mid R = 0, O = 1, k).
\end{split}
\label{eq:conditionalclickprob}
\end{equation}
The proposed user model results in the following click probability: 
\begin{equation}
P(C = 1 \mid d, k)
= P(O = 1 \mid k)(\epsilon_k^+R_d  + \epsilon_k^-(1-R_d)),
 \label{eq:longclickprob}
\end{equation}
by comparing Eq.~\ref{eq:clickprob} and~\ref{eq:longclickprob} we see that:
\begin{equation}
\alpha_k = P(O = 1 \mid k)(\epsilon_k^+ -\epsilon_k^-), \qquad
\beta_k = P(O = 1 \mid k)\epsilon_k^-.
\label{eq:alphabeta}
\end{equation}
\citet{agarwal2019addressing} provide empirical results that suggest this user-model is more accurate than the previous model that ignores the trust-bias effect: $\forall k, \;\; \beta_k = 0$~\citep{wang2016learning, joachims2017unbiased, wang2018position}.
Since the assumption in Eq.~\ref{eq:clickprob} is true in both models, our work is applicable to  most settings in earlier unbiased \ac{LTR} work~\citep{wang2016learning, joachims2017unbiased, agarwal2019addressing, vardasbi2020trust, oosterhuis2020topkrankings, singh2019policy, jagerman2020safe, oosterhuis2021onlinecounterltr, ovaisi2020correcting, wang2021non}.

\subsection{Definition of the LTR Goal}
\label{sec:ltrgoal}

The goal of our ranking task is to maximize the probability that a user will click on something they prefer.
Let $\pi$ be the ranking policy to optimize, with $\pi(k \mid d)$ indicating the probability that $\pi$ places $d$ at position $k$ and let $y$ indicate a ranking of size $K$: $y = [ y_1, y_2 , \ldots, y_K]$, lastly, let $D = \{d_1, d_2, \ldots, d_K\}$ be the collection of items to be ranked.
Most ranking metrics are a weighted sum of item relevances, where the weights $\omega_k$ depend on the item positions:
\begin{equation}
\mathcal{R}(\pi)
= \mathbb{E}_{y \sim \pi} \mleft[ \sum^K_{k=1} \omega_k R_{y_k} \mright]
= \sum_{d \in D} R_d \sum^K_{k=1} \pi(k \mid d) \omega_k.
\end{equation}
\Ac{DCG}~\citep{jarvelin2002cumulated} is a very traditional ranking metric: $\omega_{k}^\text{DCG} = \log_2(k+1)^{-1}$, however, \ac{DCG} has no clear interpretation in our assumed user model.
In contrast, we argue that our metric should actually be motivated by our user behavior assumptions, accordingly, this work will use the weights: $\omega_{k} = (\alpha_k + \beta_k)$.
This choice results in an easily interpreted metric; for brevity of notation, we first introduce the expected position weight per item:
\begin{equation}
\omega_{d}
= \mathbb{E}_{y \sim \pi}\mleft[ \omega_{k(d)} \mright]
= \mathbb{E}_{y \sim \pi}\mleft[ \alpha_{k(d)} + \beta_{k(d)}  \mright]
= \sum_{k=1}^K \pi(k \mid d) (\alpha_k + \beta_k),
\label{eq:trueomega}
\end{equation}
using Eq.~\ref{eq:clickprob}, our ranking metric can then be formulated as:
\begin{equation}
\begin{split}
\mathcal{R}(\pi)
&= \sum_{d \in D} \omega_{d} R_d
=  \sum_{d \in D} P(R = 1 \,|\ d) \sum_{k=1}^K \pi(k \,|\, d) (\alpha_k + \beta_k)
\\&
=  \sum_{d \in D} \mathbb{E}_{y \sim \pi} \mleft[ P(C = 1, R = 1 \,|\ d, k) \mright]
=  \sum_{d \in D} P(C = 1, R = 1 \,|\ d)
.
\end{split}
\label{eq:reward}
\end{equation}
This formulation clearly reveals that our chosen metric directly corresponds to the expected number of items that are both clicked and preferred in our assumed user behavior model.
Hence, we will call this metric: the number of expected clicks on preferred items, abbreviated to ECP.

Given a ranking metric, the \ac{LTR} field provides several optimization methods to train ranking models.
A popular approach for deterministic ranking models is to optimize a differentiable lower bound on the ranking metric~\citep{wang2018lambdaloss, burges2010ranknet}.
For probabilistic ranking models, a simple sampled-approximation of the policy-gradient can be applied~\citep{williams1992simple}.
Alternatively, recent methods provide more efficient approximation methods specifically designed for the \ac{LTR} problem~\citep{oosterhuis2021computationally, ustimenko2020stochasticrank}.

Finally, we note that our main contributions work with any choice of weights $\omega_k$ and are therefore equally applicable to most traditional \ac{LTR} metrics.
Furthermore, for the sake of simplicity and brevity and without loss of generalization, our notation and our defined goal are limited to a single query or ranking context.
We refer to previous work by \citet{joachims2017unbiased} and \citet{oosterhuis2021onlinecounterltr} as examples of how straightforward it is to expand these to expectations over multiple queries or contexts.

\subsection{Historically Logged Click Data}
Lastly, in the unbiased \ac{LTR} setting, optimization is performed on historically logged click data.
This means a logging policy $\pi_0$ was used to show rankings to users in order to collect the resulting clicks.
We will assume the data contains $N$ rankings that were sampled from $\pi_0$ and displayed to users, where $y_i$ is the $i$th ranking and $c_i(d) \in \{0,1\}$ indicates whether $d$ was clicked when $y_i$ was displayed.
The bias parameters $\alpha_k$ and $\beta_k$ have to be estimated from user behavior, we will use $\hat{\alpha}_k$ and $\hat{\beta}_k$ to denote the estimated values.
To keep our notation brief, we will use the following to denote that all the estimated bias parameters are accurate:
\begin{equation}
\begin{split}
\hat{\alpha} = \alpha &\longleftrightarrow ( \forall k \in \{1,2,\ldots, K\}, \; \hat{\alpha}_k = \alpha_k ),
\\
\hat{\beta} = \beta &\longleftrightarrow ( \forall k \in \{1,2,\ldots, K\}, \; \hat{\beta}_k = \beta_k ).
\end{split}
\end{equation}
We will investigate both the scenarios where the $\alpha$ and $\beta$ bias parameters are known from previous experiments~\citep{agarwal2019addressing, fang2019intervention, wang2018position} and where they still have to be estimated.
Similarly, the exact distribution of the logging policy $\pi_0$ may also have to be estimated, the following denotes that the estimated distribution $\hat{\pi}_0$ for item $d$ is accurate:
\begin{equation}
\hat{\pi}_0(d) = \pi_0(d) \longleftrightarrow \big( \forall k \in \{1,2,\ldots, K\}, \; \hat{\pi}_0(k \smid d) = \pi_0(k \smid d) \big).
\end{equation}
To summarize, our goal is to maximize $\mathcal{R}(\pi)$ based on click data gathered using $\pi_0$ from the position-biased click model in Eq.~\ref{eq:clickprob}.

\section{Background: Applying Generic Counterfactual Estimation to Click-Through-Rates}
\label{background:generic}

This section will give an overview of counterfactual estimation for generic reinforcement learning~\citep{dudik2014doubly, kang2007demystifying, sutton1998introduction}, its purpose is two-fold:
 \begin{enumerate*}[label=(\roman*)]
\item to illustrate why it is not effective for the \ac{LTR} problem; and
\item to contrast the generic estimators with the existing estimators for \ac{LTR} and our novel estimators.
\end{enumerate*}
While it appears that the field is aware that generic counterfactual estimation is not a practical solution to \ac{LTR}~\citep{saito2021counterfactual, kiyohara2022doubly, joachims2017unbiased}, to the best of our knowledge, previous published research has not gone in depth on the reasons for this ineffectiveness.

\subsection{The Generic Estimation Goal}
The first issue between the \ac{LTR} problem and generic counterfactual estimation is that the latter assumes that the rewards for the actions performed by the logging policy are directly observed, as is the case for standard reinforcement learning tasks~\citep{dudik2014doubly, sutton1998introduction}.
However, in the \ac{LTR} problem, clicks are observed instead of the relevances $R_d$ on which metrics are based.
Consequently, for this section, we will restrict ourselves to estimating \ac{CTR} instead, luckily this task is similar enough to the \ac{LTR} problem for our discussion.
Let $y[k]$ indicate the $k$th item in ranking $y$,
under our assumed click model, the \ac{CTR} of a policy $\pi$ is:
\begin{equation}
\text{CTR}(\pi) =
\sum_{y : \pi(y) > 0} 
\pi(y)\sum_{k=1}^K (\alpha_k R_{y[k]} + \beta_k).
\label{eq:gen:ctr}
\end{equation}
The goal in this section is thus to estimate $\text{CTR}(\pi)$ for any given policy $\pi$ using data collected by a logging policy $\pi_0$.
Importantly, to be effective at optimization, the estimation process should work for any possible $\pi$ in the model space.

\subsection{Generic Inverse-Propensity-Scoring Estimation}
Standard \acf{IPS} estimation corrects for the mismatch between $\pi$ and $\pi_0$ by reweighting the observed action inversely w.r.t.\ their estimated propensity: the probability that $\pi_0$ takes this observation~\citep{dudik2014doubly}.
In our case, an action is a ranking $y$ and its estimated propensity is $\hat{\pi}_0(y)$.
Applying standard \ac{IPS} estimation to the \ac{LTR} task results in the following generic \ac{IPS} estimator: 
\begin{equation}
\widehat{\mathcal{R}}_\text{G-IPS}(\pi) =
\frac{1}{N}\sum_{i=1}^N
\hspace{-0.05cm}
\underbrace{
\frac{\pi(y_i)}{\hat{\pi}_0(y_i)}
}_\text{IPS weight}
\hspace{-0.05cm}
\sum_{k=1}^K
\hspace{-0.12cm}
\underbrace{c_i(y_i[k])}_\text{observed click}
\hspace{-0.18cm}
.
\label{eq:gen:ips}
\end{equation}
The \ac{IPS} weight aims to correct for the difference in action probability between $\pi_0$ and $\pi$, e.g.\  if an action is underrepresented in the data because $\pi_0(y) < \pi(y)$ then \ac{IPS} will compensate by giving more weight to this action as: $\pi(y)/\hat{\pi}_0(y) > 1$.
To understand when this approach can provide unbiased estimation, we first look at its expected value:
\begin{equation}
\mathbb{E}_{c,y\sim\pi_0}\mleft[\widehat{\mathcal{R}}_\text{G-IPS}(\pi)\mright] =
\sum_{y : \pi_0(y) > 0} 
 \pi(y)
 \hspace{-1.85cm}
\underbrace{
\frac{\pi_0(y)}{\hat{\pi}_0(y)}
}_\text{ratio between estimated and real propensity}
\hspace{-1.85cm}
\sum_{k=1}^K (\alpha_k R_{y[k]} + \beta_k).
\end{equation}
The result is very similar to the formula for \ac{CTR} (Eq.~\ref{eq:gen:ips}) except for the ratio between the estimated and real propensity.
Clearly, for unbiasedness this ratio must be equal to one, i.e.\ the estimated propensity needs to be correct, additionally, each action that $\pi$ may take must have a positive propensity from $\pi_0$:
\begin{equation}
\Big( \forall y, \; \pi(y) > 0 \rightarrow
\hspace{-0.31cm}
\underbrace{
\big(\pi_0(y) > 0 \land \hat{\pi}_0(y) = \pi_0(y) \big)
}_\text{estimated propensity is correct and positive}
\hspace{-0.31cm}
\Big) \rightarrow \mathbb{E}_{c,y\sim\pi_0}\mleft[\widehat{\mathcal{R}}_\text{G-IPS}(\pi)\mright] = \text{CTR}(\pi).
\end{equation}
Thus unbiased estimation via generic \ac{IPS} is possible, however, the requirements are practically infeasible for any large ranking problem due to its enormous action space.
Importantly, for unbiasedness, $\pi_0$ has to give a non-zero probability to each action $\pi$ may take, but data is often collected without knowledge of what $\pi$ will be evaluated (or reused for many different policies).
Moreover, when using \ac{IPS} for unbiased optimization it should be unbiased for any possible $\pi$ in the policy space that optimization is performed over.
In practice, this means that the number of possible rankings is enormous and each should get a non-zero probability by $\pi_0$, which translates to $\pi_0$ giving a positive probability to $K! \binom{|D|}{K}$ rankings.
This is quite undesirable since the user experience is likely to suffer under such a random policy, but moreover, it also brings serious variance problems.
In particular, when we consider the variance of the generic \ac{IPS} estimator, we see that small $\hat{\pi}_0(y)$ propensities increase variance by a massive degree:
\begin{equation}
\mathbb{V}\mleft[\widehat{\mathcal{R}}_\text{G-IPS}(\pi)\mright] =
\mathbb{E}_{y\sim\pi_0}
\bigg[
\hspace{-0.87cm}
\overbrace{
\frac{\pi(y)^2}{\hat{\pi}_0(y)^2}
}^\text{multiplier from IPS weight}
\hspace{-0.92cm}
\hspace{-0.55cm}
\underbrace{
\mathbb{V}
\bigg[\sum_{k=1}^K c(y[k])
\bigg]
}_\text{variance from clicks on ranking}
\hspace{-0.55cm}
\bigg].
\end{equation}
Thus, to summarize, to apply generic \ac{IPS} to \ac{LTR} unbiasedly, all rankings require a positive propensity, but due to the enormously large number of rankings this leads to extremely small propensity values which lead to enormous variance.
As a result, generic \ac{IPS} is not a practical solution for unbiased low-variance \ac{LTR} and has been widely avoided in practice~\citep{kiyohara2022doubly}.
Section~\ref{sec:background:ipsltr} will describe the \ac{IPS} approach specifically designed for \ac{LTR} that has much more feasible unbiasedness requirements.

\subsection{The Generic Direct Method}
Before we continue to \ac{LTR} specific methods, we will describe the \acf{DM} and generic \acf{DR} estimation, in order to compare them with our novel estimator later.
First, the direct method uses regression estimates from a regression model to estimate policy performance~\citep{dudik2014doubly}.
Let $\hat{R}_d$ indicate a regression estimate of $R_d$, the generic \ac{DM} estimator is then:
\begin{equation}
\widehat{\mathcal{R}}_\text{G-DM}(\pi) =
\sum_{y : \pi(y) > 0} 
\pi(y) \sum_{k=1}^K
\hspace{-0.15cm}
\underbrace{
\hat{\alpha}_k \hat{R}_{y[k]} + \hat{\beta}_k
}_\text{predicted click prob.}
\hspace{-0.25cm}
.
\label{eq:gen:dm}
\end{equation}
Clearly, \ac{DM} is unbiased when the regression estimates are correct for all rankings that $\pi$ may show:
\begin{equation}
\Big( \forall y, \; \pi(y) > 0 \rightarrow
\underbrace{
\mleft( \hat{\alpha} = \alpha \land \hat{\beta} = \beta \land  \forall k, \; \hat{R}_{y[k]} = R_{y[k]} \mright)
}_\text{predicted click probabilities are correct}
\Big)
\longrightarrow \mathbb{E}_{c,y \sim \pi_0}\mleft[\widehat{\mathcal{R}}_\text{G-DM}(\pi)\mright] = \text{CTR}(\pi)
.
\end{equation}
However, this is not a very useful requirement in practice, since solving the regression problem is arguably just as difficult as the subsequent \ac{LTR} problem.
Nonetheless, its unbiasedness requirements are very different than those for \ac{IPS}, a benefit of \ac{DR} is that it combines both requirements advantageously.

\subsection{Generic Doubly Robust Estimation}
Besides \ac{DM} and \ac{IPS}, \ac{DR} estimation makes use of a \ac{CV}, that aims to have a large covariance with \ac{IPS} estimation but the same expected value as \ac{DM}.
We will call this the generic \ac{CV}:
\begin{equation}
\widehat{\mathcal{R}}_\text{G-CV}(\pi) =
\frac{1}{N}\sum_{i=1}^N
\hspace{-0.05cm}
\underbrace{ 
\frac{\pi(y_i)}{\hat{\pi}_0(y_i)}
}_\text{IPS weight}
\hspace{-0.03cm}
\sum_{k=1}^K
\hspace{-0.1cm}
\underbrace{ 
\hat{\alpha}_k\hat{R}_{y[k]} + \hat{\beta}_k
}_\text{predicted click prob.}
\hspace{-0.25cm}
.
\end{equation}
\Ac{CV} makes use of the actions sampled by $\pi_0$, allowing it to have a high covariance with \ac{IPS}, but uses predicted click probabilities instead of using the actual clicks, enabling it to have the same expected value as \ac{DM}.
Importantly, when \ac{IPS} is unbiased, i.e.\ when the estimated propensities are correct, \ac{CV} is also an unbiased estimate of \ac{DM}:
\begin{equation}
\Big( \forall y, \; \pi(y) > 0 \rightarrow
\hspace{-0.3cm}
\underbrace{
\big(\pi_0(y) > 0 \land \hat{\pi}_0(y) = \pi_0(y) \big)
}_\text{estimated propensity is correct and positive}
\hspace{-0.3cm}
\Big) \rightarrow 
\mathbb{E}_{y \sim \pi_0}\mleft[
\widehat{\mathcal{R}}_\text{G-CV}(\pi)
\mright]
=
\widehat{\mathcal{R}}_\text{G-DM}(\pi).
\end{equation}
The \ac{DR} estimator is a combination of the above three estimators~\citep{dudik2014doubly, kang2007demystifying, robins1994estimation}:
\begin{equation}
\begin{split}
\widehat{\mathcal{R}}_\text{G-DR}(\pi)
&=
\widehat{\mathcal{R}}_\text{G-DM}(\pi)
+
\widehat{\mathcal{R}}_\text{G-IPS}(\pi)
-
\widehat{\mathcal{R}}_\text{G-CV}(\pi)
\\
&=
\hspace{-0.7cm}
\underbrace{
\widehat{\mathcal{R}}_\text{G-DM}(\pi)
}_\text{predicted model performance}
\hspace{-0.7cm}
+
\frac{1}{N}\sum_{i=1}^N 
\underbrace{
\frac{\pi(y_i)}{\pi_0(y_i)}
}_\text{IPS weight}
\sum_{k=1}^K 
\hspace{-0.9cm}
\overbrace{\Big(c_i(y_i[k]) - \hat{\alpha}_k\hat{R}_{y[k]} - \hat{\beta}_k\Big)}^{\text{diff. between observed click and predicted click prob.}} 
\hspace{-1.05cm}
.
\end{split}
\label{eq:gen:dr}
\end{equation}
The \ac{DR} starts with regression-based estimate of \ac{DM} then for each logged ranking it adds the difference between the observed clicks and the predicted clicks.
In other words, \ac{DR} uses \ac{DM} as a baseline and adds an \ac{IPS} estimate of the difference between the regression model and the actual clicks.

The first advantage of \ac{DR} estimation are its unbiasedness requirements.
It has the following bias w.r.t.\ \ac{CTR}:
\begin{align}
\mathbb{E}\mleft[ \widehat{\mathcal{R}}_\text{G-DR}(\pi) \mright] - CTR(\pi)
&=
\hspace{-0.3cm}
\sum_{y \,:\, \pi(y) > 0} 
\pi(y)\sum_{k=1}^K \mleft(\frac{\pi_0(y)}{\hat{\pi}_0(y)} - 1 \mright)\mleft( \alpha_kR_{y[k]} + \beta_k\mright)
+ \mleft(1 - \frac{\pi_0(y)}{\hat{\pi}_0(y)}\mright)\mleft( \hat{\alpha}_k\hat{R}_{y[k]} + \hat{\beta}_k\mright)
\nonumber \\ &=
\hspace{-0.3cm}
\sum_{y \,:\,  \pi(y) > 0} 
\pi(y)
\hspace{-0.1cm}
\underbrace{
\mleft(\frac{\pi_0(y)}{\hat{\pi}_0(y)} - 1 \mright)
}_\text{error in propensity}
\hspace{-0.1cm}
\sum_{k=1}^K
\underbrace{
\mleft( \alpha_kR_{y[k]} + \beta_k - \hat{\alpha}_k\hat{R}_{y[k]} - \hat{\beta}_k\mright)
}_\text{error in click prob. prediction}
\hspace{-0.05cm}
.
\label{eq:gen:biasdr}
\end{align}
As we can see the bias of \ac{DR} is a summation over rankings, where per ranking the error in propensity is multiplied with the error in predicted click probability.
Due to this product, only one of the two errors has to be zero  per ranking for the total bias to be zero.
In other words, \ac{DR} is unbiased when for each ranking either the propensity or the predicted click probabilities are correct:
\begin{equation}
\begin{split}
&\big( \forall y, \; \pi(y) > 0 \rightarrow
\big(
\overbrace{
\big(\pi_0(y) > 0 \land \hat{\pi}_0(y) = \pi_0(y) \big)
}^\text{propensity is correct and positive}
\lor
\overbrace{
\big( \hat{\alpha} = \alpha \land \hat{\beta} = \beta \land  \forall k, \; \hat{R}_{y[k]} = R_{y[k]} \big)
}^\text{predicted click probabilities are correct}
\big)
\big)
\\& \hspace{8cm}
\longrightarrow \mathbb{E}_{c,y \sim \pi_0}\mleft[\widehat{\mathcal{R}}_\text{G-DR}(\pi)\mright] = \text{CTR}(\pi).
\end{split}
\label{eq:gen:biasdrreq}
\end{equation}
As a result, if either \ac{DM} or \ac{IPS} is unbiased then \ac{DR} is unbiased, furthermore, it can potentially be unbiased when neither \ac{DM} or \ac{IPS} are.
In addition, to the beneficial unbiasedness requirements, \ac{DR} can also allow for reduced variance if there is a positive covariance between \ac{IPS} and \ac{CV}:
\begin{equation}
\mathds{V}\mleft[ \widehat{\mathcal{R}}_\text{G-DR}(\pi) \mright]
 =
 \mathds{V}\mleft[ \widehat{\mathcal{R}}_\text{G-IPS}(\pi) \mright]
 + 
 \mathds{V}\mleft[ \widehat{\mathcal{R}}_\text{G-CV}(\pi) \mright]
 - 
 2\mathbb{C}\text{ov}\mleft(\widehat{\mathcal{R}}_\text{G-IPS}(\pi),  \widehat{\mathcal{R}}_\text{G-CV}(\pi)\mright).
 \label{eq:gen:drvar}
\end{equation}
In practice, this means that somewhat accurate regression estimates can provide a decrease in variance over \ac{IPS}.
Nevertheless, this decrease is not enough to overcome the variance problems that stem from the small propensities that are involved in \ac{LTR}.
Consequently, state-of-the-art unbiased \ac{LTR} does not apply standard \ac{IPS} and \ac{DR}~\citep{oosterhuis2021onlinecounterltr, ai2020unbiased}.

\section{Background: Counterfactual Estimation for Unbiased Learning-to-Rank}
\label{sec:background:ltr}

While Section~\ref{background:generic} covers the standard counterfactual estimation techniques and why they are ineffective for the \ac{LTR} problem, this section describes the existing \ac{IPS} estimator that is specifically designed for \ac{LTR} and discuss previous work that has applied \ac{DR} estimation to click data.

\subsection{Inverse-Propensity-Scoring in Unbiased LTR}
\label{sec:background:ipsltr}

As discussed in Section~\ref{sec:relatedwork}, the main approach in state-of-the-art unbiased \ac{LTR} work is based on \acf{IPS} estimation~\citep{horvitz1952generalization}.
Under the affine click model in Eq.~\ref{eq:clickprob}, the propensities are not the probability of observation, as in the earliest unbiased \ac{LTR} work~\citep{wang2016learning, joachims2017unbiased}, but the expected correlation between the click probability and the user preference for an item $d$ under $\pi_0$~\citep{vardasbi2020trust, oosterhuis2021onlinecounterltr}:
\begin{equation}
\rho_{d} = \mathbb{E}_{y \sim \pi_0}\mleft[ \alpha_{k(d)} \mright] = \sum_{k=1}^K \pi_0(k \mid d)  \alpha_k,
\label{eq:truerho}
\end{equation}
where $\pi_0(k \mid d)$ indicates the probability that $\pi_0$ places $d$ at position $k$.
Because $\alpha$ are $\beta$ may be unknown, we use the following estimated values:
\begin{equation}
\begin{split}
\hat{\rho}_d &= \max\mleft(\mathbb{E}_{y \sim \pi_0}\mleft[ \hat{\alpha}_{k(d)} \mright],\, \tau\mright)
= \max\mleft(\sum_{k=1}^K \hat{\pi}_0( k \mid d) \hat{\alpha}_k,\, \tau\mright),
\\
\hat{\omega}_d &=  \sum_{k=1}^K \pi( k \mid d) \mleft( \hat{\alpha}_k + \hat{\beta}_k \mright),
\end{split}
\label{eq:estimatedvalues}
\end{equation}
where the clipping parameter $\tau \in (0,1]$ prevents small $\hat{\rho}_d$ values and is applied for variance reduction~\citep{joachims2017unbiased, strehl2010logged}.
The state-of-the-art \ac{IPS} estimator introduced by \citet{oosterhuis2021onlinecounterltr} 
first corrects each (non-)click with $\hat{\beta}_{k_i(d)}$ -- the $\hat{\beta}$ value for the position where the click took place -- to correct for clicks in spite of preference (trust-bias), and then inversely weights the result by $\hat{\rho}_{d}$ to correct for the correlation between the user preference and the click probability (position-bias and item-selection-bias~\citep{oosterhuis2020topkrankings}):
\begin{equation} 
\widehat{\mathcal{R}}_\text{IPS}(\pi) = \frac{1}{N}\sum_{i=1}^N \sum_{d \in D}
\hspace{-0.3cm}
 \overbrace{
 \frac{\hat{\omega}_{d}}{\hat{\rho}_{d}}
 }^\text{IPS weight}
\hspace{-0.68cm}
 \underbrace{
 \mleft(c_i(d) - \hat{\beta}_{k_i(d)}\mright)
  }_\text{click corrected for trust-bias}
\hspace{-0.38cm}
  ,
\label{eq:ips}
\end{equation}
where $k_i(d)$ indicates the position of $d$ in the $i$th ranking.

The main difference between the generic \ac{IPS} estimation (Eq.~\ref{eq:gen:ips}) and the \ac{LTR} \ac{IPS} estimator (Eq.~\ref{eq:ips}) is that the former bases its corrections on the differences between the action probabilities of $\pi$ and $\pi_0$, while the latter uses the correlation between clicks and relevance under $\pi_0$.
Thus while both use the behavior of the logging policy $\pi_0$,  the \ac{LTR} \ac{IPS} estimator uses the assumed click model of Eq.~\ref{eq:clickprob} as well.
While the reliance on the click model makes the estimator more effective, it also makes the estimator specifically designed for this click behavior.
The remainder of this section discusses that, due to this specific design, the theoretical properties of the \ac{LTR} \ac{IPS} estimator are clearly preferable over those of generic \ac{IPS} estimation applied to the \ac{LTR} problem.

To start, the \ac{IPS} estimator has the following bias:
\begin{equation}
\mathds{E}_{c,y \sim \pi_0}\mleft[ \hat{\mathcal{R}}_\text{\normalfont IPS}(\pi) \mright] - \mathcal{R}(\pi)
=
\sum_{d \in D} \frac{\hat{\omega}_{d}}{\hat{\rho}_{d}}\bigg(
\hspace{-0.23cm}
\underbrace{
\mleft(
\rho_{d} - \hat{\rho}_{d}\frac{\omega_{d}}{\hat{\omega}_{d}}
\mright)
}_\text{error from $\hat{\alpha}$, $\hat{\beta}$ and $\hat{\pi}$}
\hspace{-0.23cm}
R_d +
\underbrace{
\mathds{E}_{y \sim \pi_0}\mleft[ \beta_{k(d)} - \hat{\beta}_{k(d)} \mright]
}_\text{error from $\hat{\beta}$}
\bigg).
\end{equation}
Appendix~\ref{appendix:proofipsbias} provides a derivation of its bias, it also proves that the \ac{IPS} estimator is unbiased when
both the bias parameters and the logging policy distribution are correctly estimated and clipping has no effect~\citep{joachims2017unbiased, vardasbi2020trust}:
\begin{equation}
\hspace{-0.45cm}
\overbrace{
\big(\hat{\alpha} = \alpha \land \hat{\beta} = \beta
}^\text{pos. bias correctly estimated}
\hspace{-0.45cm}
 \land
\hspace{-0.15cm}
\underbrace{
\mleft(\forall d \in D, \;  \hat{\pi}_0(d) = \pi_0(d) \land \rho_d \geq \tau \mright) \big)
}_\text{$\hat{\pi}_0$ is correctly estimated and clipping has no effect}
\hspace{-0.15cm}
\longrightarrow
\mathbb{E}_{c,y \sim \pi_0}\big[ \widehat{\mathcal{R}}_\text{IPS}(\pi) \big] = \mathcal{R}(\pi).
\label{eq:ips:biascondition}
\end{equation}
Conversely,
\ac{IPS} is biased when clipping does have an effect, even if the bias parameters and logging policy distribution are correctly estimated.

Importantly, these unbiasedness requirements are much more feasible than those for the generic \ac{IPS} estimator (Eq.~\ref{eq:gen:biasdrreq});
where the generic \ac{IPS} estimator requires each ranking to have a positive probability ($\forall y, \; \pi_0(y) > 0$) of being displayed during logging, the \ac{LTR} \ac{IPS} estimator requires each item to have a propensity greater than $\tau$ ($\forall d, \; \rho_d \geq \tau$).
In other words, the correlation between clicks and relevances should be greater than $\tau$, this is even feasible under a deterministic logging policy that always displays the same single ranking if all items are displayed at once~\citep{joachims2017unbiased, oosterhuis2020topkrankings}.
However, the \ac{IPS} estimator does need accurate estimate of the bias parameters $\hat{\alpha}$ and $\hat{\beta}$ in addition to an accurate estimate of the logging policy $\hat{\pi}_0$, but previous work indicates that this is actually doable in practice~\citep{joachims2017unbiased, wang2016learning, wang2018position, agarwal2019estimating}.
In summary, compared to generic \ac{IPS} estimation, the \ac{IPS} estimator for \ac{LTR} has replaced infeasible requirements on the logging policy with attainable requirements on bias estimation.

The variance of the \ac{IPS} estimator can be decomposed in the following parts:
\begin{equation}
\mathds{V}\big[ \widehat{\mathcal{R}}_\text{IPS}(\pi) \big]
 = \frac{1}{N}\sum_{d \in D}
 \hspace{-1.1cm}
 \overbrace{
 \frac{\hat{\omega}_{d}^2}{\hat{\rho}_{d}^2}
 }^\text{multiplier from IPS weight}
 \hspace{-1.1cm}
\big(
 \hspace{-0.15cm}
\underbrace{
\mathds{V}\big[ c(d) \big]
}_\text{var. from click}
 \hspace{-0.05cm}
 +
 \hspace{-0.8cm}
\overbrace{
 \mathds{V}\big[ \hat{\beta}_{k(d)} \big]
}^\text{var. from trust-bias correction}
 \hspace{-0.8cm}
  -
 \hspace{-0.42cm}
\underbrace{
  2 
  \mathds{C}\text{ov}\big[c(d), \hat{\beta}_{k(d)}\big]
}_\text{cov. between click and correction}
 \hspace{-0.5cm}
\big).
\label{eq:ipsvariance}
\end{equation}
We see how clipping prevents extremely large values for the variance multiplier from the \ac{IPS} weight by preventing small $\hat{\rho}_{d}$ values and can thereby greatly reduce variance~\citep{joachims2017unbiased, strehl2010logged}.
Importantly, the reduction in variance is often much greater than the increase in bias, making this an attractive trade-off that has been widely adopted by the unbiased \ac{LTR} field~\citep{agarwal2019counterfactual, oosterhuis2020topkrankings}.
There is currently no known method for variance reduction in IPS-based position-bias correction that does not introduce bias, and thus, in practice unbiased LTR methods are actually often applied in a biased manner\footnote{\citet{oosterhuis2021onlinecounterltr} apply unbiased LTR in an online fashion to reduce variance, but this solution does not apply to our problem setting.}.

\subsection{Existing Cross-Entropy Loss Estimation}
As discussed, \ac{DM} requires an accurate regression model to be unbiased or effective.
In the ideal situation, one may optimize a regression model for estimating relevance using the cross-entropy loss:
\begin{equation}
\mathcal{L}(\hat{R}) = - \sum_{d \in D} R_d \log(\hat{R}_d) + (1 - R_d)\log(1 -\hat{R}_d).
\label{eq:trueCEloss}
\end{equation}
However, this loss cannot be computed from the click-data since $R_d$ cannot be observed.
Luckily, \citet{bekker2019beyond} have introduced an estimator that can be applied to position-biased clicks:
\begin{equation}
\widehat{\mathcal{L}}'(\hat{R}) = - \frac{1}{N} \sum_{i = 1}^N \sum_{d \in D} \mleft( \frac{c_i(d)}{\hat{\rho}_d}\log(\hat{R}_d)
 + \mleft(1 - \frac{c_i(d)}{\hat{\rho}_d} \mright)\log(1 -\hat{R}_d)\mright).
 \label{eq:prevCE}
\end{equation}
\citet{saito2020unbiased} showed that this estimator is effective for recommendation tasks on position-biased click data.
$\widehat{\mathcal{L}}'$ is unbiased~\citep{bekker2019beyond, saito2020unbiased} when there is no trust-bias, propensities are accurate and clipping has no effect: 
\begin{equation}
\underbrace{
\big(
(\forall k, \,\beta_k = 0 )
}_\text{no trust-bias}
\land
\hspace{-1.05cm}
\overbrace{
\hat{\alpha} = \alpha
}^\text{pos. bias correctly estimated}
\hspace{-1.05cm}
\land
\hspace{-0.15cm}
\underbrace{
\mleft(\forall d \in D, \;  \hat{\pi}_0(d) = \pi_0(d) \land \rho_d \geq \tau \mright) \big)
}_\text{$\hat{\pi}$ is correctly estimated and clipping has no effect}
\hspace{-0.15cm}
\longrightarrow \mathbb{E}_{c,y\sim \pi_0}\big[\widehat{\mathcal{L}}'(\hat{R})\big] = \mathcal{L}(\hat{R}).
\end{equation}
In Section~\ref{sec:novelcrossentropy}, we propose a novel variation on this estimator that can also correct for trust-bias and that treats non-displayed items in a more intuitive way.

\subsection{Existing Doubly-Robust Estimation for Logged Click Data}
\label{sec:doublyrobust:prev}

Our discussion of generic counterfactual estimation in Section~\ref{background:generic} concluded with \ac{DR} estimation and the advantageous properties it can have over \ac{IPS} and \ac{DM}~\citep{kang2007demystifying, robins1994estimation}.
Given that the current state-of-the-art in \ac{LTR} is based on \ac{IPS}, it seems very promising to apply \ac{DR} to unbiased \ac{LTR}.
Unfortunately, at first glance it seems \ac{DR} is inapplicable to the \ac{LTR} problem, since treatment is the examination of the user and we cannot directly observe whether a user has examined a non-clicked item or not.
Because \ac{DR} estimation balances \ac{IPS} and regression estimates unbiasedly using the knowledge of treatment in the data, e.g.\ which actions where taken~\citep{dudik2014doubly} (cf.\ Eq.~\ref{eq:gen:dr}), it appears this characteristic problem of position-biased clicks makes existing \ac{DR} estimation inapplicable.

However, as discussed in Section~\ref{background:generic}, this is not a problem for generic counterfactual estimators for \ac{CTR} estimation from logged click data.
Accordingly, previous work that has applied \ac{DR} estimation to clicks has taken the generic approach with corrections based on purely based on action propensities.
For instance, \citet{yuan2020unbiased} use \ac{IPS} and \ac{DR} estimators for \ac{CTR} prediction on advertisements that are presented in different display positions.
Their \ac{IPS} weights are based on the difference in action probabilities between the logging policy and the evaluated policy (cf.\ Eq.~\ref{eq:gen:ips} \&~\ref{eq:gen:dr}).
In a similar vain, \citet{kiyohara2022doubly} propose a \ac{DR} estimator for predicting a \ac{CTR}-based slate-metric under cascading user behavior, they also use \ac{IPS} weights based solely on action probabilities.
These method are very different from the \ac{IPS} approach for unbiased \ac{LTR} (Section~\ref{sec:background:ipsltr}) because their corrections are not based on the mismatch between clicks and relevance, but on the mismatch between action probabilities between policies.
As a result, they cannot handle situations where $\pi_0$ is deterministic and position-bias occurs, in contrast with the \ac{LTR} \ac{IPS} estimator.
We thus argue that the approaches of \citet{yuan2020unbiased} and \citet{kiyohara2022doubly} are better understood as methods for correcting policy differences, instead of methods designed for correcting position-bias in clicks directly.

Another \ac{DR} estimator applied to click data was proposed by \citet{saito2020doubly}, who realized that when estimating post-click conversions, the click signal can be seen as the treatment variable.
This avoids the unobservable examination problem as clicks are always directly observable in the data.
The propensities of \citeauthor{saito2020doubly} are thus based on click probabilities, instead of action or examination probabilities.
While being very useful for post-click conversions, their method cannot be applied to predicting click probabilities or our \ac{LTR} problem setting.

In summary, existing \ac{DR} estimation does not seem directly applicable to position-bias since the treatment variable, item examination, is unobservable in click logs.
To the best of our knowledge, \ac{DR} estimators that have been applied to clicks correct for the mismatch between logging policy and the evaluated policy.
Currently, there does not appear to be a \ac{DR} estimator that uses the correlations between clicks and relevances as state-of-the-art \ac{IPS} estimation for \ac{LTR}.

\section{Method: The Direct Method and Doubly Robust Estimation for Learning to Rank}
\label{sec:method}

In Section~\ref{background:generic} the generic \ac{IPS}, \ac{DM} and \ac{DR} estimators were introduced,
subsequently, Section~\ref{sec:background:ltr} showed how \ac{IPS} has been successfully adapted for the \ac{LTR} problem specifically.
This naturally raises the question whether adaptations of \ac{DM} and \ac{DR} estimation for \ac{LTR} could bring additional success to the field.
To answer this question, this section introduces novel \ac{DM} and \ac{DR} estimators for \ac{LTR} and also a novel estimator for the cross-entropy loss.
 
\subsection{The Direct Method for Learning to Rank}
\label{sec:regression}

As discussed in Section~\ref{background:generic}, the \acf{DM} solely relies on regression to estimate performance, in contrast with \ac{IPS} which uses click frequencies and propensities~\citep{dudik2014doubly}.
The generic \ac{DM} estimator in Eq.~\ref{eq:gen:dm} estimates \ac{CTR} with the estimated bias parameters $\hat{\alpha}$ and $\hat{\beta}$ and the relevance estimates $\hat{R}_\pi$.
However, to estimate $\mathcal{R}$ (Eq.~\ref{eq:reward}), we only require the weight estimate $\omega_d$ and $\hat{R}_d$ per item $d$.
The \ac{DM} estimate of $\mathcal{R}(\pi)$ then is:
\begin{equation}
\widehat{\mathcal{R}}_\text{DM}(\pi) = \sum_{d \in D} \hat{\omega}_d \hat{R}_d.
\label{eq:regression}
\end{equation}
While to the best of our knowledge it is novel, our \ac{DM} estimator is extremely straightforward: for each item we multiply its estimated expected position weight $\hat{\omega}_d$ with its relevance estimate: $\hat{R}_d$.
The biggest difference with the generic \ac{DM} is that instead of using the policy probabilities $\hat{\pi}(y)$ or $\hat{\alpha}$ and $\hat{\beta}$ directly, it uses $\hat{\omega}_d$ which is based on their values (Eq.~\ref{eq:estimatedvalues}).

By considering Eq.~\ref{eq:reward},~\ref{eq:estimatedvalues} and \ref{eq:regression}, we can clearly see $\widehat{\mathcal{R}}_\text{DM}$ has the following condition for unbiasedness: 
\begin{equation}
\big(
\hspace{-0.52cm}
\overbrace{
\hat{\alpha} = \alpha \land \hat{\beta} = \beta
}^\text{pos. bias correctly estimated}
\hspace{-0.45cm}
 \land
\hspace{-0.48cm}
 \underbrace{
\big(\forall d \in D, \;  \hat{R}_d = R_d \big)
}_\text{all regression estimates are correct}
\hspace{-0.55cm}
\big) \longrightarrow
\widehat{\mathcal{R}}_\text{DM}(\pi) = \mathcal{R}(\pi).
\label{eq:regressionbiascondition}
\end{equation}
In other words, both the bias parameters and the regression model have to be accurate for $\widehat{\mathcal{R}}_\text{DM}(\pi)$ to be unbiased.
The first part of the condition is required because accurate $\hat{\alpha}$ and $\hat{\beta}$ are needed for an accurate estimate of the $\omega$ weights.
The second part of the condition: that all regression estimates $\hat{R}_d$ need to be correct, show that it is practically infeasible for \ac{DM} to be unbiased since finding an accurate $\hat{R}_d$ values appears to be as difficult as the ranking task itself.
This reasoning could explain why -- to the best of our knowledge -- no existing work has applied \ac{DM} to unbiased \ac{LTR}.
However, the experimental findings in this paper cast doubt on this reasoning, since they show that \ac{DM} can be more effective than \ac{IPS}, especially when the number of displayed rankings $N$ is not very large.

An advantage of \ac{DM} over \ac{IPS} is how non-clicked items are treated:
The \ac{IPS} estimator (Eq.~\ref{eq:ips}) treats items that are not clicked in the logged data as completely irrelevant items that should be placed at the bottom of a ranking.
As pointed out in previous work~\citep{wang2021non}, this seems very unfair to items that were never displayed during logging, and this \emph{winner-takes-all} behavior could potentially explain the high variance of \ac{IPS}.
In contrast, because \ac{DM} relies on regression estimates it can provide non-zero values to all items, even those never displayed.
Additionally, \ac{DM} does not require any estimate of the logging policy $\hat{\pi}_0$ whereas \ac{IPS} heavily relies on $\hat{\pi}_0$.
But \ac{DM} does not utilize any of the click-data, and thus, \ac{DM} cannot correct for inaccuracies in the regression estimates.
Furthermore, the unbiasedness criteria for \ac{DM} are much less feasible than those of \ac{IPS}.
Ideally, the advantageous properties of both \ac{IPS} and \ac{DM} should be combined in a single estimator, while avoiding the downsides of each approach.
The following subsection considers whether \ac{DR} estimation could result in such a combination.

\subsection{A Novel Doubly-Robust Estimator for Relevance Estimation under Position-Bias}
\label{sec:doublyrobust}

Now that we have \ac{IPS} and \ac{DM} estimators for \ac{LTR}, we only require a \acf{CV} to construct a \acf{DR} estimator~\citep{dudik2014doubly}.
As discussed in Section~\ref{background:generic}, \ac{CV} should have the same expected value as \ac{DM} when \ac{IPS} is unbiased, while simultaneously having a high covariance with \ac{IPS}.
With these requirements in mind, we propose the following \ac{CV}:
\begin{equation}
\widehat{\mathcal{R}}_\text{CV}(\pi) = \frac{1}{N}\sum_{i=1}^N \sum_{d \in D}
\hspace{-0.35cm}
\overbrace{
\frac{\hat{\omega}_{d}}{\hat{\rho}_{d}}
}^\text{IPS weight}
\hspace{-0.35cm}
\hspace{-1.4cm}
\underbrace{
\hat{\alpha}_{k_i(d)} \hat{R}_d
}_\text{increase in click prob. due to relevance}
\hspace{-1.5cm}
.
\label{eq:cv}
\end{equation}
Interestingly, \ac{CV} does not use the observed clicks but utilizes at what ranks an item was displayed in the logged data.
The last part of the estimator represents the increase in click probability an item receives by being displayed at a position, the hope is that this correlates with the actual observed clicks.
Importantly, \ac{CV} has the same expected value as \ac{DM} when $\hat{\pi}$ is correct and clipping has no effect (see Theorem~\ref{theorem:cvunbiasreq} for proof):
\begin{equation}
\underbrace{
\big(
\forall d, \; \hat{\pi}_0(d) = \pi_0(d) \land \hat{\rho}_d \geq \tau
\big)
}_\text{$\hat{\pi}$ is correct and clipping has no effect}
\longrightarrow
\mathds{E}_{y \sim \pi_0}\mleft[
\widehat{\mathcal{R}}_\text{CV}(\pi)
\mright]
 = \widehat{\mathcal{R}}_\text{DM}(\pi)
 .
\end{equation}
If we compare the above condition with the unbiasedness condition of \ac{IPS} in Eq.~\ref{eq:ips:biascondition}, we see that the latter encapsulates the former.
In other words, \ac{CV} is an unbiased estimate of \ac{DM} when \ac{IPS} is an unbiased estimate of $\mathcal{R}$.

Since our \ac{CV} has the required properties, we can straightforwardly propose our novel \ac{DR} estimator (cf.\ Eq.~\ref{eq:gen:dr}):
\begin{equation}
\begin{split}
\widehat{\mathcal{R}}_\text{DR}(\pi)
&=
\widehat{\mathcal{R}}_\text{DM}(\pi) + \widehat{\mathcal{R}}_\text{IPS}(\pi) - \widehat{\mathcal{R}}_\text{CV}(\pi)
\\&=
\widehat{\mathcal{R}}_\text{DM}(\pi)
+ \frac{1}{N}\sum_{i=1}^N \sum_{d \in D}
\hspace{-0.32cm}
\overbrace{
\frac{\hat{\omega}_{d}}{\hat{\rho}_{d}}
}^\text{IPS weight}
\hspace{-1.22cm}
\underbrace{
\mleft(c_i(d)  - 
\hat{\alpha}_{k_i(d)} \hat{R}_d  - \hat{\beta}_{k_i(d)}
 \mright)}_\text{diff. between observed click and predicted click prob.}
 \hspace{-0.92cm}.
\end{split}
\label{eq:dr}
\end{equation}
We see that our \ac{DR} estimator follows a similar structure as generic \ac{DR} estimation (Section~\ref{background:generic}):
it starts with \ac{DM} as a baseline then adds an \ac{IPS} estimate of the difference between \ac{DM} and the true reward $\mathcal{R}$.
Concretely, the difference between each observed click signal $c_i(d)$ and the predicted click probability $\hat{\alpha}_{k_i(d)} \hat{R}_d  + \hat{\beta}_{k_i(d)}$ is taken and reweighted with an \ac{IPS} estimate.
Effectively, the observed clicks are thus used to estimate and correct the error of \ac{DM}.

An intuitive advantage of \ac{DR} is that for items that were never displayed during logging (i.e.\ $\forall 0 < i \leq N, \; \hat{\alpha}_{k_i(d)} = 0$), \ac{DR} relies solely on regression to estimate their relevance, similar to \ac{DM}.
Yet for items that have been displayed many times, \ac{DR} will estimate relevance more similar to \ac{IPS} for those items, thereby it is able to correct for regression mistakes with clicks.
The combination of these properties, means that \ac{DR} can avoid the \emph{winner-takes-all} behavior of \ac{IPS} where all non-displayed or non-clicked items are seen as completely non-relevant and pushed to the bottom of the ranking.
We expect this to mean that \ac{DR} does not have the same variance problems as \ac{IPS}.
At the same time, \ac{DR} still relies on clicks and thus does not require perfectly accurate regression estimates.
Our theoretical analysis shows that this enables \ac{DR} to have more reasonable unbiasedness requirements than \ac{DM}.

The main difference with standard \ac{DR} estimation for contextual bandits, i.e.\ as described by \citet{dudik2014doubly}, and our \ac{DR} estimator is that our \ac{CV} uses a soft expected-treatment variable $\hat{\alpha}_{k_i(d)}$.
This difference is necessary because relevances $R_d$ cannot be observed directly and have to be inferred from click signals $c_i(d)$.
Thus, while standard \ac{CV} would use the observed reward signal, our \ac{CV} infers the relevance from the observed click.
We call $\hat{\alpha}_{k_i(d)}$ a soft expected-treatment variable because it can be seen as the expected effect that relevance had on the click probability.
To the best of our knowledge, our \ac{DR} estimator is the first to use such a soft-treatment variable.

Moreover, in contrast with the existing methods described in Section~\ref{background:generic} and~\ref{sec:doublyrobust:prev} that use propensities based on the mismatch between $\pi_0$ and $\pi$~\citep{yuan2020unbiased, kiyohara2022doubly, saito2020doubly}.
Our \ac{DR} estimator uses the correlation between clicks and relevance to correct for position-bias, it is thus also applicable when the logging policy is deterministic and inherents the advantages that the \ac{LTR} \ac{IPS} estimator has over generic \ac{IPS} estimation.
We thus argue that our \ac{DR} estimator is the first that is designed to directly correct for position-bias, and therefore provides a very significant contribution to the unbiased \ac{LTR} field.

\subsection{Theoretical Properties of the Novel Doubly-Robust Estimator}
\label{sec:drtheory}
Sections~\ref{sec:experimentalsetup} and~\ref{sec:results} experimentally investigate the performance improvements our contribution brings, whereas Appendix~\ref{appendix:drbias} proves several theoretical advantages \ac{DR} has over both \ac{IPS} and \ac{DM} in terms of bias and variance.
We summarize our main theoretical findings in the remainder of this section.

Theorem~\ref{theorem:drlongbias} shows that our \ac{DR} has the following bias:
\begin{equation}
\begin{split}
&
\mathds{E}_{c,y \sim \pi_0}\mleft[
\widehat{\mathcal{R}}_\text{DR}(\pi)
\mright]
- \mathcal{R}(\pi)
\\
&\hspace{0.6cm}
=
\sum_{d \in D}
\frac{\hat{\omega}_{d}}{\hat{\rho}_{d}}
\bigg(
\hspace{-0.1cm}
\underbrace{
\mleft(\mathds{E}_{y \sim \pi_0}\mleft[\alpha_{k(d)}\mright] - \hat{\rho}_{d} \frac{\omega_{d}}{\hat{\omega}_{d}}\mright)
}_\text{error from $\hat{\pi}$, $\hat{\alpha}$, $\hat{\beta}$ and clipping}
\hspace{-0.03cm}
R_d +
\underbrace{
\mleft(\hat{\rho}_{d}-\mathds{E}_{y \sim \pi_0}\mleft[\hat{\alpha}_{k(d)}\mright] \mright)
}_\text{error from $\hat{\pi}$ and clipping}
\hspace{-0.04cm}
 \hat{R}_d + 
 \underbrace{
 \underset{y \sim \pi_0}{\mathds{E}}\mleft[\beta_{k(d)} - \hat{\beta}_{k(d)}\mright]
 }_\text{error from $\hat{\beta}$}
 \bigg)
.
\end{split}
\end{equation}
Furthermore, Corollary~\ref{theorem:drsimplebias} shows that if the bias parameters are correctly estimated, this can be simplified to:
\begin{equation}
\begin{split}
\hspace{-0.6cm}
\underbrace{
\mleft(
\hat{\alpha} = \alpha \land \hat{\beta} = \beta
\mright)
}_\text{pos. bias is correctly estimated}
\hspace{-0.48cm}
\longrightarrow
\mathds{E}_{c,y \sim \pi_0}\mleft[
\widehat{\mathcal{R}}_\text{DR}(\pi)
\mright]
- \mathcal{R}(\pi)
&=
\sum_{d \in D}
\frac{\omega_{d}}{\hat{\rho}_{d}}
\bigg(
\hspace{-0.05cm}
\underbrace{
\mleft(\mathds{E}_{y \sim \pi_0}\mleft[\alpha_{k(d)}\mright] - \hat{\rho}_{d} \mright)
}_\text{error from $\hat{\pi}$ and clipping}
\hspace{-0.05cm}
\hspace{-0.4cm}
\overbrace{
\mleft(R_d - \hat{R}_d \mright)
}^\text{error from regression}
\hspace{-0.4cm}
 \bigg)
.
\end{split}
\end{equation}
The multiplication of errors in the bias can be beneficial for more robustness (cf.\ Eq.~\ref{eq:gen:biasdr}), however, it only occurs when the bias parameters are correct.
From the simplified bias, Theorem~\ref{theorem:drbias} derives the following unbiasedness conditions:
 \begin{equation}
 \hspace{0.25cm}
\big(
\hspace{-0.65cm}
\underbrace{
\hat{\alpha} = \alpha \land \hat{\beta} = \beta
}_\text{pos. bias is correctly estimated}
\hspace{-0.58cm}
\land \,
\big(\forall d \in D, \, 
\hspace{-0.23cm}
\underbrace{
(\hat{\pi}_0(d) = \pi_0(d)  \land \rho_d \geq \tau )
}_\text{$\hat{\pi}$ is correct and clipping has no effect}
\hspace{-0.18cm}
 \lor
\hspace{-0.5cm}
 \overbrace{
 \hat{R}_d = R_d
 }^\text{regression is correct}
\hspace{-0.55cm}
 \big) \big)
 \longrightarrow
\mathds{E}_{c,y \sim \pi_0}\big[ \widehat{\mathcal{R}}_\text{DR}(\pi) \big] = \mathcal{R}_\pi.
\end{equation}
In other words, \ac{DR} is unbiased when the bias parameters are correctly estimated and per item \emph{either} the logging policy distribution is correctly estimated and clipping has no effect \emph{or} the regression estimate is correct.
In contrast, remember that \ac{IPS} needs an accurate $\hat{\pi}_0(d)$ and clipping to have no effects for \emph{all} items, and \ac{DM} needs accurate regression estimates for \emph{all} items.
Therefore, \ac{DR} is unbiased when either \ac{IPS} or \ac{DM} is but can also be unbiased in situations where neither is.
Clearly, our \ac{DR} is more robust than \ac{IPS} and \ac{DM}, yet all of the \ac{LTR} estimators still require accurate bias parameters.
This seems inescapable since our reward $\mathcal{R}$ is also based on user behavior, i.e.\ due to its $\omega$ weights accurate $\alpha$ and $\beta$ estimates are needed.

In addition to the better unbiasedness conditions, Theorem~\ref{theorem:dripsbiascomp} proves our \ac{DR} estimator has less or equal bias than \ac{IPS} when $\hat{\alpha}$ and $\hat{\beta}$ are accurate and each $\hat{R}_d$ estimate is less than twice the true $R_d$ value:
\begin{equation}
\begin{split}
&\hspace{0.25cm}
\big(
\hspace{-0.65cm}
\underbrace{
\hat{\alpha} = \alpha \land \hat{\beta} = \beta
}_\text{pos. bias is correctly estimated}
\hspace{-0.58cm}
\land
\hspace{-1.42cm}
\overbrace{
\big(\forall d \in D, \; 0 \leq \hat{R}_d \leq 2 R_d \big)
}^\text{regression estimates between zero and twice true relevances}
\hspace{-1.5cm}
\big)
\\[-4.5ex]
&\hspace{4.2cm}
\longrightarrow
\underbrace{
| \mathbb{E}_{c,y \sim \pi_0}\mleft[ \mathcal{R}(\pi) \mright] - \hat{\mathcal{R}}_\text{DR}(\pi) |
\leq
| \mathbb{E}_{c,y \sim \pi_0}\mleft[  \mathcal{R}(\pi) \mright] - \hat{\mathcal{R}}_\text{IPS}(\pi) |
}_\text{bias of DR is less or equal than bias of IPS}
\hspace{-0.02cm}
.
\end{split}
\label{eq:dr:lessbiascond}
\end{equation}
We see that our \ac{DR} estimator is able to reduce bias with somewhat accurate regression estimates.
In particular, it appears that it mitigates some of the bias introduced to \ac{IPS} by clipping.
Overall, it appears that our \ac{DR} estimator has better unbiasedness criteria than \ac{IPS} or \ac{DM} and has lower bias than \ac{IPS} given adequate regression estimates.

Besides bias, we should also consider the variance of our \ac{DR} estimator, from Eq.~\ref{eq:dr} it follows that (cf.\ Eq.~\ref{eq:gen:drvar}):
\begin{equation}
\mathds{V}\mleft[ \widehat{\mathcal{R}}_\text{DR}(\pi) \mright]
 =
 \mathds{V}\mleft[ \widehat{\mathcal{R}}_\text{IPS}(\pi) \mright]
 + 
 \mathds{V}\mleft[ \widehat{\mathcal{R}}_\text{CV}(\pi) \mright]
 - 
 2\mathbb{C}\text{ov}\mleft(\widehat{\mathcal{R}}_\text{IPS}(\pi),  \widehat{\mathcal{R}}_\text{CV}(\pi)\mright).
\end{equation}
Thus, a large covariance between \ac{IPS} and \ac{CV} allows for a reduction in the variance of our \ac{DR} estimator.
To better understand when this may be the case, Theorem~\ref{theorem:drvariance} proves the following condition for improved variance over \ac{IPS}:
\begin{equation}
\hspace{0.25cm}
\big(
\hspace{-0.65cm}
\underbrace{
\hat{\alpha} = \alpha \land \hat{\beta} = \beta
}_\text{pos. bias is correctly estimated}
\hspace{-0.58cm}
\land
\hspace{-1.43cm}
\overbrace{
\big(\forall d \in D, \; 0 \leq \hat{R}_d \leq 2 R_d \big)
}^\text{regression estimates between zero and twice true relevances}
\hspace{-1.5cm}
\big)
\longrightarrow
\hspace{-0.51cm}
\underbrace{
\mathds{V}\big[ \widehat{\mathcal{R}}_\text{DR}(\pi) \big] \leq \mathds{V}\big[ \widehat{\mathcal{R}}_\text{IPS}(\pi) \big]
}_\text{variance of DR is less or equal than that of IPS}
\hspace{-0.68cm}
.
\end{equation}
We note that this is the same condition as in Eq.~\ref{eq:dr:lessbiascond}: correct $\hat{\alpha}$ and $\hat{\beta}$ estimates and regression estimates that are somewhat correct.
Interestingly, this shows that under this condition \ac{DR} can improve over \ac{IPS} both in terms of bias and variance.
In contrast, while the practice of clipping reduces variance but introduces bias~\citep{joachims2017unbiased, strehl2010logged}, it appears that under certain conditions \ac{DR} can avoid this tradeoff altogether.

Finally, we note that there is also an important exception; our \ac{DR} estimator is equivalent to \ac{IPS} when all $\hat{\alpha}_{k_i(d)}$ are equal to their corresponding $\rho_d$:
\begin{equation}
\big( \forall\, 0 < i \leq N, \; \forall d \in D, \; \hat{\alpha}_{k_i(d)}  = \hat{\rho}_d\big) \longrightarrow \widehat{\mathcal{R}}_\text{DR}(\pi) = \widehat{\mathcal{R}}_\text{IPS}(\pi).
\end{equation}
There are only two non-trivial situations where this can occur:
\begin{enumerate*}[label=(\roman*)]
\item when the logging policy $\pi_0$ is deterministic \emph{and} clipping has no effect: $\forall d \in D,\, \rho_d \geq \tau$; and
\item when all regression estimates are zero: $\forall d \in D, \, \hat{R}_d =0$.
\end{enumerate*}
In all other scenarios, our \ac{DR} estimator does not reduce to \ac{IPS} estimation.
This means that even when $\pi_0$ is deterministic, \ac{DR} can have benefits over \ac{IPS} when clipping is applied.

To summarize, we have introduced a novel \ac{DR} estimator that is specifically designed for the \ac{LTR} problem.
In terms of bias and variance, our \ac{DR} estimator is more robust than both the \ac{IPS} and the \ac{DM} estimators: 
when either of \ac{IPS} or \ac{DM} is unbiased the \ac{DR} estimator is also unbiased, and in addition, there exist cases where \ac{DR} is unbiased and neither \ac{IPS} nor \ac{DM} are.
Moreover, when the bias parameters are accurate and all regression estimates are between zero and twice the true preferences, we can prove that both the bias and variance of \ac{DR} are less or equal to those of \ac{IPS}.

In terms of theory, our novel \ac{DR} estimator is a breakthrough for the unbiased \ac{LTR} field:
it is the first unbiased \ac{LTR} method that uses \ac{DR} estimation to directly correct for position-bias,
importantly, this makes it provenly more robust than \ac{IPS} estimation in terms of both bias and variance.
Our \ac{DR} estimator is applicable in any unbiased \ac{LTR} setting where \ac{IPS} can be applied and with any regression estimates, allowing for widespread adoption across the entire  field.

\subsection{Applying \ac{LTR} to Doubly-Robust Ranking Metric Estimates}

It might not be directly obvious how \ac{LTR} can be performed with the \ac{DR} estimator, while it is actually very straightforward.
To begin, we consider the common approaches for \ac{LTR} when relevances are known: bounding~\citep{wang2018lambdaloss, burges2010ranknet} and sample-based approximation~\citep{williams1992simple, oosterhuis2021computationally, ustimenko2020stochasticrank}.
Bounding has a long tradition in \ac{LTR} for optimizing deterministic models~\citep{burges2010ranknet}; \citet{wang2018lambdaloss} introduced the LambdaLoss method and proved that it can bound ranking metrics, let $R$ be a vector of all true item relevances then:
$
\text{LambdaLoss}(\pi, R)
\leq
\mathcal{R}(\pi)
$.
For probabilistic policies, the policy gradient can be approximated based on sampled rankings~\citep{williams1992simple}, recently \citet{oosterhuis2021computationally} proposed the PL-Rank method:
$
\text{PL-Rank}(\pi, R)
\approx
\frac{\delta}{\delta \pi}\mathcal{R}(\pi)
$.

To apply \ac{LTR} methods like these to estimated metrics, we follow \citet{oosterhuis2020topkrankings} and reformulate the $\widehat{\mathcal{R}}_\text{DR}(\pi)$ to a sum over items that with expected-rank weights $\hat{\omega}_{d}$ and relevance estimates $\hat{\mu}_d$:
\begin{equation}
\widehat{\mathcal{R}}_\text{DR}(\pi)
=
\sum_{d \in D}
\hat{\omega}_{d}
\bigg(
\hat{R}_d
+
 \frac{1}{N \cdot \hat{\rho}_d}\sum_{i=1}^N
 \mleft(c_i(d)  - 
\hat{\alpha}_{k_i(d)} \hat{R}_d  - \hat{\beta}_{k_i(d)}
 \mright)
\bigg)
=
\sum_{d \in D}
\hat{\omega}_{d} \hat{\mu}_d
.
\end{equation}
Let $\mu$ indicate a vector of all relevance estimates, 
\citet{oosterhuis2020topkrankings} prove that LambaLoss can be used as a bound on the estimated metric:
$
\text{LambdaLoss}(\pi, \hat{\mu})
\leq
\widehat{\mathcal{R}}_\text{DR}(\pi)
$.
Similarly, the derivation of \citet{oosterhuis2021computationally} is equally applicable to $\widehat{\mathcal{R}}_\text{DR}(\pi)$ and can thus approximate the policy gradient:
$
\text{PL-Rank}(\pi, \hat{\mu})
\approx
\frac{\delta}{\delta \pi}\widehat{\mathcal{R}}_\text{DR}(\pi)
$.
As such, existing \ac{LTR} methods are straightforwardly applied to our \ac{DR} estimator in order to optimize ranking models w.r.t.\ unbiased click-based estimates of performance.
 
\subsection{Novel Cross-Entropy Loss Estimation}
\label{sec:novelcrossentropy}
While our \ac{DR} estimator can be applied with any regression model, we will propose a novel estimator for the cross-entropy loss to optimize an accurate regression model.
There are two issues with the existing $\widehat{\mathcal{L}}'$ estimator (Eq.~\ref{eq:prevCE}) we wish to avoid:
 \begin{enumerate*}[label=(\roman*)]
\item $\widehat{\mathcal{L}}'$ does not correct for trust-bias, and
\item for any never-displayed item $d$ the $\widehat{\mathcal{L}}'$ estimate contains the $\log(1 -\hat{R}_d)$ loss that pushes $\hat{R}_d$ towards zero.
\end{enumerate*}
In other words, $\widehat{\mathcal{L}}'$ penalizes positive $\hat{R}_d$ values for items that were never displayed during logging, while
it seems more intuitive that a loss estimate should be indifferent to the $\hat{R}_d$ values of never-displayed items.
We propose the following estimator:
\begin{equation}
\widehat{\mathcal{L}}(\hat{R}) = - \frac{1}{N} \sum_{i = 1}^N \sum_{d \in D}
\hspace{-0.34cm}
\overbrace{
\frac{1}{\hat{\rho}_d}
}^\text{IPS weight}
\hspace{-0.32cm}
\Big(
\hspace{-2.2cm} 
\underbrace{
\big(c_i(d) - \hat{\beta}_{k_i(d)}\big)
}_\text{diff. between observed click and predicted click prob. if $R_d =0$}
\hspace{-2.22cm} 
\log\big(\hat{R}_d\big)
 +
 \hspace{-1.58cm} 
 \overbrace{
 \big(\hat{\alpha}_{k_i(d)} + \hat{\beta}_{k_i(d)} - c_i(d)\big)
 }^\text{diff. between predicted click prob. if $R_d =1$ and observed click}
 \hspace{-1.63cm} 
 \log\big(1 -\hat{R}_d\big)\Big).
\label{eq:CEestimator}
\end{equation}
Our novel estimator has $\hat{\beta}$ corrections to deal with trust-bias and utilizes the $\hat{\alpha}_{k_i(d)}$ to weight the negative part of the loss: $\log(1 -\hat{R}_d)$.
One possible interpretation is that $\hat{\alpha}_{k_i(d)}$ replaces the $1$ in Eq.~\ref{eq:prevCE} using the fact that $\mathds{E}_{y\sim \pi_0}[\alpha_{k_i(d)}/\rho_d] = 1$.
Another interpretation is that the second weight looks at the difference between the expected click probability if the item was maximally relevant ($R_d =1$) and the observed click frequency, the expected difference reveals how much relevance the item \emph{lacks}.
Regardless of interpretation, the important property is that $ \mathds{E}_{c,y\sim \pi_0}\mleft[\frac{1}{\hat{\rho}_d}\mleft(\hat{\alpha}_{k_i(d)} + \hat{\beta}_{k_i(d)} - c_i(d)\mright)\mright] = (1 - R_d)$.
Furthermore, when an item $d$ is never displayed, the corresponding $\hat{R}_d$ does not affect the estimate since in that case: $\forall \, 0 < i \leq N, \; \hat{\alpha}_{k_i(d)} = 0 \land  \hat{\beta}_{k_i(d)} = 0 \land c_i(d) = 0$.
Appendix~\ref{appendix:loglikelihoodbias} proves $\widehat{\mathcal{L}}$ is unbiased in the following circumstances:
\begin{equation}
\hspace{-0.45cm}
\overbrace{
\big(\hat{\alpha} = \alpha \land \hat{\beta} = \beta
}^\text{pos. bias correctly estimated}
\hspace{-0.45cm}
 \land
\hspace{-0.15cm}
\underbrace{
\mleft(\forall d \in D, \;  \hat{\pi}_0(d) = \pi_0(d) \land \rho_d \geq \tau \mright) \big)
}_\text{$\hat{\pi}$ is correctly estimated and clipping has no effect}
\hspace{-0.15cm}
\longrightarrow
\mathds{E}_{c,y\sim \pi_0}\big[\widehat{\mathcal{L}}(d)\big] = \mathcal{L}(d).
\end{equation}
These are the same conditions as those we proved for the \ac{IPS} estimator (Eq.~\ref{eq:ips:biascondition}): the bias parameters and the logging policy need to be accurately estimated and clipping should have no effect.
Thus our novel cross-entropy loss estimator can correct for position-bias, even when trust-bias is present, and is indifferent to predictions on never-displayed items.

\section{Experimental Setup}
\label{sec:experimentalsetup}
In order to evaluate our novel \ac{DR} estimator, we apply the semi-synthetic setup that is common in unbiased \ac{LTR}~\citep{oosterhuis2018differentiable, oosterhuis2019optimizing, hofmann2013reusing, zhuang2020counterfactual, vardasbi2020trust, ovaisi2020correcting, joachims2017unbiased}.
This simulates a web-search scenario by sampling queries and documents from commercial search datasets, while user interactions and rankings are simulated using probabilistic click models.
We use the three largest publicly-available \ac{LTR} industry datasets: \emph{Yahoo!\ Webscope}~\citep{Chapelle2011}, \emph{MSLR-WEB30k}~\citep{qin2013introducing} and \emph{Istella}~\citep{dato2016fast}.
Each dataset contains queries, preselected documents per query and for the query-document pairs: feature representations and labels indicating expert-judged relevance, with $\text{label}(d) \in \{0,1,2,3,4\}$ we use $P(R = 1 \mid d) = 0.25 \cdot \text{label}(d)$.
The queries in the datasets are divided into training, validation and test partitions.
Our logging policy is obtained by supervised training on 1\% of the training partition~\citep{joachims2017unbiased}.
At each interaction $i$, a query is sampled uniformly over the training and validation partitions and a corresponding ranking is sampled from the logging policy.
Clicks are simulated using the click model in Eq.~\ref{eq:clickprob}.
We simulate both a top-$5$ setting, where only five items can be displayed at once, and a full-ranking setting where all items are displayed simultaneously.
The parameters for the top-5 setting are based on empirical work by \citet{agarwal2019addressing}:
$\alpha^\text{top-5} = [0.35, 0.53, 0.55, 0.54, 0.52, 0, 0, \ldots]$ and $\beta^\text{top-5} = [0.65, 0.26, 0.15, 0.11, 0.08, 0, 0, \ldots]$;
for the full-ranking setting we use Eq.~\ref{eq:alphabeta} with: $P(O=1 \,|\, k) = (1 + (k-1)/5)^{-2}$,
$\epsilon^+_k = 1$, and $\epsilon^-_k = 0.1 + 0.6/(1+ k/20)$, because these closely match the top-5 parameters while being applicable to longer rankings.
We simulate both top-$5$ settings where $\alpha$ and $\beta$ are known and where they are estimated with \ac{EM}~\citep{vardasbi2020trust, agarwal2019addressing}.
All models are neural networks with two 32-unit hidden layers, applied in Plackett-Luce ranking models optimized using policy gradients estimated with PL-Rank-2~\citep{oosterhuis2021computationally}.
The only exception is the logging policy in the full-ranking settings which is a deterministic ranker to better match earlier work~\citep{joachims2017unbiased, vardasbi2020trust, agarwal2019counterfactual}.
Propensities $\rho_d$ use frequentist estimates of the logging policy: $\hat{\pi}_0(k \,|\, d) = \frac{1}{N}\sum_{i=1}^N \mathds{1}[k_i(d) = k]$, we clip with $\tau^\text{top-5} = 10/\sqrt{N}$ in the top-5 setting and $\tau^\text{full} = 100/\sqrt{N}$ in the full-ranking setting.
Early stopping is applied using counterfactual estimates based on clicks on the validation set.

Our main performance metric is the expected number of clicks on preferred items (ECP), as introduced in Section~\ref{sec:ltrgoal}.
In addition, Appendix~\ref{appendix:ndcg} also provides our main results measured with the NDCG metric.

Our results evaluate our \ac{DM} estimator (Eq.~\ref{eq:regression}) and our \ac{DR} estimator (Eq.~\ref{eq:regression}) both using the estimates of a regression model optimized by our $\widehat{\mathcal{L}}$ loss (Eq.~\ref{eq:CEestimator}).
Their performance is compared with the following baselines:
 \begin{enumerate*}[label=(\roman*)]
\item a naive estimator that ignores bias (Eq.~\ref{eq:ips} with $\tau = 1$);
\item \ac{IPS} (Eq.~\ref{eq:ips});
\item \acf{RPS}~\citep{wang2021non}; and
\item \ac{DM} optimized with $\widehat{\mathcal{L}}'$ (Eq.~\ref{eq:prevCE} \&~\ref{eq:regression}) from previous work~\citep{bekker2019beyond, saito2020unbiased}.
\end{enumerate*}
None of the estimators receive any information about queries that were not sampled in the training data.
To compare the differences with the optimal performance possible, we also optimize a model based on the true labels (full-information).

As an example for clarity, the following procedure is used to evaluate the performance of our \ac{DR} estimator at $N$ displayed rankings in the top-5 setting with estimated bias parameters:
 \begin{enumerate*}[label=(\arabic*)]
\item $N$ queries are sampled with replacement from the training and validation partitions, a displayed ranking is generated for each sampled query using the stochastic logging policy.
\item Clicks on each ranking are simulated using the click model in Eq.~\ref{eq:clickprob} and the true $\alpha$ and $\beta$ parameters.
\item \ac{EM} is applied to the simulated click data to obtain estimated $\hat{\alpha}$ and $\hat{\beta}$  parameters.
\item A regression model is optimized using $\mathcal{L}$, $\hat{\alpha}$, $\hat{\beta}$ and the click data simulated on the training set, $\hat{R}_d$ is computed for each item.
\item A ranking model is optimized to maximize the \ac{DR} estimate of its ECP, using $\hat{\alpha}$, $\hat{\beta}$, $\hat{R}_d$ and the training click data, early stopping criteria are estimated with the validation click data.
\item Finally, the true ECP ($\mathcal{R}_\pi$, Eq.~\ref{eq:reward}) of the resulting ranking model is computed on the test-set and added to our results.
\end{enumerate*}
We repeat each procedure twenty times independently and report the mean results in addition to standard deviation and 90\% confidence intervals.
Statistical differences with the performance of our \ac{DR} estimator were measured via a two-sided student's t-test~\citep{student1908probable}.

{\renewcommand{\arraystretch}{0.2}
\begin{figure*}[tp]
\centering
\begin{tabular}{@{}r @{}l @{}l @{}l}
&
 \multicolumn{1}{c}{\hspace{0.24cm} \footnotesize Yahoo! Webscope}
&
 \multicolumn{1}{c}{\hspace{0.13cm} \footnotesize MSLR-WEB30k}
&
 \multicolumn{1}{c}{\hspace{-0.08cm} \footnotesize Istella}
\\
\multirow{2}{5mm}{\raisebox{.2\normalbaselineskip}[0pt][0pt]{\rotatebox[origin=c]{90}{
\footnotesize Top-5 Setting with Known Bias Parameters
}}} &
\includegraphics[scale=0.338]{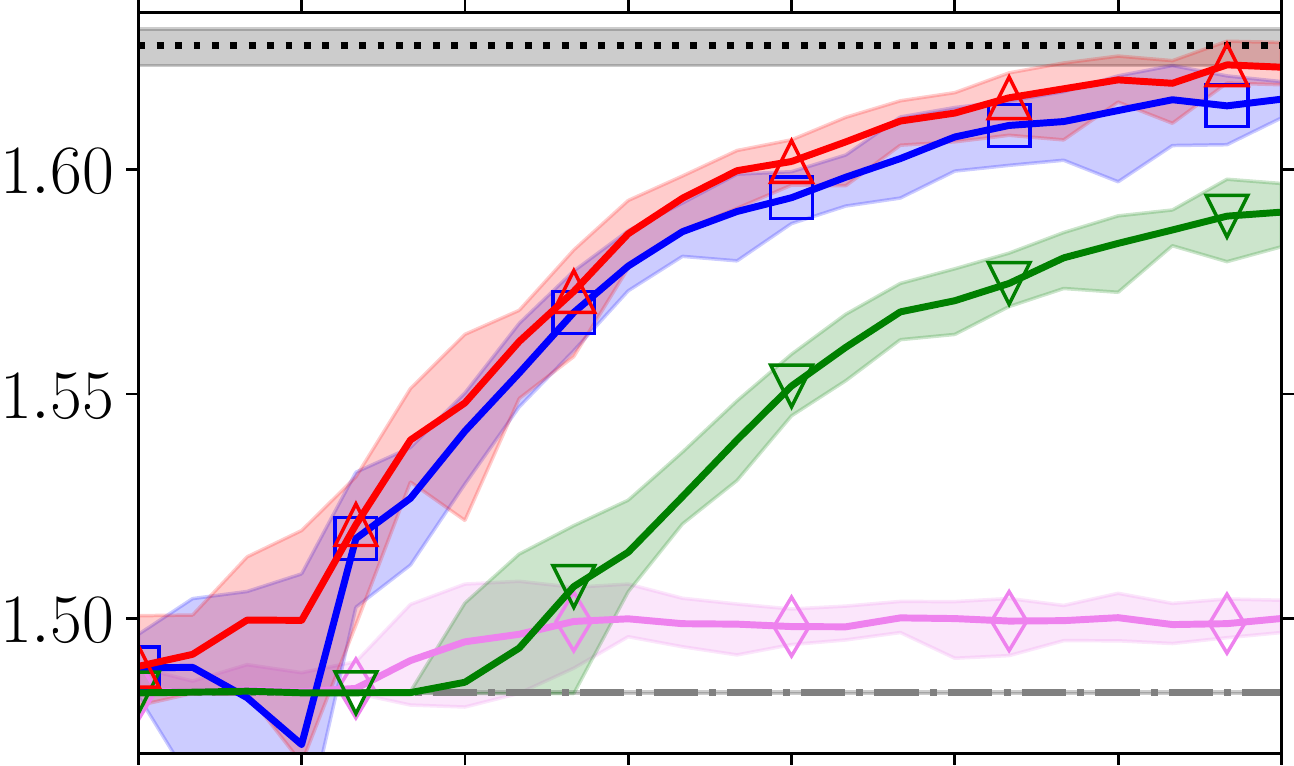}\hspace{1.28mm} &
\includegraphics[scale=0.338]{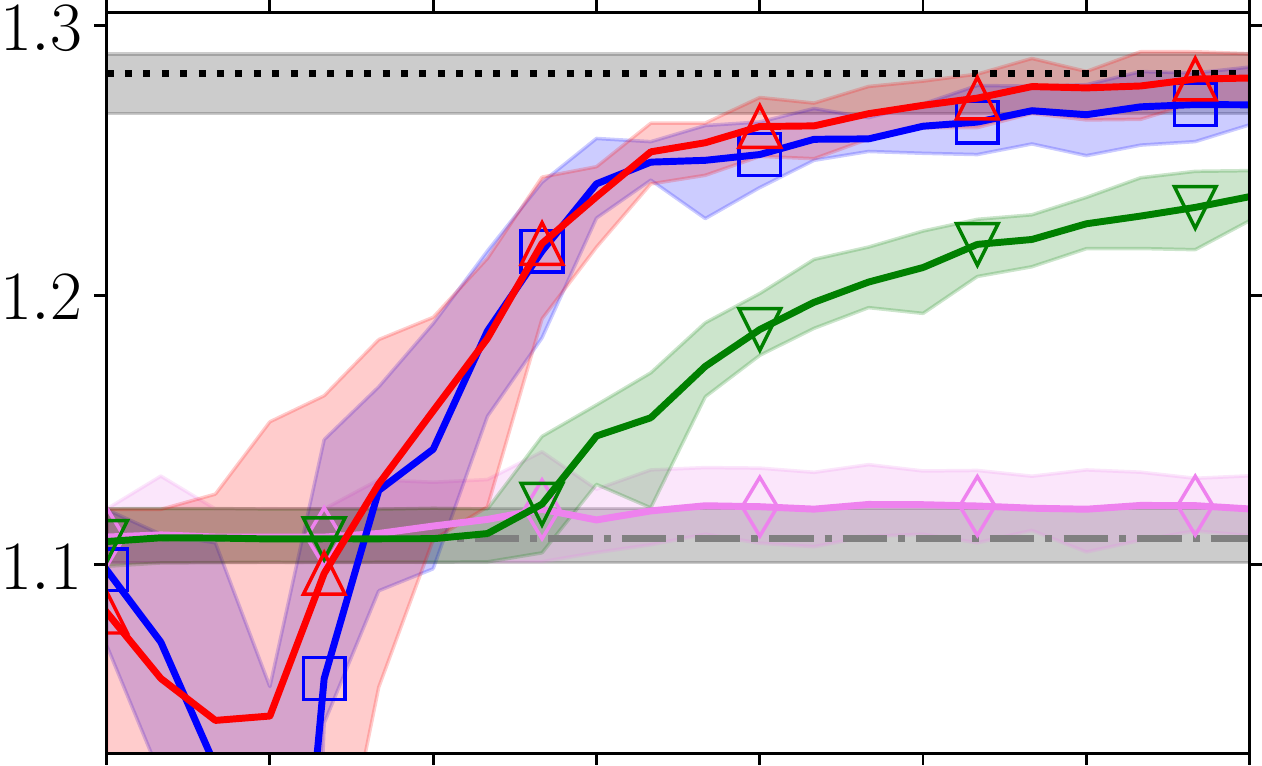}\hspace{1.28mm} &
\includegraphics[scale=0.338]{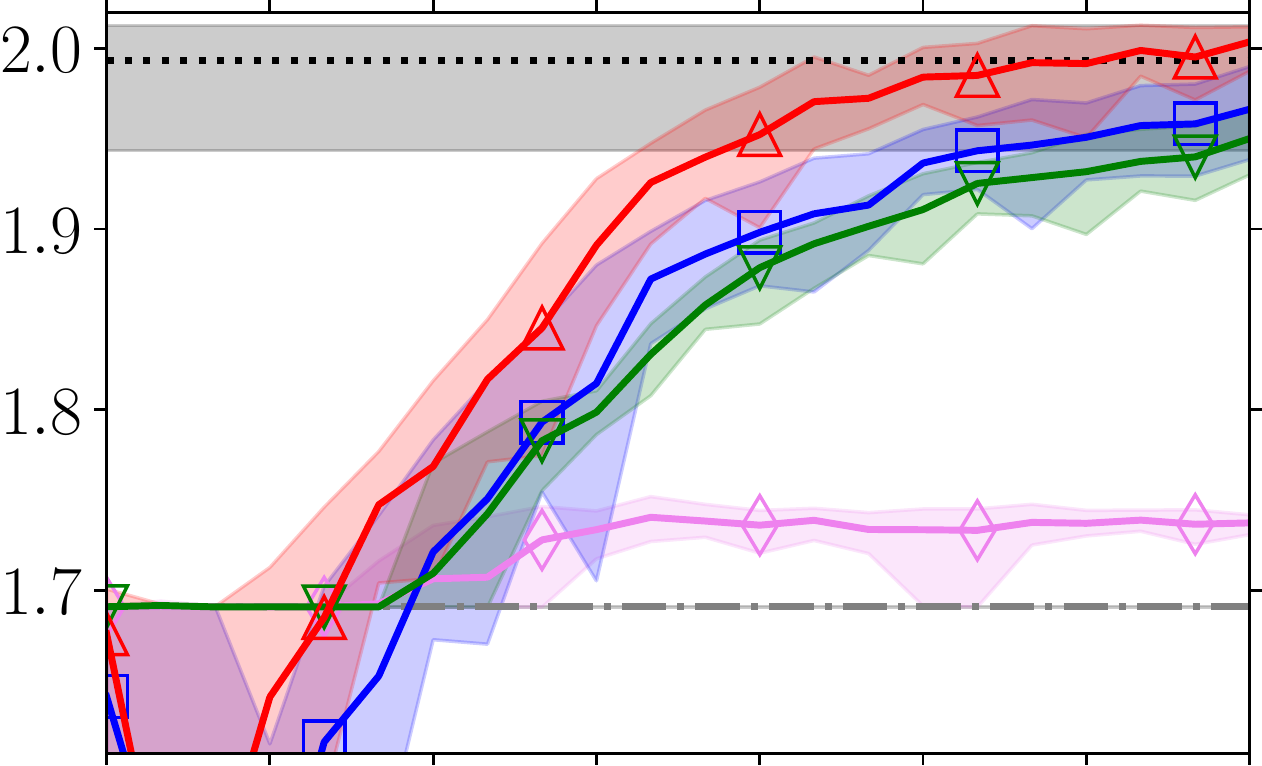}\hspace{1.28mm}
\\&
\includegraphics[scale=0.338]{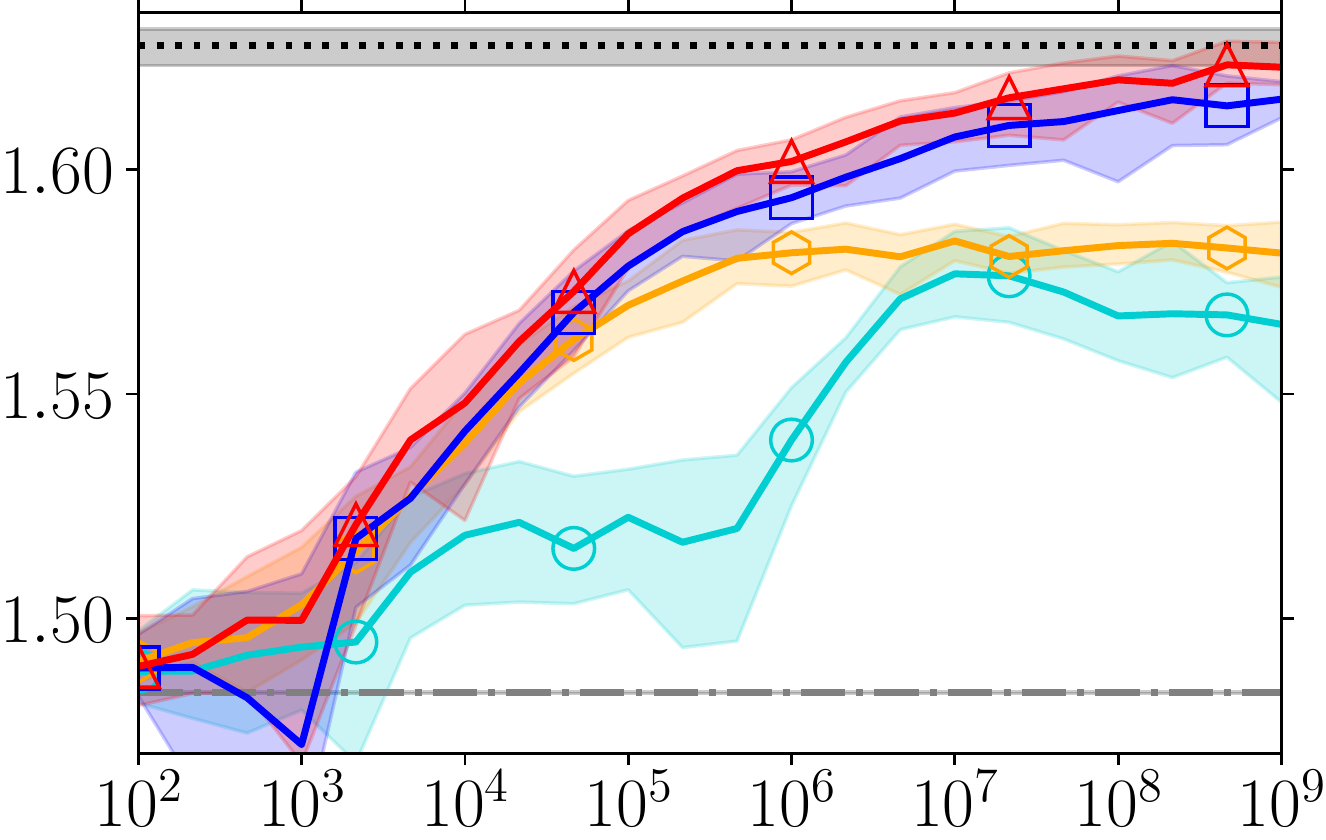} &
\includegraphics[scale=0.338]{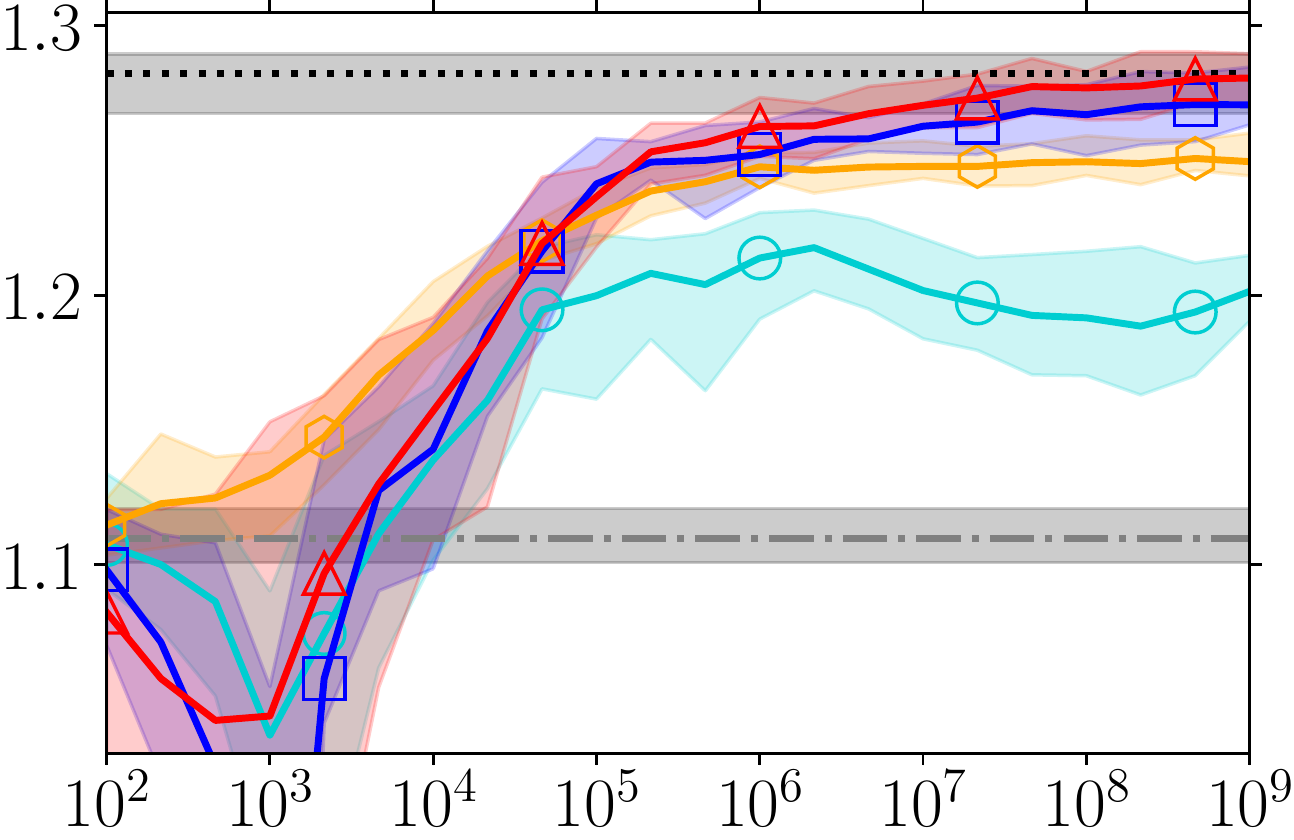} &
\includegraphics[scale=0.338]{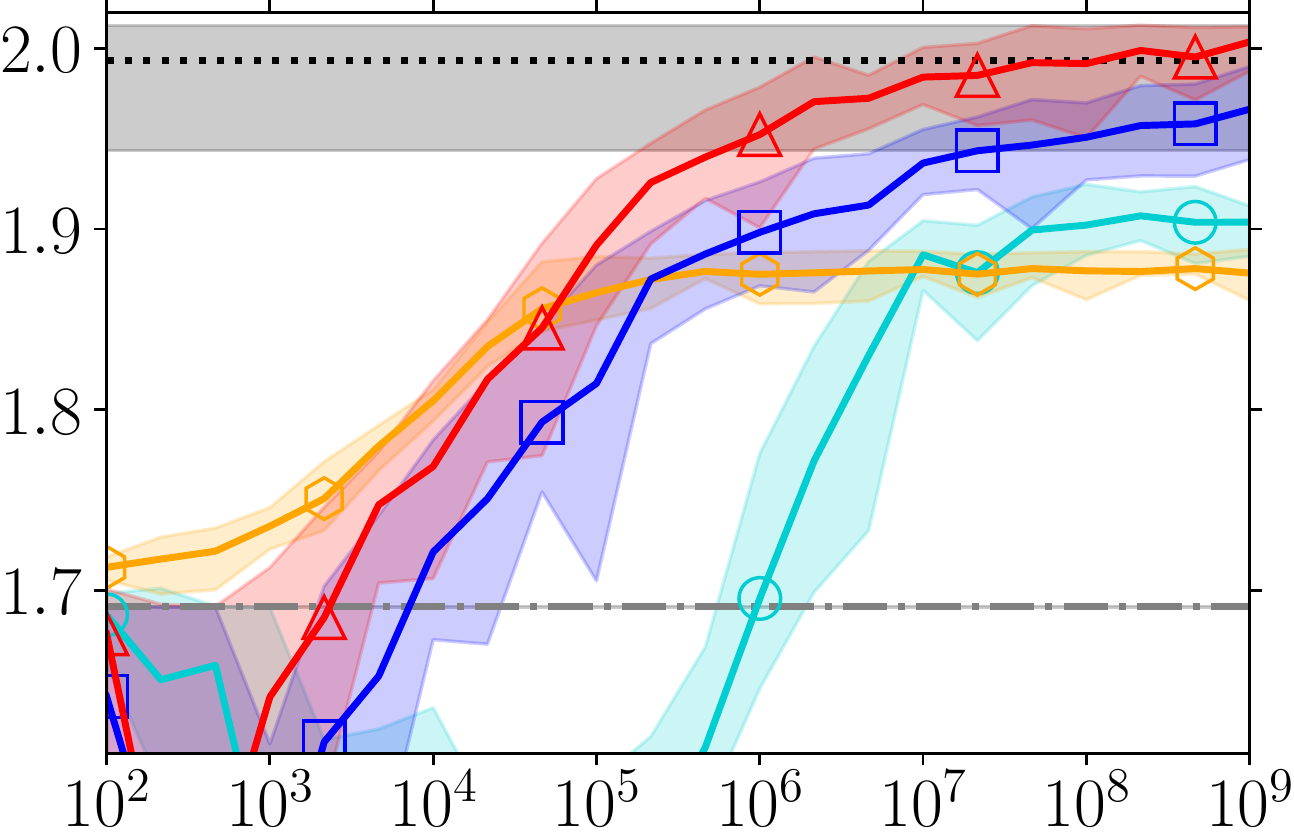}
\\
\midrule
\multirow{2}{5mm}{\raisebox{.2\normalbaselineskip}[0pt][0pt]{\rotatebox[origin=c]{90}{
\footnotesize Top-5 Setting with Estimated Bias Parameters
}}} &
\includegraphics[scale=0.338]{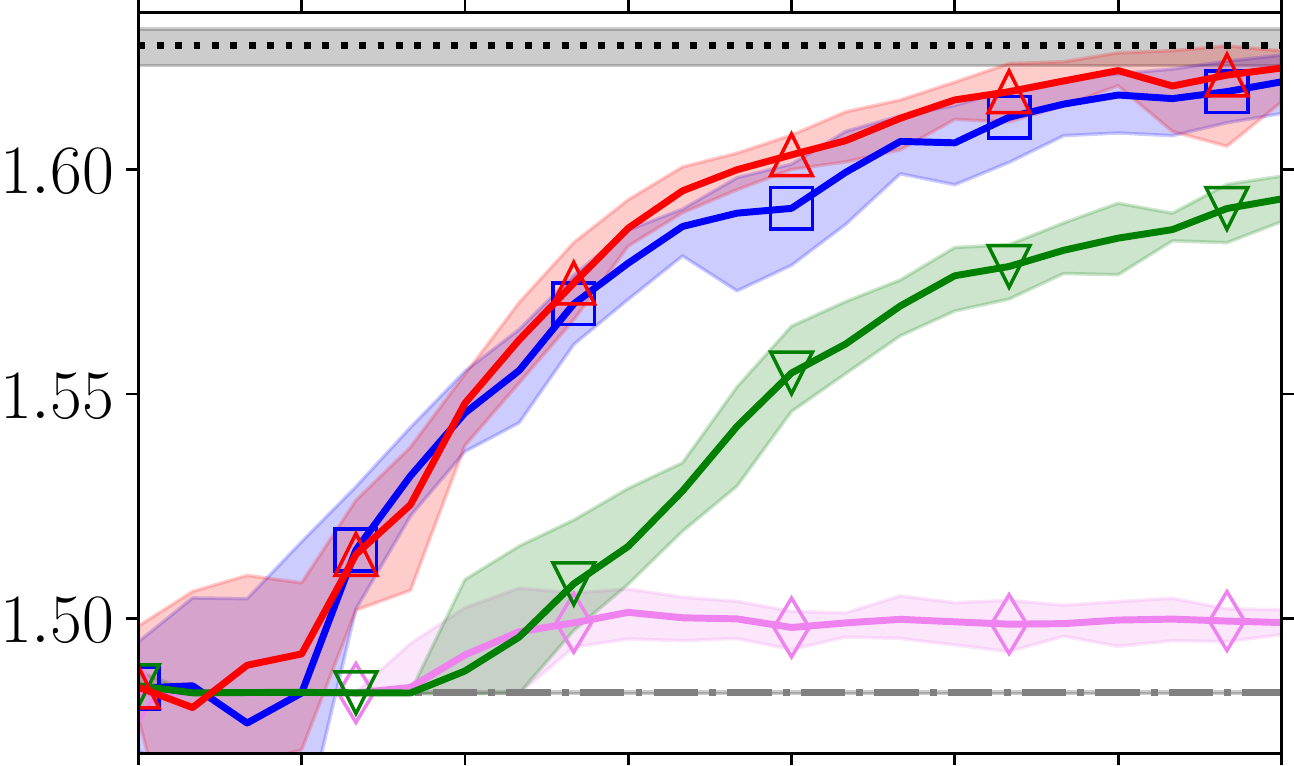}\hspace{1.28mm} &
\includegraphics[scale=0.338]{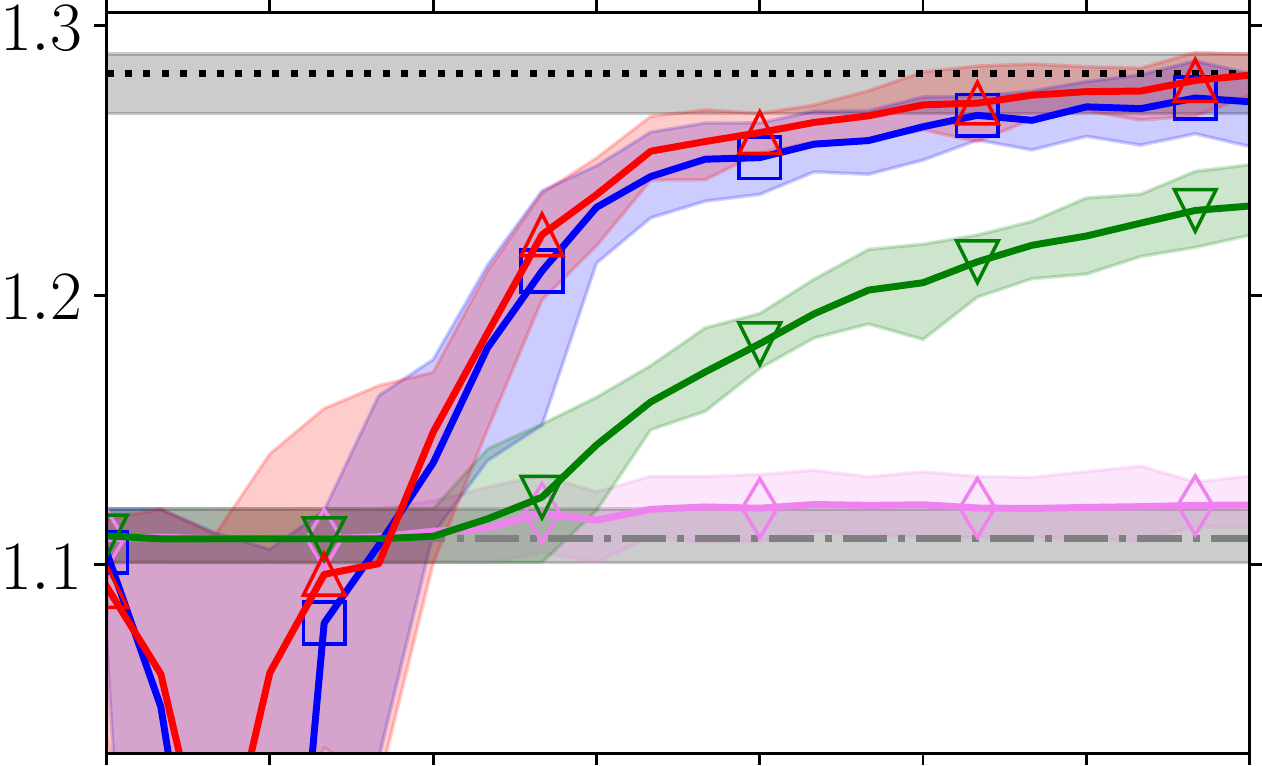}\hspace{1.28mm} &
\includegraphics[scale=0.338]{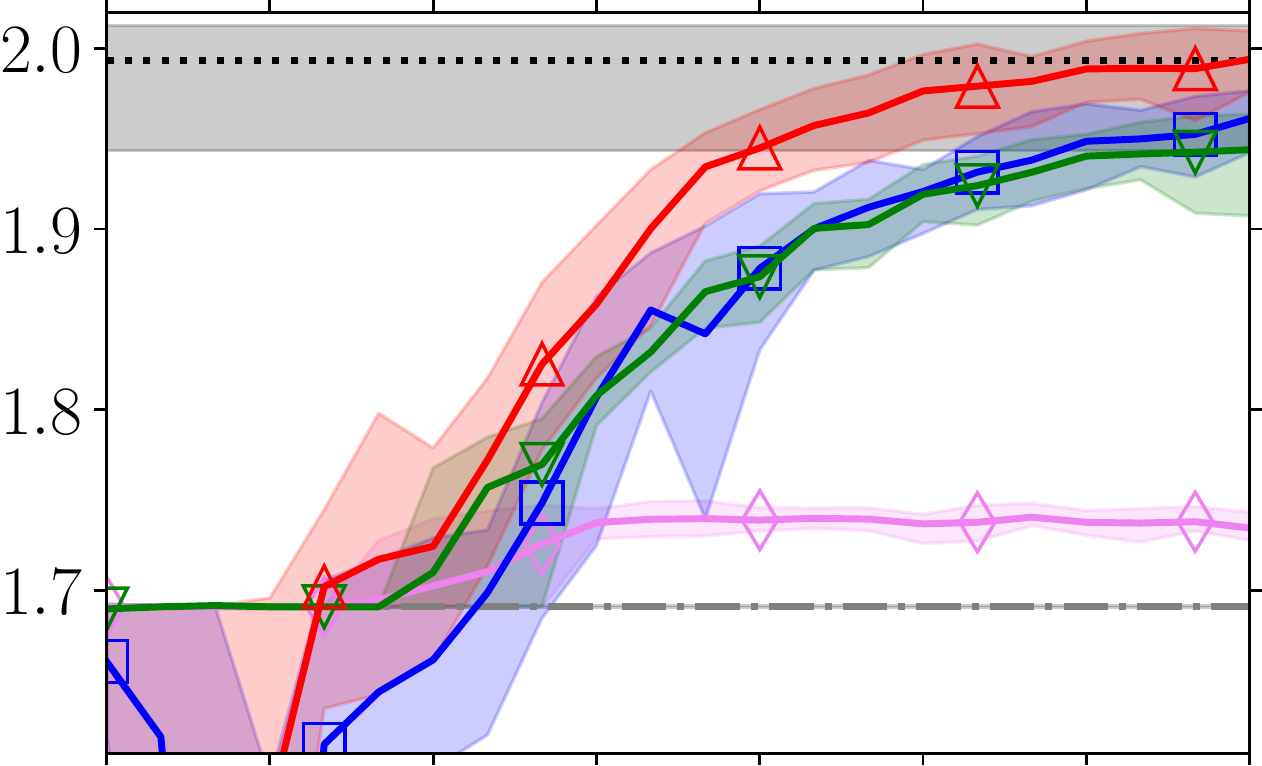}\hspace{1.28mm}
\\&
\includegraphics[scale=0.338]{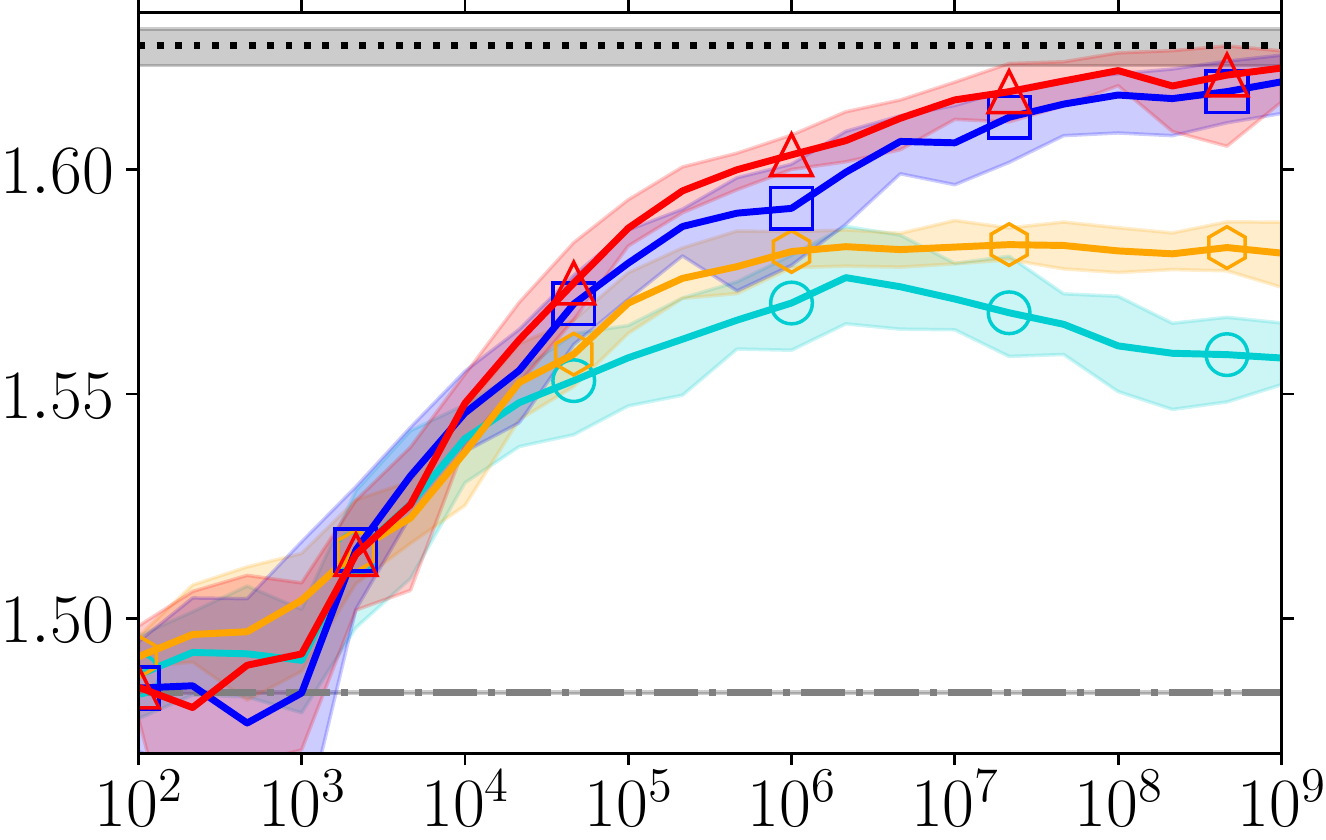} &
\includegraphics[scale=0.338]{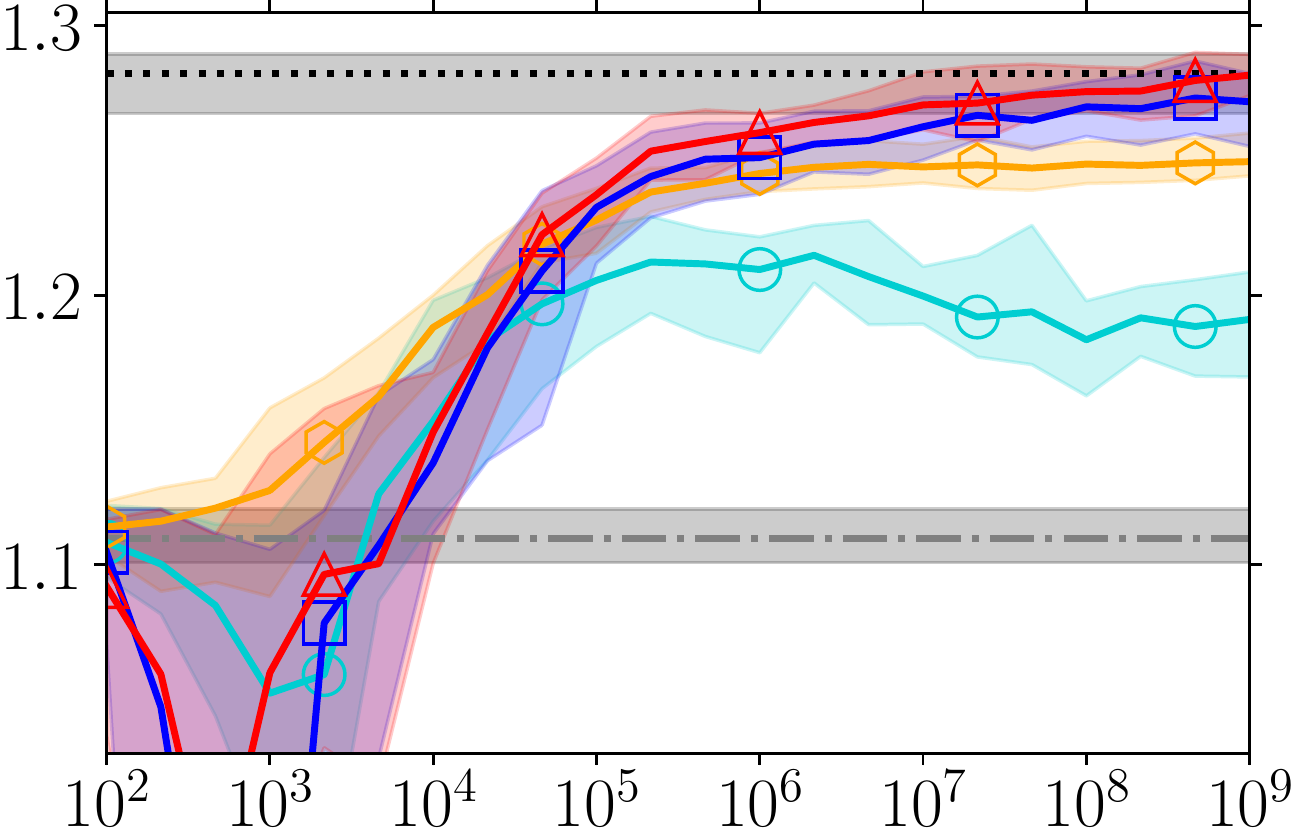} &
\includegraphics[scale=0.338]{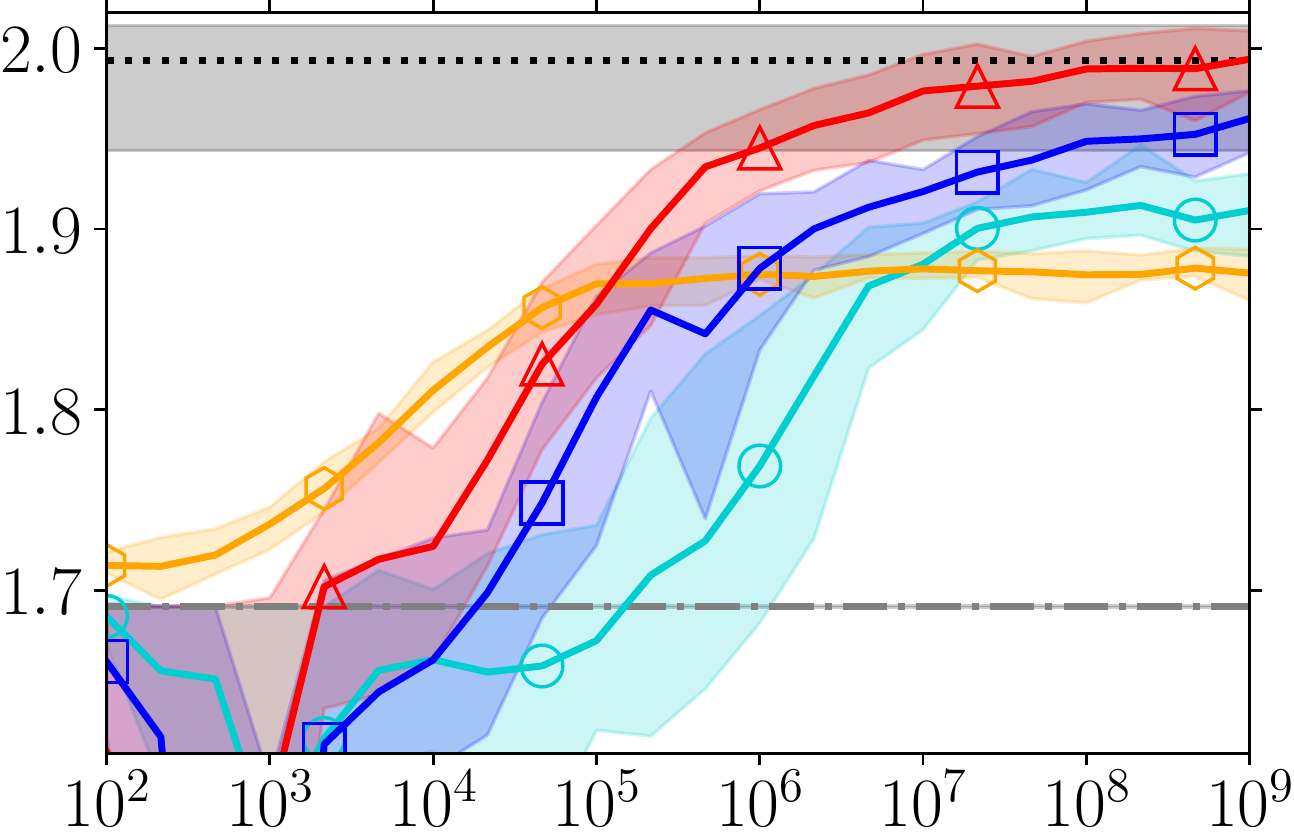}
\\
\midrule
\multirow{2}{5mm}{\raisebox{.2\normalbaselineskip}[0pt][0pt]{\rotatebox[origin=c]{90}{
\footnotesize Full-Ranking with Known Bias Parameters
}}} &
\includegraphics[scale=0.338]{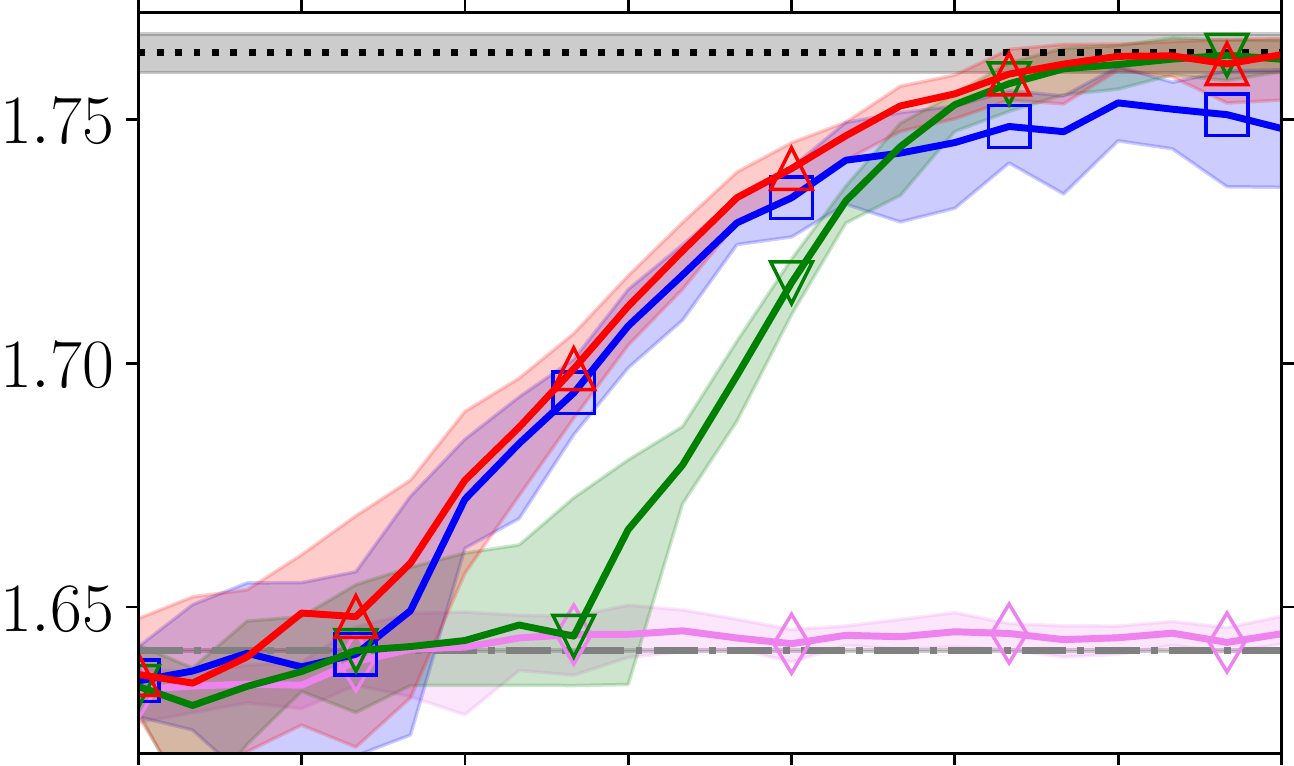}\hspace{1.28mm} &
\includegraphics[scale=0.338]{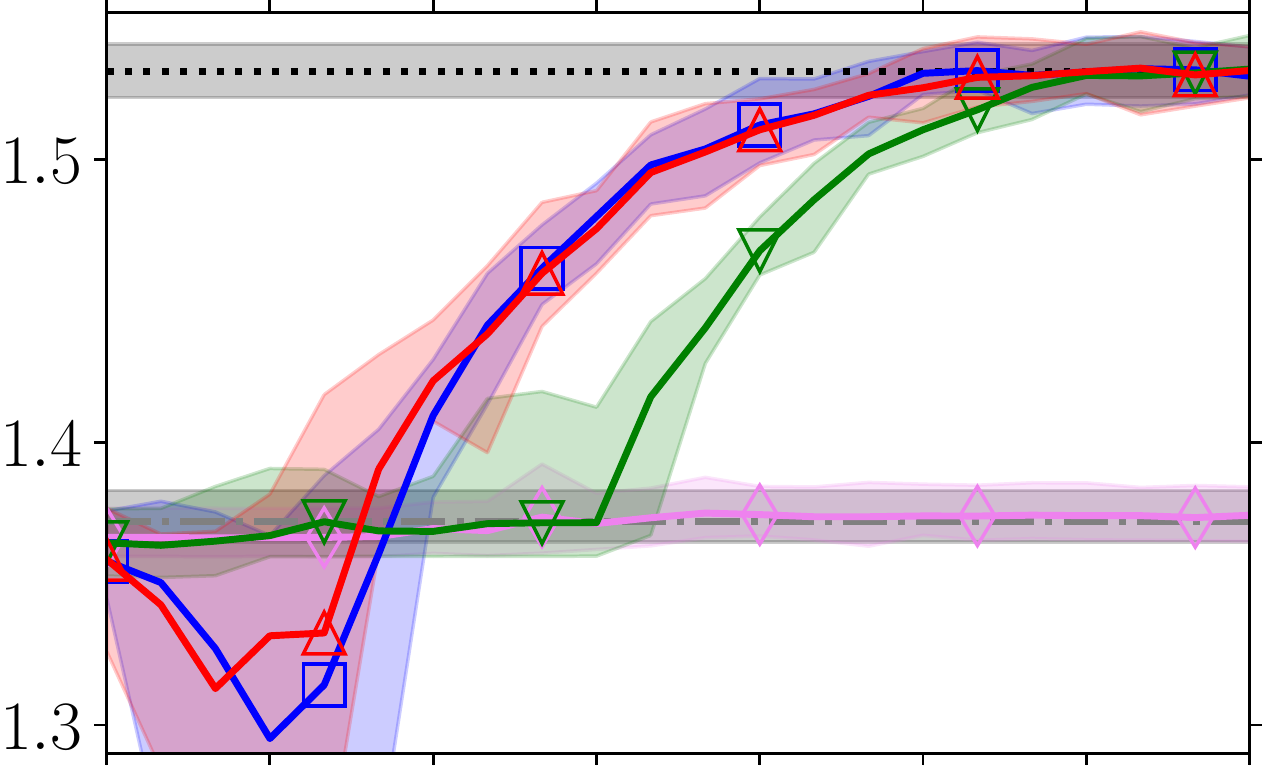}\hspace{1.28mm} &
\includegraphics[scale=0.338]{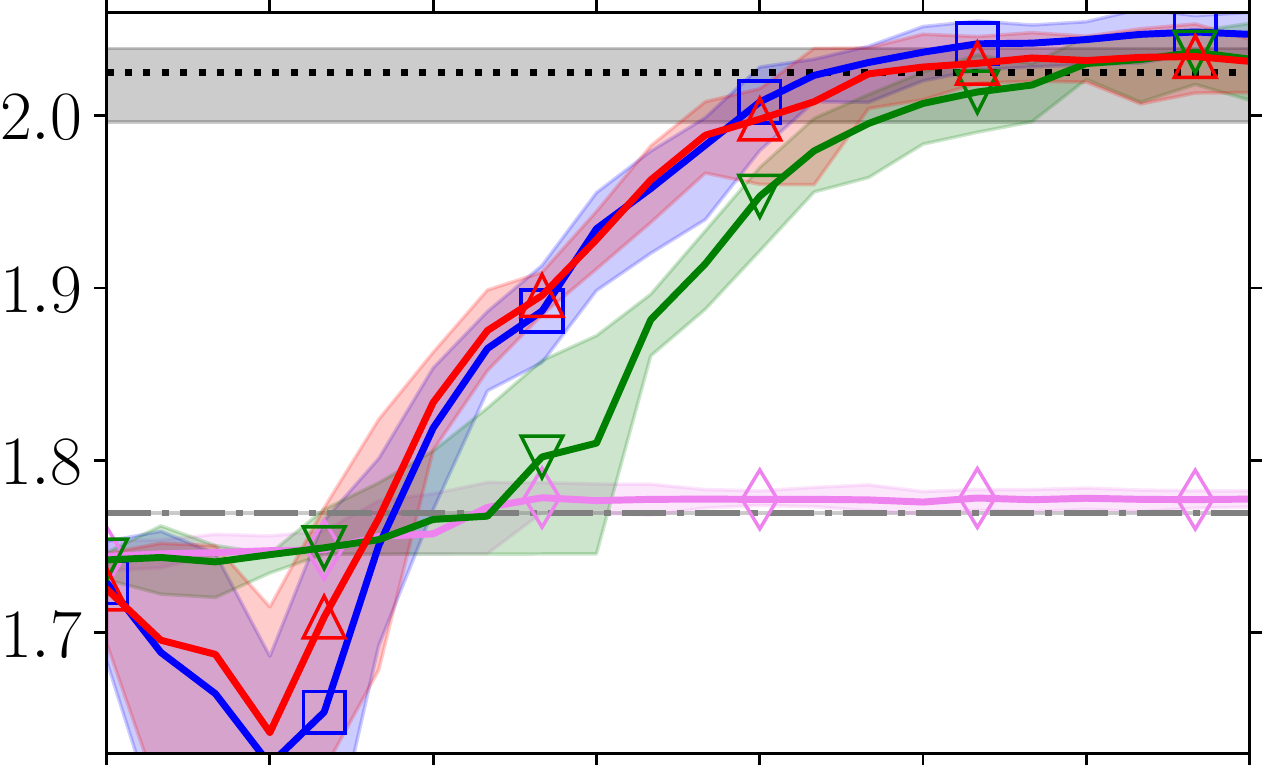}\hspace{1.28mm}
\\&
\includegraphics[scale=0.338]{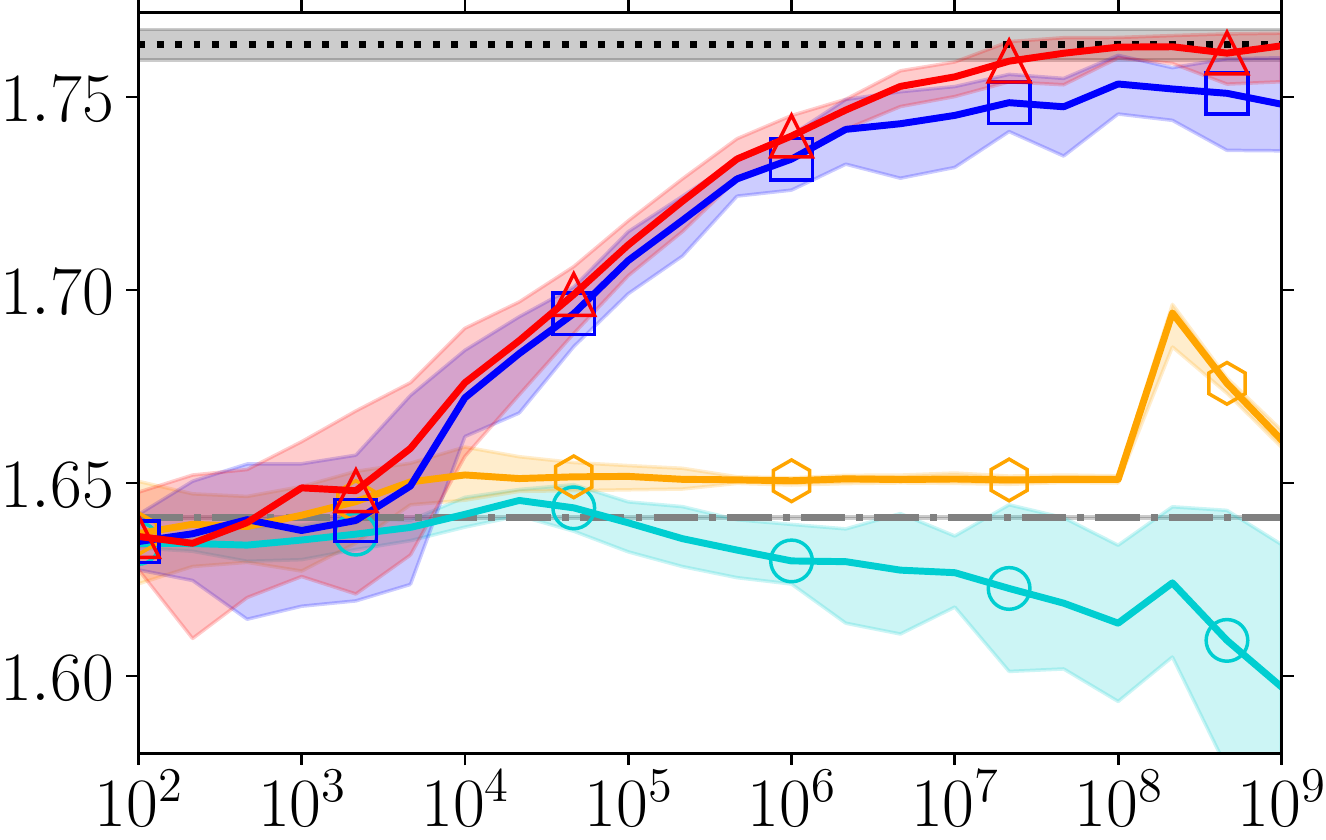} &
\includegraphics[scale=0.338]{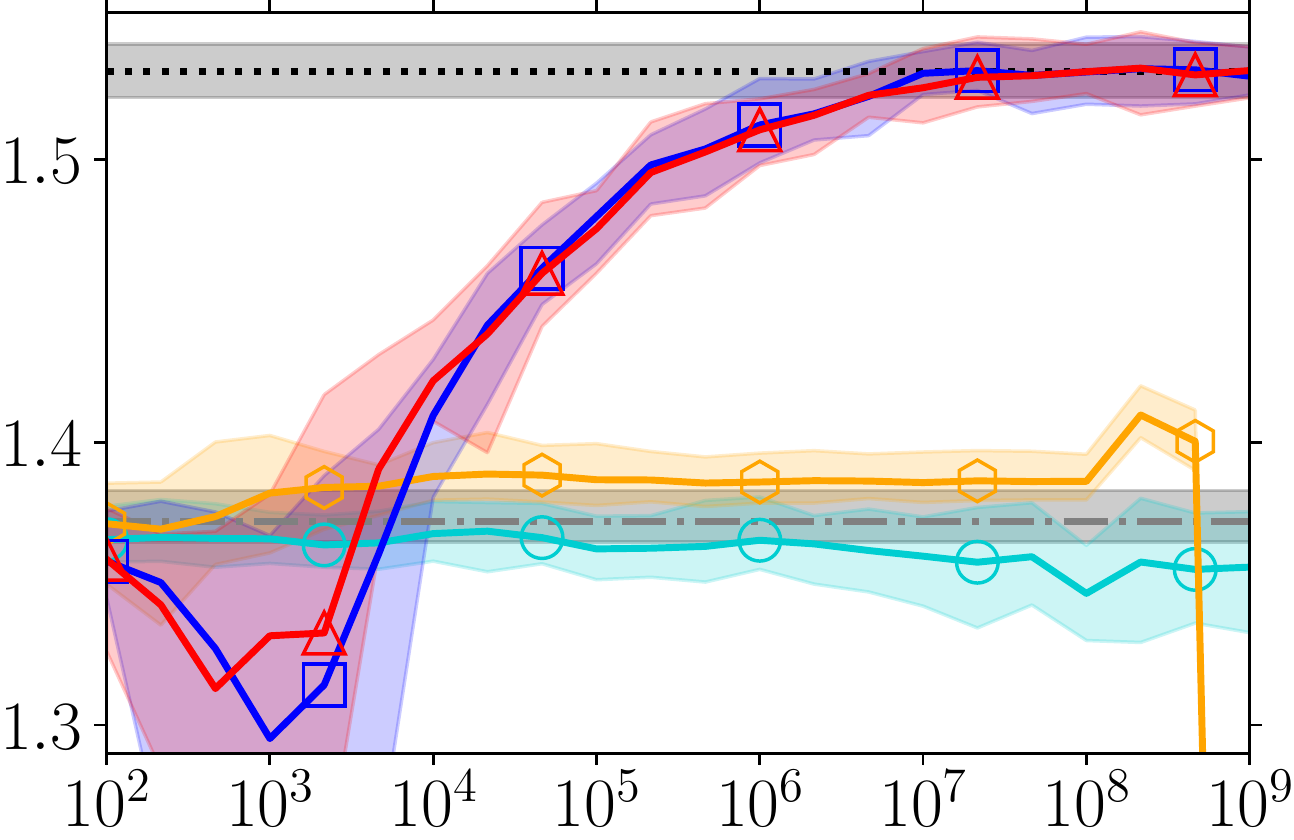} &
\includegraphics[scale=0.338]{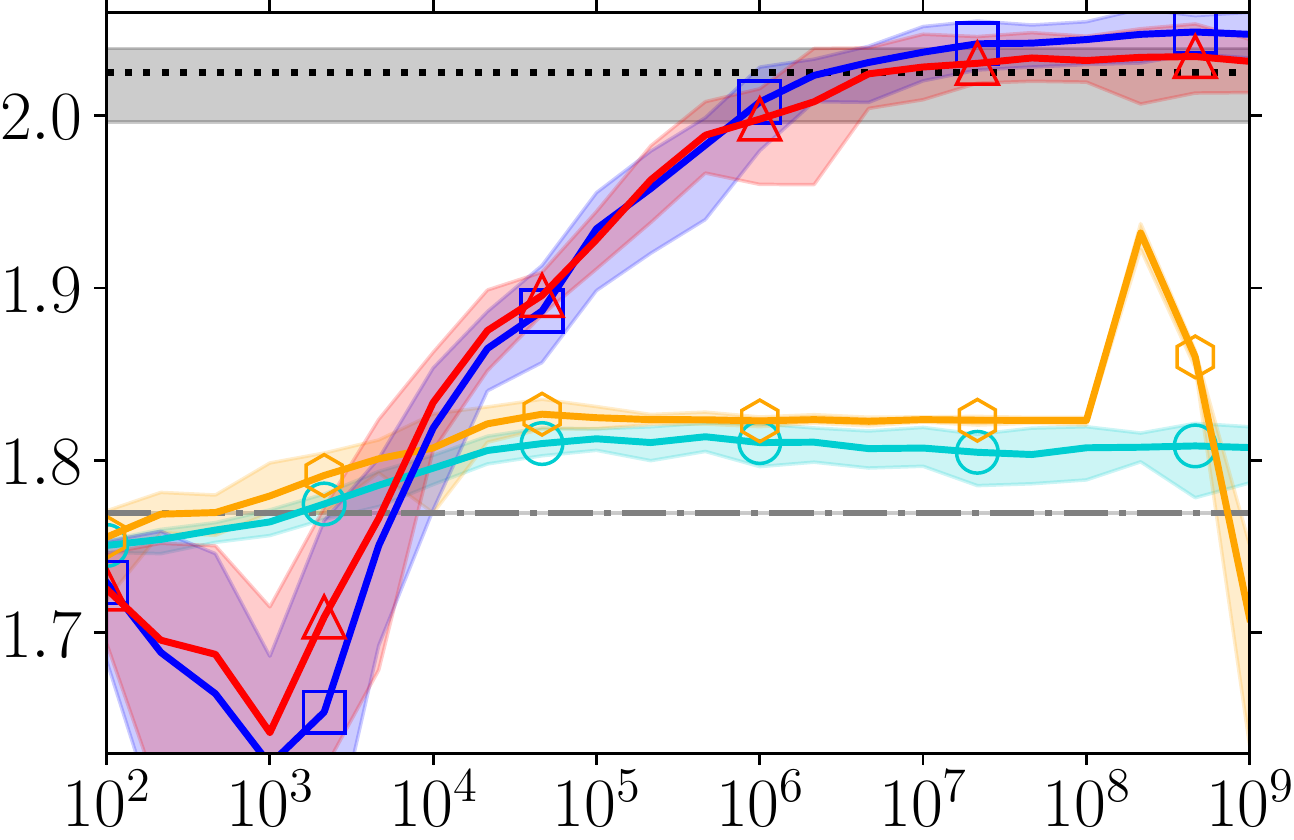}
\\
\multicolumn{4}{c}{
\includegraphics[scale=0.43]{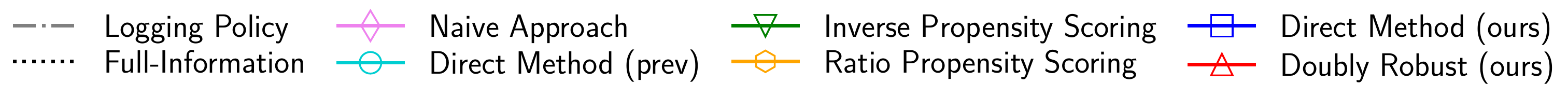}
} 
\end{tabular}
\caption{
Policy performance in terms of ECP (Eq.~\ref{eq:reward}) reached on three datasets and several settings.
Top row: top-5 setting with known $\alpha$ and $\beta$ bias parameters; middle row:
top-5 setting with estimated $\hat{\alpha}$ and $\hat{\beta}$;
bottom row: full-ranking setting (no cutoff) with known $\alpha$ and $\beta$ bias parameters.
Results are means over 20 independent runs, shaded areas indicate the 90\% confidence intervals; y-axis: ECP on the held-out test-set; x-axis: $N$ the number of displayed rankings in the simulated training set.
}
\label{fig:results}
\end{figure*}
}

{
\setlength{\tabcolsep}{0.000cm}
\begin{table*}[tp]
\centering
\caption{
Policy performance in terms of ECP (Eq.~\ref{eq:reward}) reached using different estimators in three different settings and three datasets for several $N$ values: the number of displayed rankings in the simulated training set.
In addition, the performance of the logging policy and a model trained on the ground-truth (Full-Information) are included to indicate estimated lower and upper bounds on possible performance respectively.
Top part: top-5 setting with known $\alpha$ and $\beta$ bias parameters; middle part:
top-5 setting with estimated $\hat{\alpha}$ and $\hat{\beta}$;
bottom part: full-ranking setting (no cutoff) with known $\alpha$ and $\beta$ bias parameters.
Reported numbers are averages over 20 independent runs evaluated on held-out test-sets, brackets display the standard deviation (logging policy deviation is ommitted since it did not vary between runs).
Bold numbers indicate the highest performance per setting, dataset and $N$ combination.
Statistical differences with our \ac{DR} estimator are measured via a two-sided student-t test,
$^{\tiny \blacktriangledown}$ and $^{\tiny \blacktriangle}$ indicate methods with significantly lower or higher ECP with $p<0.01$ respectively; additionally, $^{\tiny \triangledown}$ and $^{\tiny \triangle}$ indicate significant differences with $p<0.05$.
}
\label{tab:results}
\vspace{3\baselineskip}
\resizebox{\textwidth}{!}{
\begin{tabular}{ l c c c | c c c | c c c}
&\multicolumn{3}{c}{\footnotesize Yahoo! Webscope}
&\multicolumn{3}{c}{\footnotesize MSLR-WEB30k}
&\multicolumn{3}{c}{\footnotesize Istella}
\\ 
&\footnotesize $N = 10^4$& \footnotesize $N = 10^6$ & \footnotesize $N = 10^9$& \footnotesize $N = 10^4$ & \footnotesize $N = 10^6$ & \footnotesize $N = 10^9$& \footnotesize $N = 10^4$ & \footnotesize $N = 10^6$ & \footnotesize $N = 10^9$\\
 \toprule 
&\multicolumn{9}{c}{ \footnotesize \emph{Top-5 Setting with Known Bias Parameters} } \\ \midrule
\multicolumn{1}{l}{\footnotesize Logging } & \footnotesize 1.483 \phantom{\tiny(0.000)}\phantom{$^{\tiny -}$} & \footnotesize 1.483 \phantom{\tiny(0.000)}\phantom{$^{\tiny -}$} & \footnotesize 1.483 \phantom{\tiny(0.000)}\phantom{$^{\tiny -}$} & \footnotesize 1.110 \phantom{\tiny(0.007)}\phantom{$^{\tiny -}$} & \footnotesize 1.110 \phantom{\tiny(0.007)}\phantom{$^{\tiny -}$} & \footnotesize 1.110 \phantom{\tiny(0.007)}\phantom{$^{\tiny -}$} & \footnotesize 1.691 \phantom{\tiny(0.000)}\phantom{$^{\tiny -}$} & \footnotesize 1.691 \phantom{\tiny(0.000)}\phantom{$^{\tiny -}$} & \footnotesize 1.691 \phantom{\tiny(0.000)}\phantom{$^{\tiny -}$} \\ 
\multicolumn{1}{l}{\footnotesize Full-Info. } & \footnotesize 1.628 {\tiny(0.005)}$^{\tiny \blacktriangle}$ & \footnotesize 1.628 {\tiny(0.005)}$^{\tiny \blacktriangle}$ & \footnotesize 1.628 {\tiny(0.005)}$^{\tiny \blacktriangle}$ & \footnotesize 1.282 {\tiny(0.008)}$^{\tiny \blacktriangle}$ & \footnotesize 1.282 {\tiny(0.008)}$^{\tiny \blacktriangle}$ & \footnotesize 1.282 {\tiny(0.008)}\phantom{$^{\tiny -}$} & \footnotesize 1.994 {\tiny(0.021)}$^{\tiny \blacktriangle}$ & \footnotesize 1.994 {\tiny(0.021)}$^{\tiny \blacktriangle}$ & \footnotesize 1.994 {\tiny(0.021)}\phantom{$^{\tiny -}$} \\  \cline{2-10} 
\multicolumn{1}{l}{\footnotesize Naive } &\footnotesize 1.495 {\tiny(0.009)}$^{\tiny \blacktriangledown}$ &\footnotesize 1.498 {\tiny(0.002)}$^{\tiny \blacktriangledown}$ &\footnotesize 1.500 {\tiny(0.002)}$^{\tiny \blacktriangledown}$ &\footnotesize 1.114 {\tiny(0.009)}$^{\tiny \blacktriangledown}$ &\footnotesize 1.121 {\tiny(0.009)}$^{\tiny \blacktriangledown}$ &\footnotesize 1.121 {\tiny(0.007)}$^{\tiny \blacktriangledown}$ &\footnotesize 1.707 {\tiny(0.019)}$^{\tiny \blacktriangledown}$ &\footnotesize 1.736 {\tiny(0.006)}$^{\tiny \blacktriangledown}$ &\footnotesize 1.737 {\tiny(0.003)}$^{\tiny \blacktriangledown}$ \\
\multicolumn{1}{l}{\footnotesize DM (prev) } &\footnotesize 1.519 {\tiny(0.010)}$^{\tiny \blacktriangledown}$ &\footnotesize 1.540 {\tiny(0.009)}$^{\tiny \blacktriangledown}$ &\footnotesize 1.565 {\tiny(0.009)}$^{\tiny \blacktriangledown}$ &\footnotesize 1.139 {\tiny(0.022)}$^{\tiny \triangledown}$ &\footnotesize 1.214 {\tiny(0.016)}$^{\tiny \blacktriangledown}$ &\footnotesize 1.201 {\tiny(0.011)}$^{\tiny \blacktriangledown}$ &\footnotesize 1.539 {\tiny(0.072)}$^{\tiny \blacktriangledown}$ &\footnotesize 1.696 {\tiny(0.042)}$^{\tiny \blacktriangledown}$ &\footnotesize 1.904 {\tiny(0.009)}$^{\tiny \blacktriangledown}$ \\
\multicolumn{1}{l}{\footnotesize RPS } &\footnotesize 1.539 {\tiny(0.007)}$^{\tiny \blacktriangledown}$ &\footnotesize 1.581 {\tiny(0.004)}$^{\tiny \blacktriangledown}$ &\footnotesize 1.581 {\tiny(0.005)}$^{\tiny \blacktriangledown}$ &\footnotesize \textbf{1.187 {\tiny(0.009)}}$^{\tiny \blacktriangle}$ &\footnotesize 1.248 {\tiny(0.004)}$^{\tiny \blacktriangledown}$ &\footnotesize 1.250 {\tiny(0.006)}$^{\tiny \blacktriangledown}$ &\footnotesize \textbf{1.805 {\tiny(0.007)}}$^{\tiny \blacktriangle}$ &\footnotesize 1.875 {\tiny(0.008)}$^{\tiny \blacktriangledown}$ &\footnotesize 1.876 {\tiny(0.007)}$^{\tiny \blacktriangledown}$ \\
\multicolumn{1}{l}{\footnotesize IPS } &\footnotesize 1.486 {\tiny(0.007)}$^{\tiny \blacktriangledown}$ &\footnotesize 1.552 {\tiny(0.006)}$^{\tiny \blacktriangledown}$ &\footnotesize 1.590 {\tiny(0.004)}$^{\tiny \blacktriangledown}$ &\footnotesize 1.110 {\tiny(0.007)}$^{\tiny \blacktriangledown}$ &\footnotesize 1.187 {\tiny(0.008)}$^{\tiny \blacktriangledown}$ &\footnotesize 1.237 {\tiny(0.007)}$^{\tiny \blacktriangledown}$ &\footnotesize 1.710 {\tiny(0.032)}$^{\tiny \blacktriangledown}$ &\footnotesize 1.879 {\tiny(0.014)}$^{\tiny \blacktriangledown}$ &\footnotesize 1.950 {\tiny(0.014)}$^{\tiny \blacktriangledown}$ \\
\multicolumn{1}{l}{\footnotesize DM (ours) } &\footnotesize 1.542 {\tiny(0.006)}$^{\tiny \triangledown}$ &\footnotesize 1.594 {\tiny(0.004)}$^{\tiny \blacktriangledown}$ &\footnotesize 1.616 {\tiny(0.003)}$^{\tiny \blacktriangledown}$ &\footnotesize 1.143 {\tiny(0.030)}\phantom{$^{\tiny -}$} &\footnotesize 1.252 {\tiny(0.008)}$^{\tiny \blacktriangledown}$ &\footnotesize 1.271 {\tiny(0.007)}$^{\tiny \blacktriangledown}$ &\footnotesize 1.721 {\tiny(0.035)}$^{\tiny \blacktriangledown}$ &\footnotesize 1.898 {\tiny(0.023)}$^{\tiny \blacktriangledown}$ &\footnotesize 1.966 {\tiny(0.017)}$^{\tiny \blacktriangledown}$ \\
\multicolumn{1}{l}{\footnotesize DR (ours) } &\footnotesize \textbf{1.548 {\tiny(0.011)}}\phantom{$^{\tiny -}$} &\footnotesize \textbf{1.602 {\tiny(0.004)}}\phantom{$^{\tiny -}$} &\footnotesize \textbf{1.623 {\tiny(0.003)}}\phantom{$^{\tiny -}$} &\footnotesize 1.157 {\tiny(0.026)}\phantom{$^{\tiny -}$} &\footnotesize \textbf{1.263 {\tiny(0.008)}}\phantom{$^{\tiny -}$} &\footnotesize \textbf{1.281 {\tiny(0.006)}}\phantom{$^{\tiny -}$} &\footnotesize 1.769 {\tiny(0.037)}\phantom{$^{\tiny -}$} &\footnotesize \textbf{1.952 {\tiny(0.024)}}\phantom{$^{\tiny -}$} &\footnotesize \textbf{2.004 {\tiny(0.009)}}\phantom{$^{\tiny -}$} \\
\midrule
&\multicolumn{9}{c}{ \footnotesize \emph{Top-5 Setting with Estimated Bias Parameters} } \\ \midrule
\multicolumn{1}{l}{\footnotesize Logging } & \footnotesize 1.483 \phantom{\tiny(0.000)}\phantom{$^{\tiny -}$} & \footnotesize 1.483 \phantom{\tiny(0.000)}\phantom{$^{\tiny -}$} & \footnotesize 1.483 \phantom{\tiny(0.000)}\phantom{$^{\tiny -}$} & \footnotesize 1.110 \phantom{\tiny(0.007)}\phantom{$^{\tiny -}$} & \footnotesize 1.110 \phantom{\tiny(0.007)}\phantom{$^{\tiny -}$} & \footnotesize 1.110 \phantom{\tiny(0.007)}\phantom{$^{\tiny -}$} & \footnotesize 1.691 \phantom{\tiny(0.000)}\phantom{$^{\tiny -}$} & \footnotesize 1.691 \phantom{\tiny(0.000)}\phantom{$^{\tiny -}$} & \footnotesize 1.691 \phantom{\tiny(0.000)}\phantom{$^{\tiny -}$} \\ 
\multicolumn{1}{l}{\footnotesize Full-Info. } & \footnotesize 1.628 {\tiny(0.005)}$^{\tiny \blacktriangle}$ & \footnotesize 1.628 {\tiny(0.005)}$^{\tiny \blacktriangle}$ & \footnotesize 1.628 {\tiny(0.005)}$^{\tiny \blacktriangle}$ & \footnotesize 1.282 {\tiny(0.008)}$^{\tiny \blacktriangle}$ & \footnotesize 1.282 {\tiny(0.008)}$^{\tiny \blacktriangle}$ & \footnotesize 1.282 {\tiny(0.008)}\phantom{$^{\tiny -}$} & \footnotesize 1.994 {\tiny(0.021)}$^{\tiny \blacktriangle}$ & \footnotesize 1.994 {\tiny(0.021)}$^{\tiny \blacktriangle}$ & \footnotesize 1.994 {\tiny(0.021)}\phantom{$^{\tiny -}$} \\  \cline{2-10} 
\multicolumn{1}{l}{\footnotesize Naive } &\footnotesize 1.492 {\tiny(0.006)}$^{\tiny \blacktriangledown}$ &\footnotesize 1.498 {\tiny(0.003)}$^{\tiny \blacktriangledown}$ &\footnotesize 1.499 {\tiny(0.003)}$^{\tiny \blacktriangledown}$ &\footnotesize 1.112 {\tiny(0.008)}$^{\tiny \blacktriangledown}$ &\footnotesize 1.121 {\tiny(0.008)}$^{\tiny \blacktriangledown}$ &\footnotesize 1.122 {\tiny(0.006)}$^{\tiny \blacktriangledown}$ &\footnotesize 1.703 {\tiny(0.018)}\phantom{$^{\tiny -}$} &\footnotesize 1.739 {\tiny(0.005)}$^{\tiny \blacktriangledown}$ &\footnotesize 1.735 {\tiny(0.011)}$^{\tiny \blacktriangledown}$ \\
\multicolumn{1}{l}{\footnotesize DM (prev) } &\footnotesize 1.540 {\tiny(0.007)}$^{\tiny \blacktriangledown}$ &\footnotesize 1.570 {\tiny(0.009)}$^{\tiny \blacktriangledown}$ &\footnotesize 1.558 {\tiny(0.005)}$^{\tiny \blacktriangledown}$ &\footnotesize 1.153 {\tiny(0.025)}\phantom{$^{\tiny -}$} &\footnotesize 1.209 {\tiny(0.016)}$^{\tiny \blacktriangledown}$ &\footnotesize 1.191 {\tiny(0.011)}$^{\tiny \blacktriangledown}$ &\footnotesize 1.662 {\tiny(0.039)}$^{\tiny \blacktriangledown}$ &\footnotesize 1.769 {\tiny(0.049)}$^{\tiny \blacktriangledown}$ &\footnotesize 1.910 {\tiny(0.014)}$^{\tiny \blacktriangledown}$ \\
\multicolumn{1}{l}{\footnotesize RPS } &\footnotesize 1.537 {\tiny(0.007)}$^{\tiny \blacktriangledown}$ &\footnotesize 1.582 {\tiny(0.003)}$^{\tiny \blacktriangledown}$ &\footnotesize 1.581 {\tiny(0.005)}$^{\tiny \blacktriangledown}$ &\footnotesize \textbf{1.188 {\tiny(0.011)}}$^{\tiny \blacktriangle}$ &\footnotesize 1.245 {\tiny(0.006)}$^{\tiny \blacktriangledown}$ &\footnotesize 1.250 {\tiny(0.006)}$^{\tiny \blacktriangledown}$ &\footnotesize \textbf{1.811 {\tiny(0.009)}}$^{\tiny \blacktriangle}$ &\footnotesize 1.875 {\tiny(0.005)}$^{\tiny \blacktriangledown}$ &\footnotesize 1.876 {\tiny(0.007)}$^{\tiny \blacktriangledown}$ \\
\multicolumn{1}{l}{\footnotesize IPS } &\footnotesize 1.488 {\tiny(0.009)}$^{\tiny \blacktriangledown}$ &\footnotesize 1.555 {\tiny(0.006)}$^{\tiny \blacktriangledown}$ &\footnotesize 1.593 {\tiny(0.003)}$^{\tiny \blacktriangledown}$ &\footnotesize 1.111 {\tiny(0.007)}$^{\tiny \blacktriangledown}$ &\footnotesize 1.182 {\tiny(0.007)}$^{\tiny \blacktriangledown}$ &\footnotesize 1.233 {\tiny(0.009)}$^{\tiny \blacktriangledown}$ &\footnotesize 1.710 {\tiny(0.033)}\phantom{$^{\tiny -}$} &\footnotesize 1.874 {\tiny(0.013)}$^{\tiny \blacktriangledown}$ &\footnotesize 1.944 {\tiny(0.019)}$^{\tiny \blacktriangledown}$ \\
\multicolumn{1}{l}{\footnotesize DM (ours) } &\footnotesize 1.546 {\tiny(0.007)}\phantom{$^{\tiny -}$} &\footnotesize 1.591 {\tiny(0.010)}$^{\tiny \blacktriangledown}$ &\footnotesize 1.619 {\tiny(0.005)}$^{\tiny \triangledown}$ &\footnotesize 1.138 {\tiny(0.021)}\phantom{$^{\tiny -}$} &\footnotesize 1.251 {\tiny(0.011)}$^{\tiny \blacktriangledown}$ &\footnotesize 1.272 {\tiny(0.008)}$^{\tiny \blacktriangledown}$ &\footnotesize 1.662 {\tiny(0.047)}$^{\tiny \blacktriangledown}$ &\footnotesize 1.878 {\tiny(0.043)}$^{\tiny \blacktriangledown}$ &\footnotesize 1.961 {\tiny(0.010)}$^{\tiny \blacktriangledown}$ \\
\multicolumn{1}{l}{\footnotesize DR (ours) } &\footnotesize \textbf{1.548 {\tiny(0.006)}}\phantom{$^{\tiny -}$} &\footnotesize \textbf{1.603 {\tiny(0.003)}}\phantom{$^{\tiny -}$} &\footnotesize \textbf{1.623 {\tiny(0.004)}}\phantom{$^{\tiny -}$} &\footnotesize 1.149 {\tiny(0.028)}\phantom{$^{\tiny -}$} &\footnotesize \textbf{1.260 {\tiny(0.005)}}\phantom{$^{\tiny -}$} &\footnotesize \textbf{1.282 {\tiny(0.005)}}\phantom{$^{\tiny -}$} &\footnotesize 1.724 {\tiny(0.057)}\phantom{$^{\tiny -}$} &\footnotesize \textbf{1.945 {\tiny(0.014)}}\phantom{$^{\tiny -}$} &\footnotesize \textbf{1.994 {\tiny(0.012)}}\phantom{$^{\tiny -}$} \\
\midrule
&\multicolumn{9}{c}{ \footnotesize \emph{Full-Ranking Setting with Known Bias Parameters} } \\ \midrule
\multicolumn{1}{l}{\footnotesize Logging } & \footnotesize 1.641 \phantom{\tiny(0.000)}\phantom{$^{\tiny -}$} & \footnotesize 1.641 \phantom{\tiny(0.000)}\phantom{$^{\tiny -}$} & \footnotesize 1.641 \phantom{\tiny(0.000)}\phantom{$^{\tiny -}$} & \footnotesize 1.372 \phantom{\tiny(0.006)}\phantom{$^{\tiny -}$} & \footnotesize 1.372 \phantom{\tiny(0.006)}\phantom{$^{\tiny -}$} & \footnotesize 1.372 \phantom{\tiny(0.006)}\phantom{$^{\tiny -}$} & \footnotesize 1.769 \phantom{\tiny(0.000)}\phantom{$^{\tiny -}$} & \footnotesize 1.769 \phantom{\tiny(0.000)}\phantom{$^{\tiny -}$} & \footnotesize 1.769 \phantom{\tiny(0.000)}\phantom{$^{\tiny -}$} \\ 
\multicolumn{1}{l}{\footnotesize Full-Info. } & \footnotesize 1.764 {\tiny(0.004)}$^{\tiny \blacktriangle}$ & \footnotesize 1.764 {\tiny(0.004)}$^{\tiny \blacktriangle}$ & \footnotesize 1.764 {\tiny(0.004)}\phantom{$^{\tiny -}$} & \footnotesize 1.531 {\tiny(0.007)}$^{\tiny \blacktriangle}$ & \footnotesize 1.531 {\tiny(0.007)}$^{\tiny \blacktriangle}$ & \footnotesize 1.531 {\tiny(0.007)}\phantom{$^{\tiny -}$} & \footnotesize 2.025 {\tiny(0.013)}$^{\tiny \blacktriangle}$ & \footnotesize 2.025 {\tiny(0.013)}$^{\tiny \blacktriangle}$ & \footnotesize 2.025 {\tiny(0.013)}\phantom{$^{\tiny -}$} \\  \cline{2-10} 
\multicolumn{1}{l}{\footnotesize Naive } &\footnotesize 1.642 {\tiny(0.007)}$^{\tiny \blacktriangledown}$ &\footnotesize 1.642 {\tiny(0.002)}$^{\tiny \blacktriangledown}$ &\footnotesize 1.645 {\tiny(0.002)}$^{\tiny \blacktriangledown}$ &\footnotesize 1.369 {\tiny(0.007)}$^{\tiny \blacktriangledown}$ &\footnotesize 1.374 {\tiny(0.006)}$^{\tiny \blacktriangledown}$ &\footnotesize 1.374 {\tiny(0.006)}$^{\tiny \blacktriangledown}$ &\footnotesize 1.757 {\tiny(0.014)}$^{\tiny \blacktriangledown}$ &\footnotesize 1.777 {\tiny(0.003)}$^{\tiny \blacktriangledown}$ &\footnotesize 1.777 {\tiny(0.003)}$^{\tiny \blacktriangledown}$ \\
\multicolumn{1}{l}{\footnotesize DM (prev) } &\footnotesize 1.642 {\tiny(0.003)}$^{\tiny \blacktriangledown}$ &\footnotesize 1.630 {\tiny(0.005)}$^{\tiny \blacktriangledown}$ &\footnotesize 1.597 {\tiny(0.036)}$^{\tiny \blacktriangledown}$ &\footnotesize 1.368 {\tiny(0.007)}$^{\tiny \blacktriangledown}$ &\footnotesize 1.365 {\tiny(0.010)}$^{\tiny \blacktriangledown}$ &\footnotesize 1.356 {\tiny(0.015)}$^{\tiny \blacktriangledown}$ &\footnotesize 1.795 {\tiny(0.006)}$^{\tiny \blacktriangledown}$ &\footnotesize 1.810 {\tiny(0.006)}$^{\tiny \blacktriangledown}$ &\footnotesize 1.807 {\tiny(0.011)}$^{\tiny \blacktriangledown}$ \\
\multicolumn{1}{l}{\footnotesize RPS } &\footnotesize 1.652 {\tiny(0.004)}$^{\tiny \blacktriangledown}$ &\footnotesize 1.651 {\tiny(0.001)}$^{\tiny \blacktriangledown}$ &\footnotesize 1.661 {\tiny(0.002)}$^{\tiny \blacktriangledown}$ &\footnotesize 1.388 {\tiny(0.006)}$^{\tiny \blacktriangledown}$ &\footnotesize 1.386 {\tiny(0.006)}$^{\tiny \blacktriangledown}$ &\footnotesize 0.587 {\tiny(0.004)}$^{\tiny \blacktriangledown}$ &\footnotesize 1.807 {\tiny(0.018)}$^{\tiny \blacktriangledown}$ &\footnotesize 1.823 {\tiny(0.002)}$^{\tiny \blacktriangledown}$ &\footnotesize 1.708 {\tiny(0.044)}$^{\tiny \blacktriangledown}$ \\
\multicolumn{1}{l}{\footnotesize IPS } &\footnotesize 1.643 {\tiny(0.011)}$^{\tiny \blacktriangledown}$ &\footnotesize 1.716 {\tiny(0.003)}$^{\tiny \blacktriangledown}$ &\footnotesize 1.762 {\tiny(0.002)}\phantom{$^{\tiny -}$} &\footnotesize 1.368 {\tiny(0.009)}$^{\tiny \blacktriangledown}$ &\footnotesize 1.468 {\tiny(0.006)}$^{\tiny \blacktriangledown}$ &\footnotesize \textbf{1.532 {\tiny(0.007)}}\phantom{$^{\tiny -}$} &\footnotesize 1.766 {\tiny(0.022)}$^{\tiny \blacktriangledown}$ &\footnotesize 1.953 {\tiny(0.015)}$^{\tiny \blacktriangledown}$ &\footnotesize 2.033 {\tiny(0.013)}\phantom{$^{\tiny -}$} \\
\multicolumn{1}{l}{\footnotesize DM (ours) } &\footnotesize 1.672 {\tiny(0.009)}\phantom{$^{\tiny -}$} &\footnotesize 1.734 {\tiny(0.005)}$^{\tiny \blacktriangledown}$ &\footnotesize 1.748 {\tiny(0.009)}$^{\tiny \blacktriangledown}$ &\footnotesize 1.409 {\tiny(0.021)}\phantom{$^{\tiny -}$} &\footnotesize \textbf{1.512 {\tiny(0.010)}}\phantom{$^{\tiny -}$} &\footnotesize 1.529 {\tiny(0.007)}\phantom{$^{\tiny -}$} &\footnotesize 1.819 {\tiny(0.028)}\phantom{$^{\tiny -}$} &\footnotesize \textbf{2.008 {\tiny(0.013)}}\phantom{$^{\tiny -}$} &\footnotesize \textbf{2.047 {\tiny(0.009)}}$^{\tiny \blacktriangle}$ \\
\multicolumn{1}{l}{\footnotesize DR (ours) } &\footnotesize \textbf{1.676 {\tiny(0.010)}}\phantom{$^{\tiny -}$} &\footnotesize \textbf{1.740 {\tiny(0.004)}}\phantom{$^{\tiny -}$} &\footnotesize \textbf{1.763 {\tiny(0.004)}}\phantom{$^{\tiny -}$} &\footnotesize \textbf{1.422 {\tiny(0.017)}}\phantom{$^{\tiny -}$} &\footnotesize 1.510 {\tiny(0.009)}\phantom{$^{\tiny -}$} &\footnotesize 1.531 {\tiny(0.006)}\phantom{$^{\tiny -}$} &\footnotesize \textbf{1.834 {\tiny(0.019)}}\phantom{$^{\tiny -}$} &\footnotesize 1.998 {\tiny(0.019)}\phantom{$^{\tiny -}$} &\footnotesize 2.031 {\tiny(0.012)}\phantom{$^{\tiny -}$} \\
\bottomrule

\end{tabular}
}
\vspace{\baselineskip}
\end{table*}
}

\section{Results}
\label{sec:results}

Our main experimental results are displayed in Figure~\ref{fig:results} and Table~\ref{tab:results}.
Both display performance reached, in terms of the ECP metric (Eq.~\ref{eq:reward}), for different estimators on varying amounts of simulated interaction data; and
both are split in three rows, each indicating the results of one of the three simulated settings.
The displayed results are means over twenty independent runs,  90\% confidence intervals are visualized in Figure~\ref{fig:results} so that meaningful differences can be recognized.
Furthermore, Table~\ref{tab:results} displays standard deviations and statistical significant performance differences with \ac{DR} using a two-sided student's t-test~\citep{student1908probable}.

\subsection{Performance of Inverse Propensity Scoring}
To begin our analysis, we consider the performance of \ac{IPS};
In Figure~\ref{fig:results}, we see that in both top-5 settings \ac{IPS} is unable to reach optimal ECP when $N \leq 10^9$ on any of the datasets, an observation also made in previous work~\cite{oosterhuis2021onlinecounterltr}.
In the top-5 setting with known bias parameters, \ac{IPS} is theoretically proven to be unbiased and will converge at optimal ECP as $N \rightarrow \infty$.
Consequently, we can conclude that it is high variance which prevents us from observing \ac{IPS}'s convergence in the top-row of Figure~\ref{fig:results}
\footnote{The effect of clipping can be excluded since our clipping strategy has no effect in our experimental setting when $N = 10^9$.}.
This observation illustrates the importance of variance reduction: it is not bias but high variance that prevents \ac{IPS} from reaching optimal performance with feasible amounts of interaction data.
In contrast with the two top-rows of Figure~\ref{fig:results}, the bottom row shows that \ac{IPS} can reach optimal ECP on all datasets in the full-ranking setting.
A plausible explanation is that interactions on a complete ranking provide much more information than when only the top-5 can be interacted with.
Possibly, the item-selection-bias in the top-5 settings greatly increase the variance of \ac{IPS} due to the \emph{winner-takes-all} behavior described in Section~\ref{sec:regression}.
Overall, we see that while \ac{IPS} can approximate optimal ECP in the full-ranking setting with reasonable amounts of data, its variance prevents it from reaching good ECP in top-5 settings even when given an enormous number of interactions.

\subsection{Performance of Novel Direct Method and Doubly-Robust Estimators}
Next we consider whether our \ac{DR} estimator provides an improvement over the performance of \ac{IPS}.
Figure~\ref{fig:results} reveals that this is clearly the case: in all settings it outperforms \ac{IPS} when $N \approx 10^4$ and only in the full-ranking setting does \ac{IPS} catch up around $N \approx 10^7$.
Moreover, \ac{DR} always has a higher or comparable mean ECP than \ac{IPS} when $N \geq 10^4$, across all datasets and settings.
Table~\ref{tab:results} does not report a single instance of \ac{IPS} significantly outperforming \ac{DR} but many instances where \ac{DR} significantly outperforms \ac{IPS}.
In all of the top-5 settings, the ECP of \ac{DR} when $N=10^6$ is not reached by \ac{IPS} when $N=10^9$, regardless of whether bias parameters are known or estimated.
Therefore the \ac{DR} appears to provide an increase in data-efficiency over \ac{IPS} of a factor greater than \numprint{1,000} in all of the top-5 settings.
We thus confidently conclude that \ac{DR} provides significantly and considerably higher performance than state-of-the-art \ac{IPS}, given that $N \geq 10^4$, which in top-5 settings leads to an enormous increase in data-efficiency.

Subsequently, we compare the performance of our \ac{DM} estimator with \ac{IPS}.
In Figure~\ref{fig:results}, we see that in some cases \ac{DM} has substantially better ECP than \ac{IPS}, particularly in the top-5 settings on  Yahoo!\ and MSLR.
Yet, we also see that, on all settings on the Istella dataset, the ECP differences between \ac{DM} and \ac{IPS} are much smaller; and on the full-ranking setting on Yahoo!, \ac{DM} appears to converge at noticeably worse ECP than \ac{IPS}.
These results are quite surprising as it shows that this simple but previously-unconsidered approach is actually a competitive baseline to \ac{IPS}.
We conclude that \ac{DM} appears preferable over \ac{IPS} in top-5 settings where not all items can be displayed at once, but not necessarily in full-ranking settings.

Lastly, our comparison considers both our novel \ac{DM} and \ac{DR} estimators.
In Figure~\ref{fig:results}, we see that overall \ac{DM} has lower ECP than \ac{DR}, but in some cases it has comparable or not significantly different ECP.
Table~\ref{tab:results} reveals that in both the top-5 settings and the full-ranking setting on Yahoo!, \ac{DR} has a significantly higher ECP than \ac{DM}, even though these differences are smaller than compared with \ac{IPS}.
This result confirms that the \ac{DR} approach can indeed effectively use click data to correct for the regression mistakes of \ac{DM}.
It appears that this is especially the case on the Istella dataset, where in both top-5 settings there is a considerable performance difference between our \ac{DR} and \ac{DM} estimators; here Figure~\ref{fig:results} shows the ECP reached by our \ac{DM} when $N=10^9$ is reached by our \ac{DR} when $N\approx10^7$.
Notably, \ac{DM} appears to converge on suboptimal ECP in the full-ranking setting on Yahoo!\ and in both top-5 settings on MSLR, indicating that its unbiasedness criteria (Eq.~\ref{eq:regressionbiascondition}) are not met in these situations.
In stark contrast, our \ac{DR} estimator reaches near-optimal ECP in all tested scenarios, which corresponds with its much more robust unbiasedness criteria.
Overall, our results indicate that \ac{DR} in the majority of cases significantly outperforms \ac{DM}.

Our observations seem to confirm several of our expectations from our theoretical analysis:
The inability of \ac{IPS} to reach optimal ECP when $N=10^9$ in top-5 settings confirms that variance is its biggest obstacle.
The large increase of \ac{DM} over \ac{IPS} seems to indicate that the usage of regression estimates provides a large reduction in variance.
The cases of suboptimal convergence of \ac{DM} show that its impractical unbiasedness criteria are infeasible in some of our experimental settings.
Finally, in all settings and datasets, our \ac{DR} estimator has significantly better or comparable ECP to \ac{DM} and \ac{IPS}, while always converging near optimal ECP.
This observation shows that \ac{DR} effectively combines the variance reduction of \ac{DM} with the more feasible unbiasedness criteria of \ac{IPS}, and clearly provides the most robust and highest performance across our tested settings and datasets.

Despite the large aforementioned advantages, we note that a downside of our \ac{DM} and \ac{DR} estimators is that they provide low ECP in some settings when $N \leq 10^4$.
It appears that our early-stopping strategy, which is very effective for \ac{IPS}, is not able to handle incorrect regression estimates very well.
Future work could investigate whether this could be remedied with safe deployment strategies~\citep{jagerman2020safe, oosterhuis2021robust}, that prevent deploying models with uncertain performance.
However, it seems doubtful to us that in a real-world setting such little data is available that $N \leq 10^4$.
Nonetheless, our results show that \ac{DM} and \ac{DR} are less resilient to tiny amounts of training data than \ac{IPS}.

\subsection{Comparison with other Baselines}

Finally, our comparison also includes other baseline methods: the naive estimator, \ac{RPS} and \ac{DM} from previous work.
We note that all of these methods are biased in our setting: the naive estimator explicitly ignores position-bias, \ac{RPS} trades bias for less variance and the \ac{DM} from previous work ignores trust-bias.
Unsurprisingly, Figure~\ref{fig:results} and Table~\ref{tab:results} show that they are all unable to converge at optimal ECP in any of the settings.
The effect of trust-bias appears particularly large in the full-ranking setting, where none of these baselines are able to substantially improve ECP over the logging policy.
The ECP of \ac{RPS} appears very sensitive to propensity clipping: when $N\approx10^9$ and clipping no longer has effect its performance completely drops.
Nevertheless, these baselines show that often a decrease in variance can be favourable over unbiasedness, as some of them provide higher ECP than the unbiased \ac{IPS} estimator in the top-5 settings on Yahoo!\ and MSLR, especially when $N$ is small.
Regardless, due their bias, they are unable to combine optimal convergence with low variance.
It appears that our \ac{DR} estimator is the only method that effectively combines these properties.

\begin{figure*}[t]
\centering
\begin{tabular}{@{}l @{}l @{}l}
 \multicolumn{1}{c}{\hspace{0.12cm} \footnotesize Yahoo! Webscope}
&
 \multicolumn{1}{c}{\hspace{0.0cm} \footnotesize MSLR-WEB30k}
&
 \multicolumn{1}{c}{\hspace{-0.15cm} \footnotesize Istella}
\\
\includegraphics[scale=0.34]{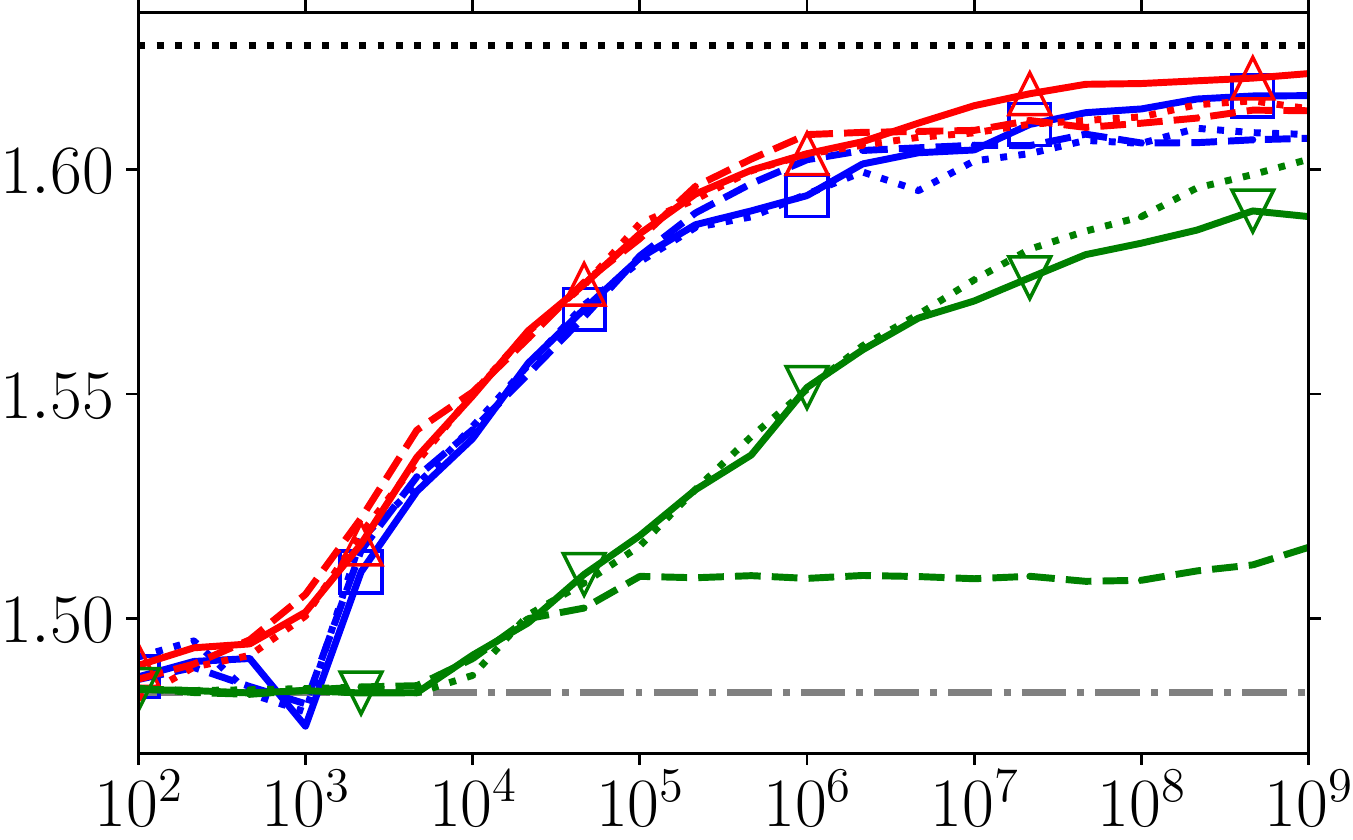}\hspace{1.28mm} &
\includegraphics[scale=0.34]{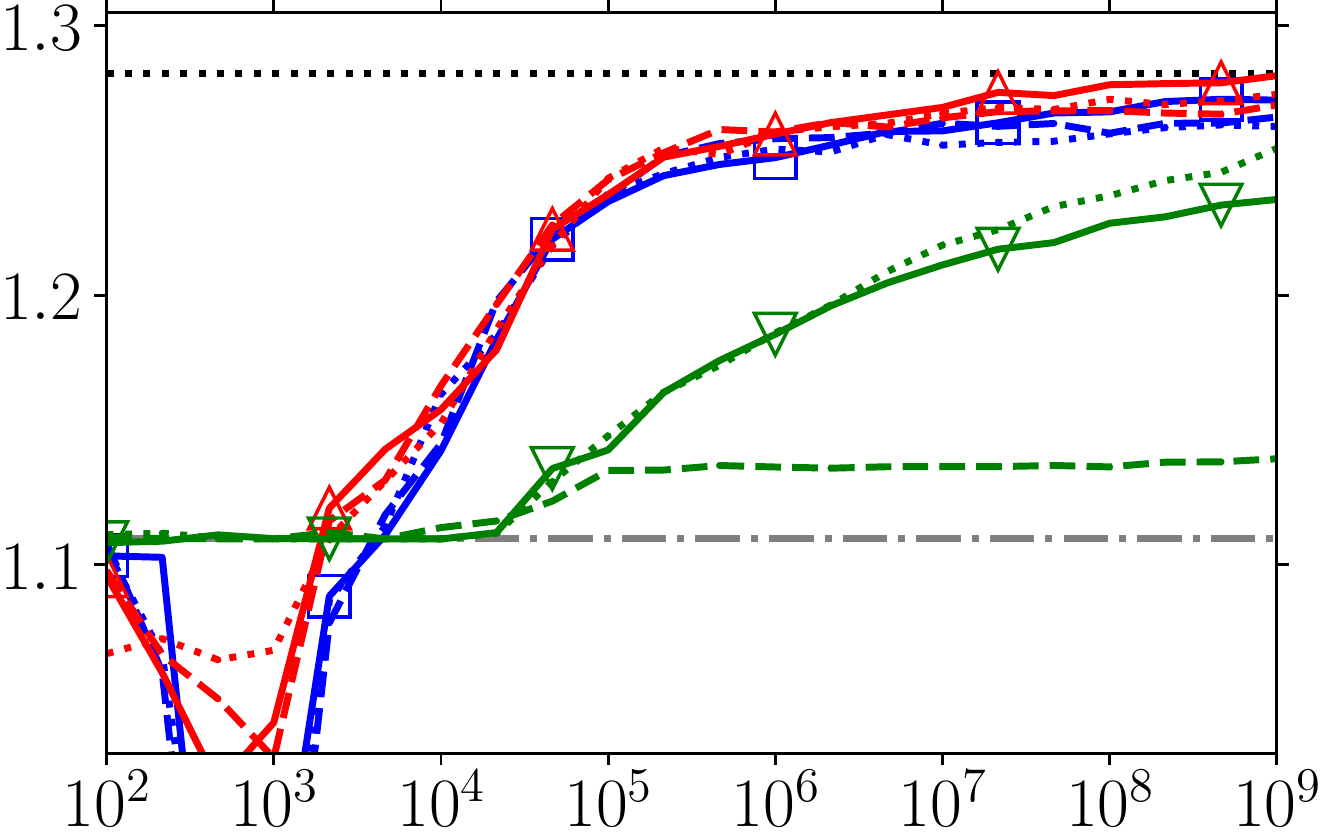}\hspace{1.28mm} &
\includegraphics[scale=0.34]{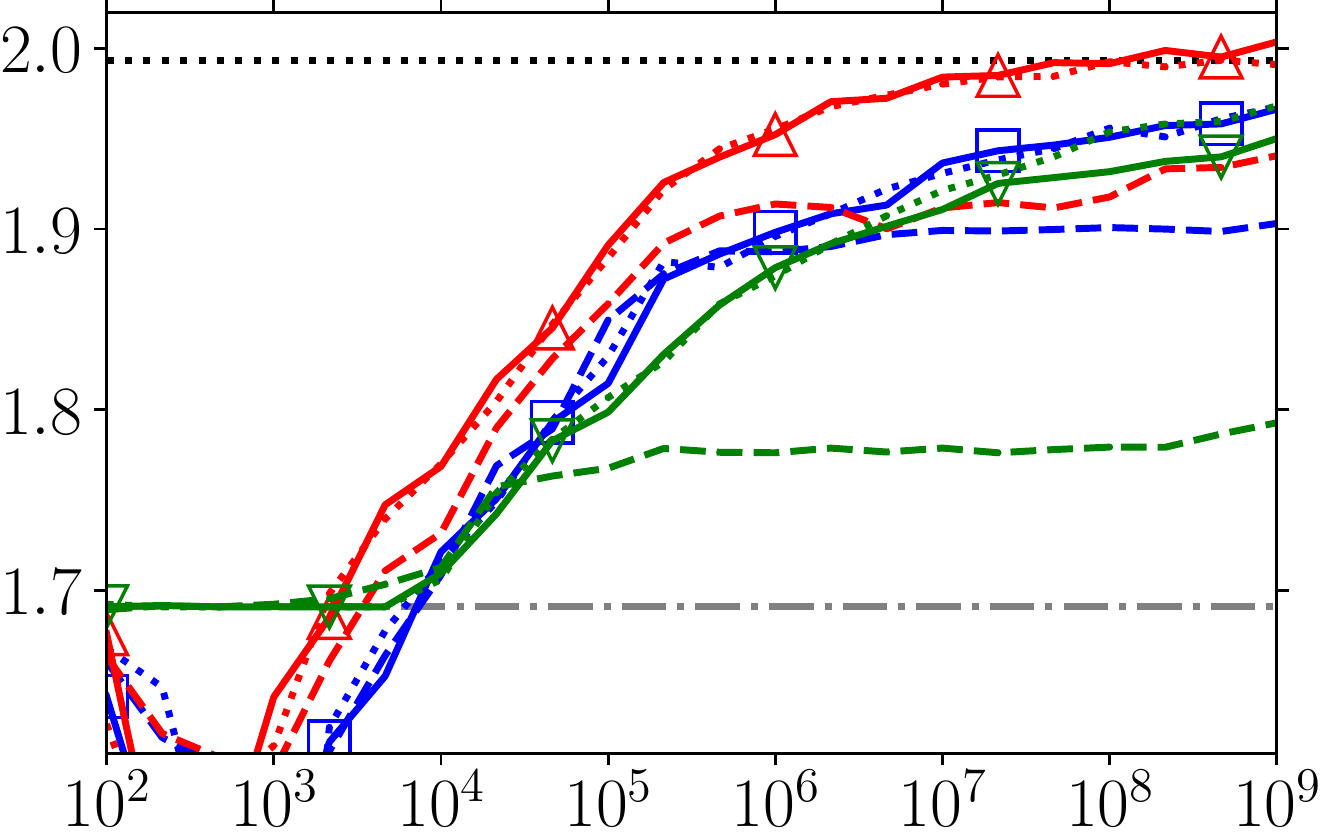}\hspace{1.28mm}
\\
 \multicolumn{1}{c}{\footnotesize Number of Displayed Queries ($N$)}
 &
 \multicolumn{1}{c}{\footnotesize Number of Displayed Queries ($N$)}
 &
 \multicolumn{1}{c}{\footnotesize Number of Displayed Queries ($N$)}
\\
\multicolumn{3}{c}{
\includegraphics[width=0.98\textwidth]{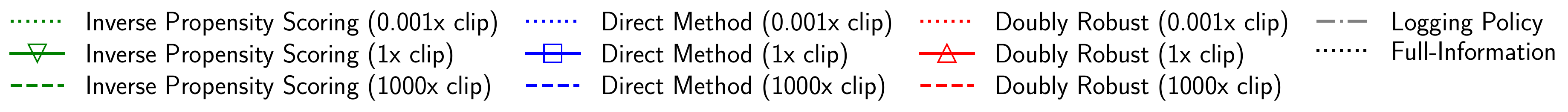}
} 
\end{tabular}
\caption{
The effect of different clipping strategies on the ECP (Eq.~\ref{eq:reward}) of three estimators in the top-5 known-bias setting.
Clipping strategies applied are standard:
$\tau^\text{1x} = 10/\sqrt{N}$, little: $\tau^\text{0.001x} = 10^{-2}/\sqrt{N}$, and heavy: $\tau^\text{1000x} = 10^4/\sqrt{N}$ (cf.~Eq.~\ref{eq:estimatedvalues}).
Results are means over 20 independent runs; y-axis: policy performance in terms of ECP (Eq.~\ref{eq:reward}) on the held-out test-set; x-axis: $N$ the number of displayed rankings in the simulated training set.
}
\label{fig:clipN}
\vspace{\baselineskip}
\centering
\begin{tabular}{@{}l @{}l @{}l}
 \multicolumn{1}{c}{\hspace{0.15cm} \footnotesize Yahoo! Webscope}
&
 \multicolumn{1}{c}{\hspace{0.15cm} \footnotesize MSLR-WEB30k}
&
 \multicolumn{1}{c}{\hspace{-0.15cm} \footnotesize Istella}
\\
\includegraphics[scale=0.34]{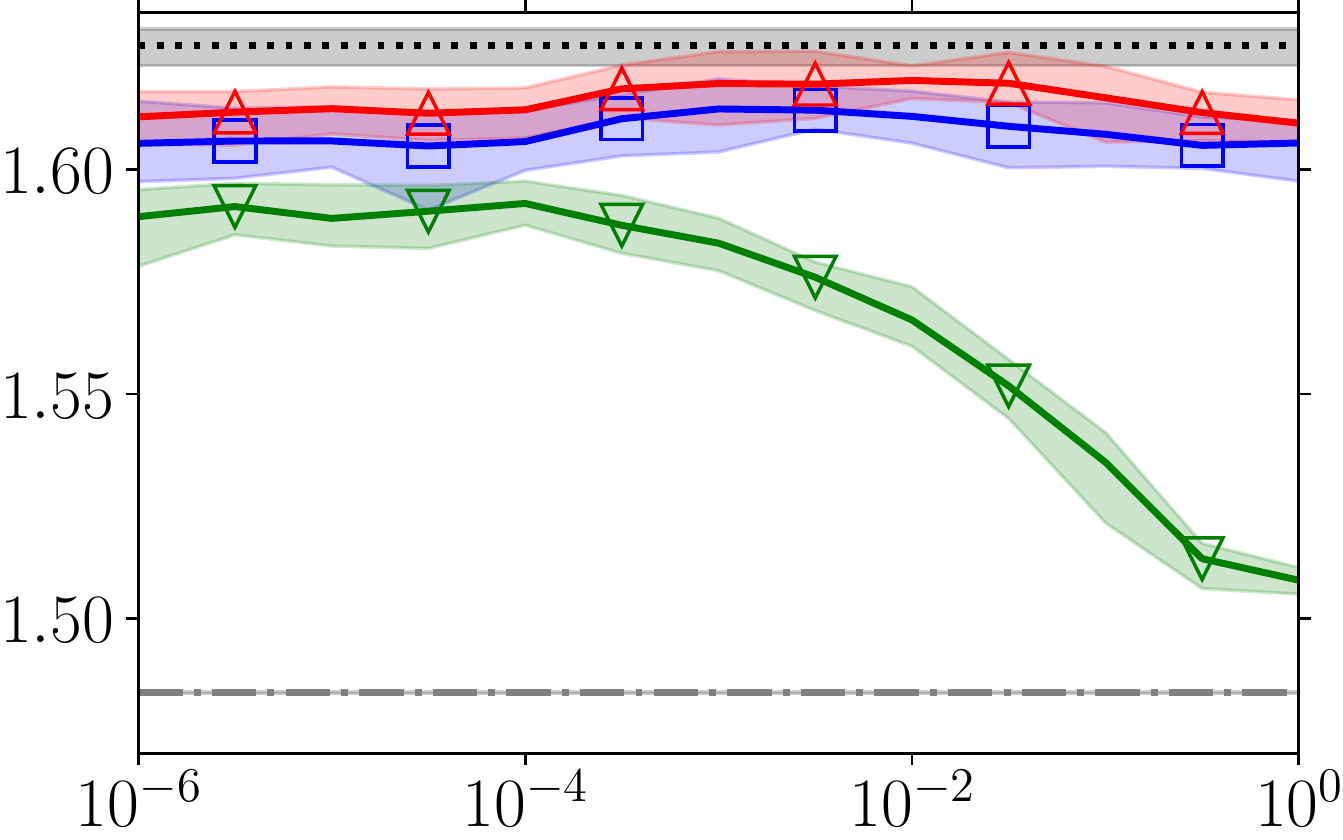}\hspace{1.28mm} &
\includegraphics[scale=0.34]{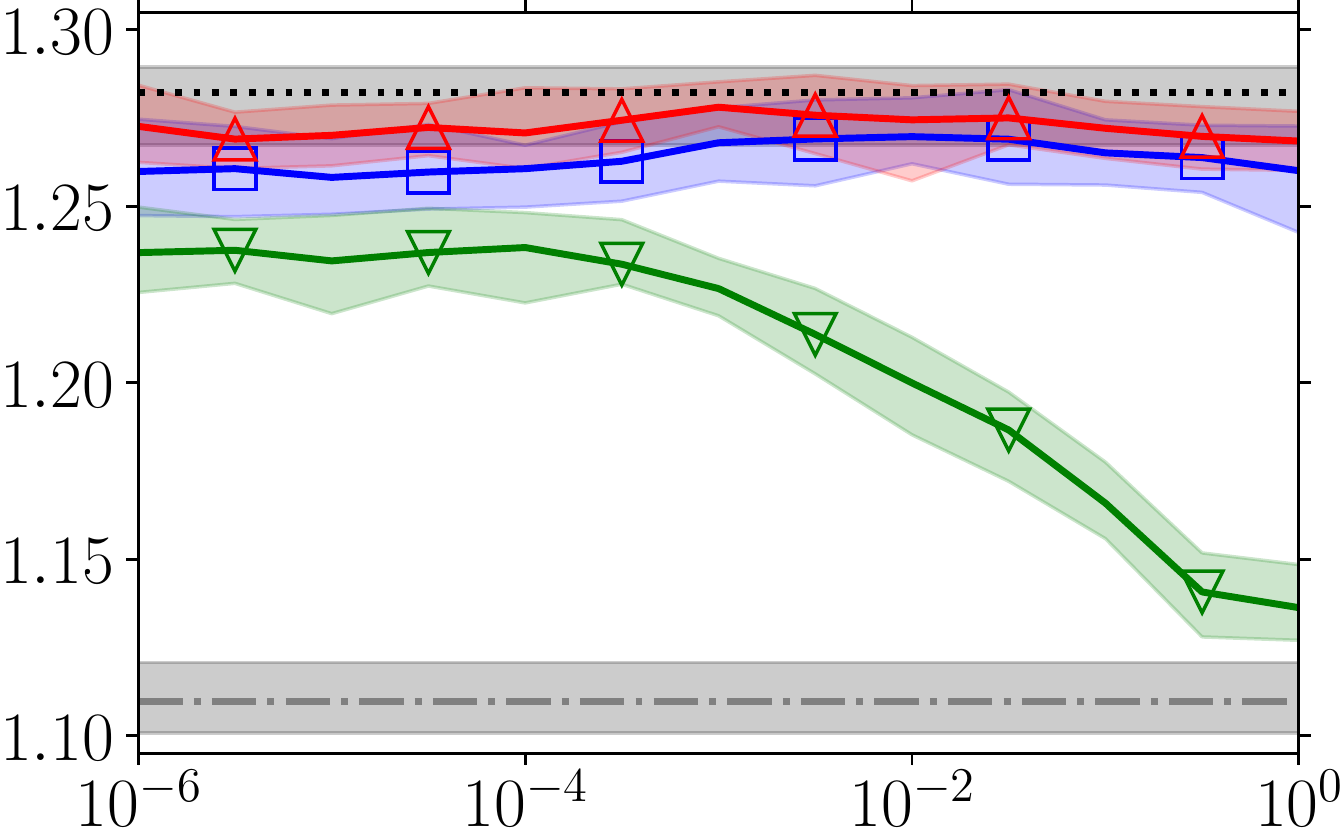}\hspace{1.28mm} &
\includegraphics[scale=0.34]{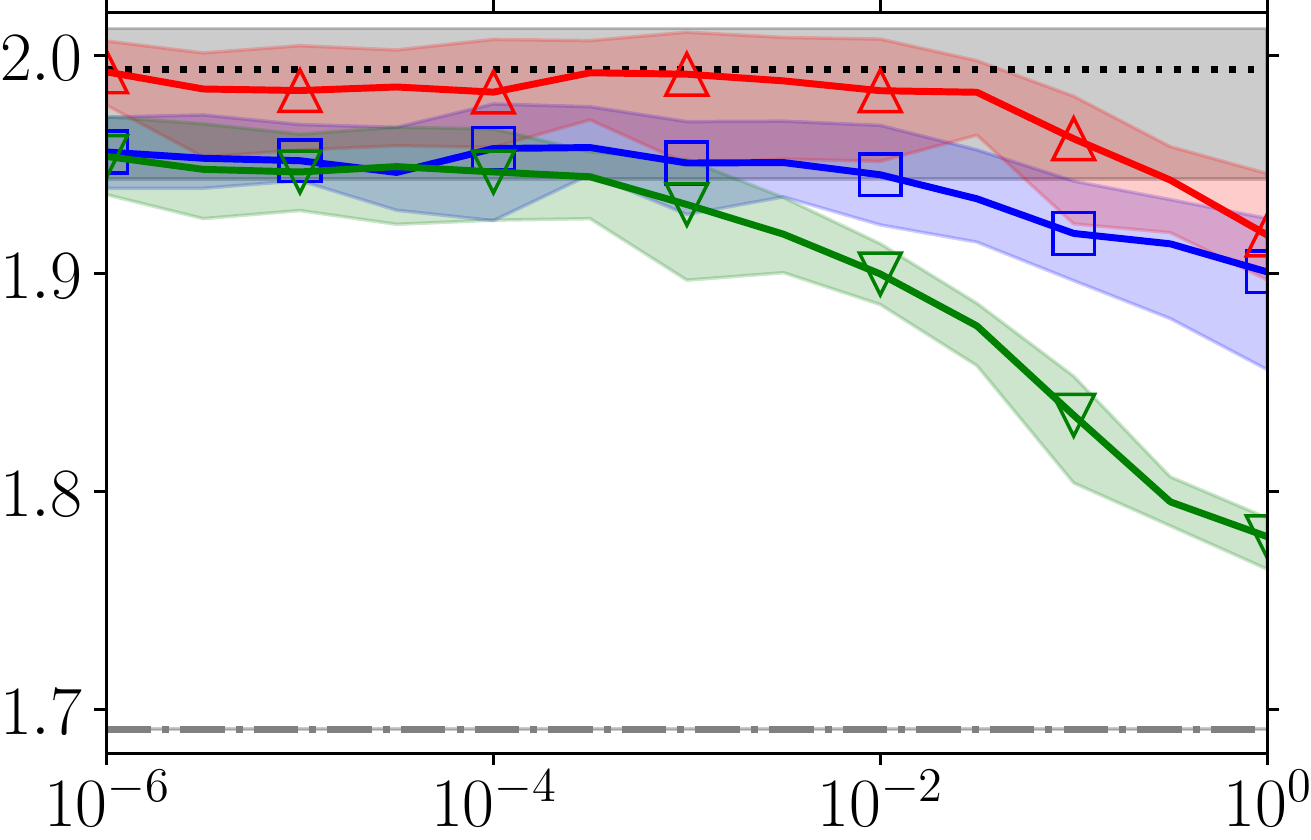}\hspace{1.28mm}
\\
 \multicolumn{1}{c}{\footnotesize Clipping Threshold  ($\tau$)}
 &
 \multicolumn{1}{c}{\footnotesize Clipping Threshold  ($\tau$)}
 &
 \multicolumn{1}{c}{\footnotesize Clipping Threshold  ($\tau$)}
\\
\multicolumn{3}{c}{
\includegraphics[width=0.98\textwidth]{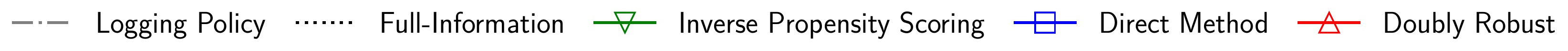}
} 
\end{tabular}
\caption{
The effect of the clipping parameter $\tau$ (Eq.~\ref{eq:estimatedvalues}) on the ECP (Eq.~\ref{eq:reward}) of three estimators in the top-5 known-bias setting when the number of impressions $N=10^8$.
Results are means over 20 independent runs, shaded areas indicate the 90\% confidence intervals; y-axis: policy performance in terms of ECP (Eq.~\ref{eq:reward}) on the held-out test-set; x-axis: $\tau$ the clipping threshold.
}
\label{fig:cliptau}
\end{figure*}

\begin{figure*}[t]
\centering
\begin{tabular}{@{}l @{}l @{}l}
 \multicolumn{1}{c}{\hspace{0.12cm} \footnotesize Yahoo! Webscope}
&
 \multicolumn{1}{c}{\hspace{0.0cm} \footnotesize MSLR-WEB30k}
&
 \multicolumn{1}{c}{\hspace{-0.15cm} \footnotesize Istella}
\\
\includegraphics[scale=0.34]{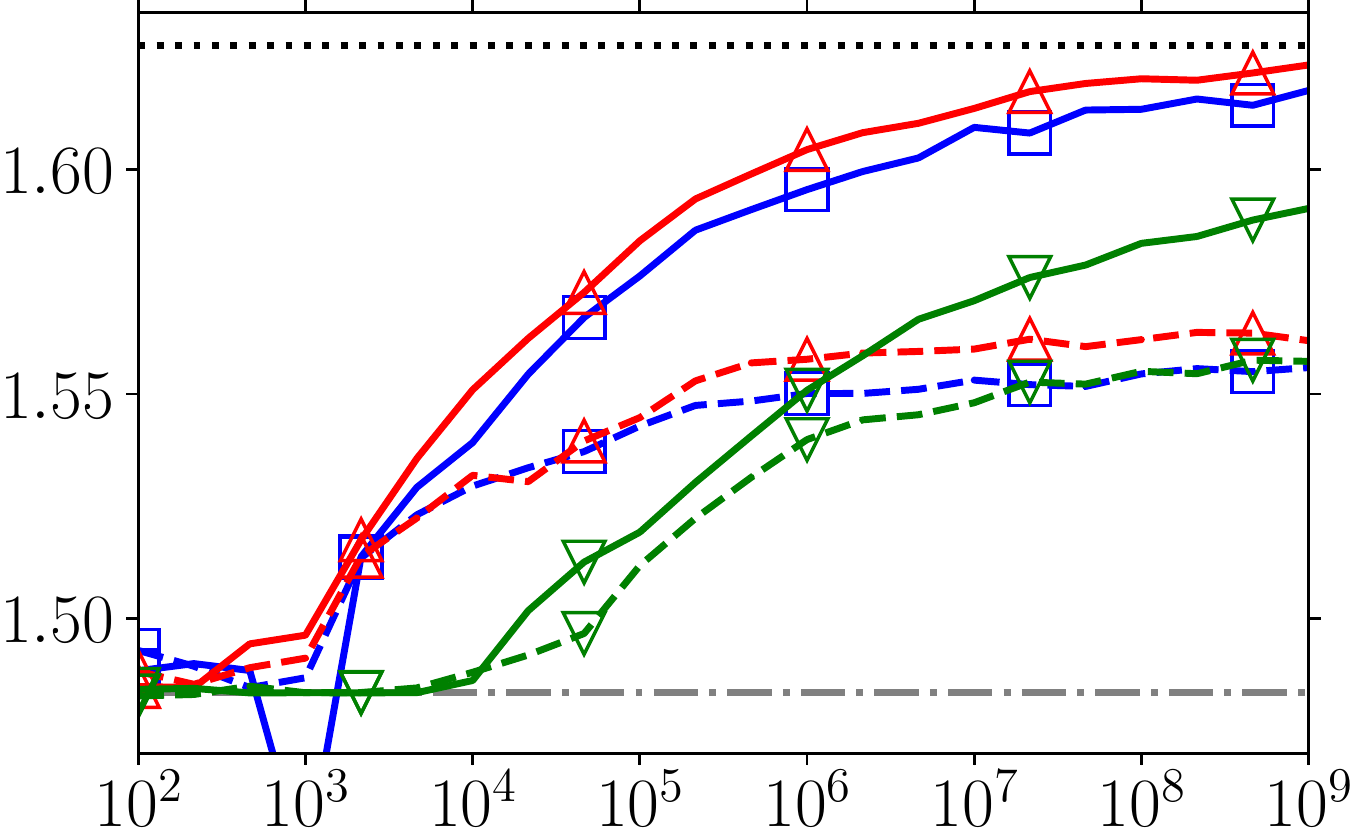}\hspace{1.28mm} &
\includegraphics[scale=0.34]{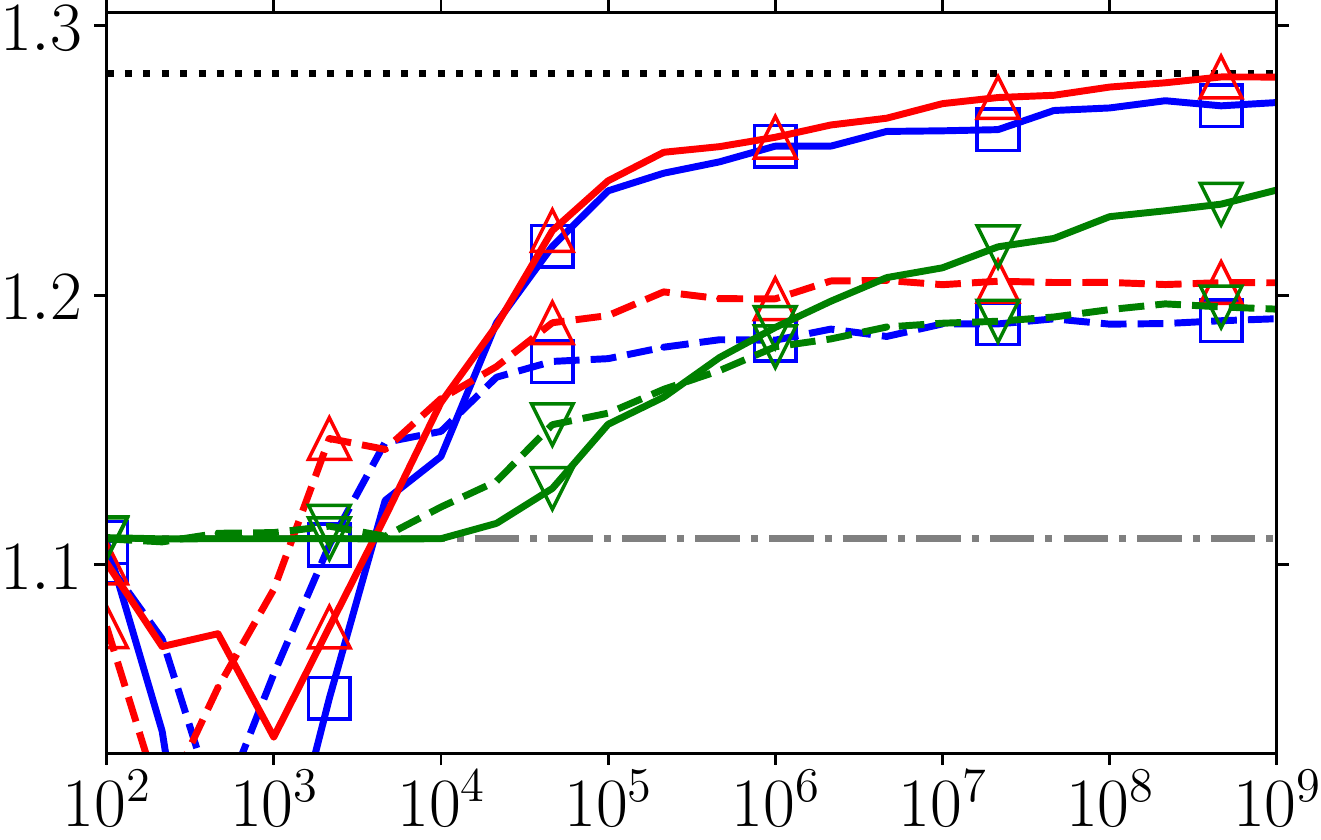}\hspace{1.28mm} &
\includegraphics[scale=0.34]{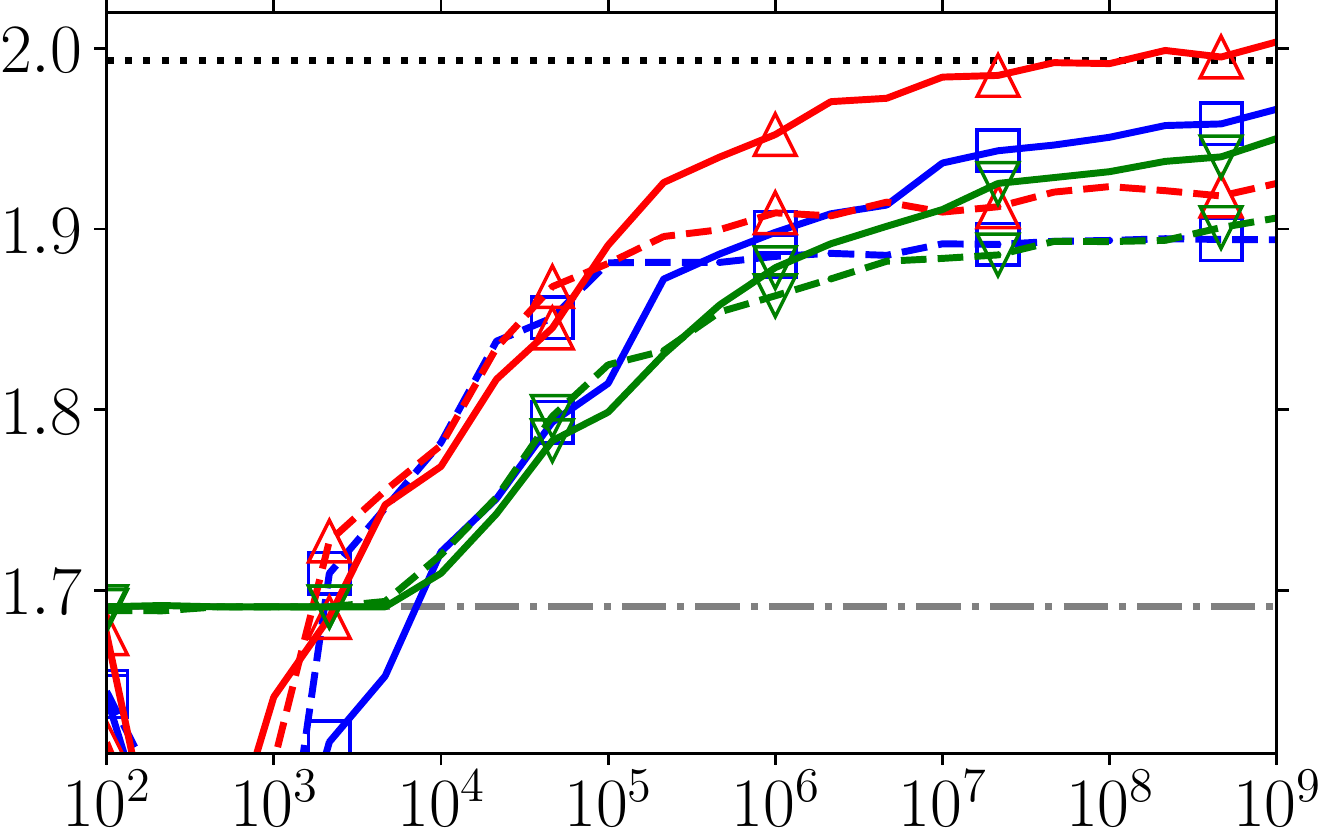}\hspace{1.28mm}
\\
 \multicolumn{1}{c}{\footnotesize Number of Displayed Queries ($N$)}
 &
 \multicolumn{1}{c}{\footnotesize Number of Displayed Queries ($N$)}
 &
 \multicolumn{1}{c}{\footnotesize Number of Displayed Queries ($N$)}
\\
\multicolumn{3}{c}{
\includegraphics[width=0.98\textwidth]{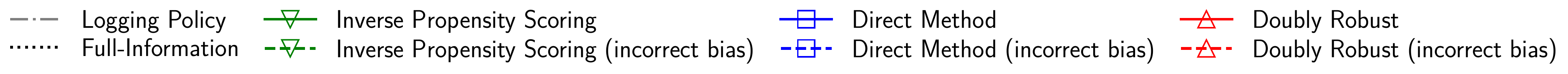}
} 
\end{tabular}
\caption{
Effect of very incorrect bias parameters $\hat{\alpha}$ and $\hat{\beta}$ on the ECP (Eq.~\ref{eq:reward}) of three estimators in the top-5 known-bias setting.
Incorrect bias estimates are the mean of the true values across all positions: $\hat{\alpha}_k = \sum_{i=1}^5 \alpha_i/ 5$ and $\hat{\beta}_k = \sum_{i=1}^5 \beta_i / 5$, as if there is no position-bias effect within the top-5.
Results are means over 20 independent runs; y-axis: policy performance in terms of ECP (Eq.~\ref{eq:reward}) on the held-out test-set; x-axis: $N$ the number of displayed rankings in the simulated training set.
}
\label{fig:biasN}
\vspace{\baselineskip}
\centering
\begin{tabular}{@{}l @{}l @{}l}
 \multicolumn{1}{c}{\hspace{0.15cm} \footnotesize Yahoo! Webscope}
&
 \multicolumn{1}{c}{\hspace{0.15cm} \footnotesize MSLR-WEB30k}
&
 \multicolumn{1}{c}{\hspace{-0.15cm} \footnotesize Istella}
\\
\includegraphics[scale=0.34]{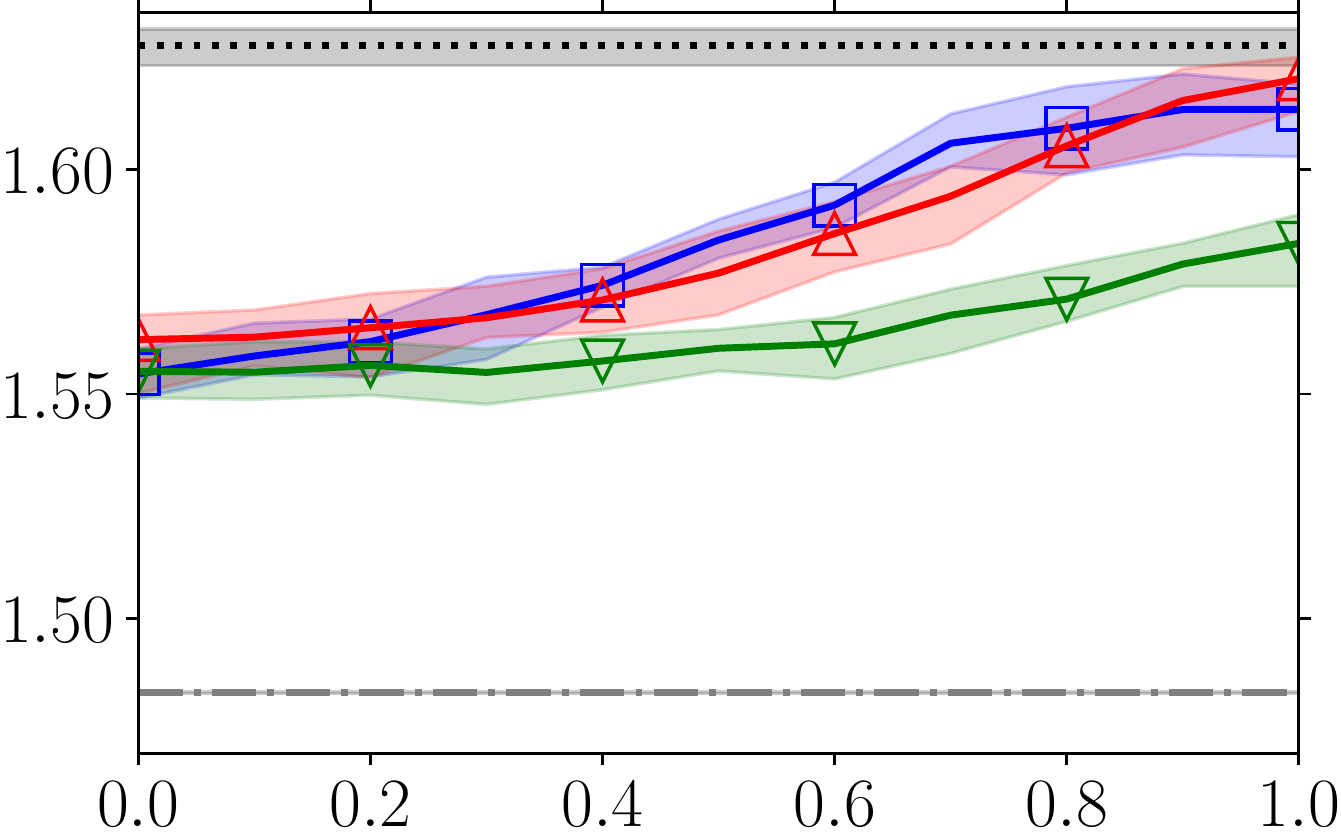}\hspace{1.28mm} &
\includegraphics[scale=0.34]{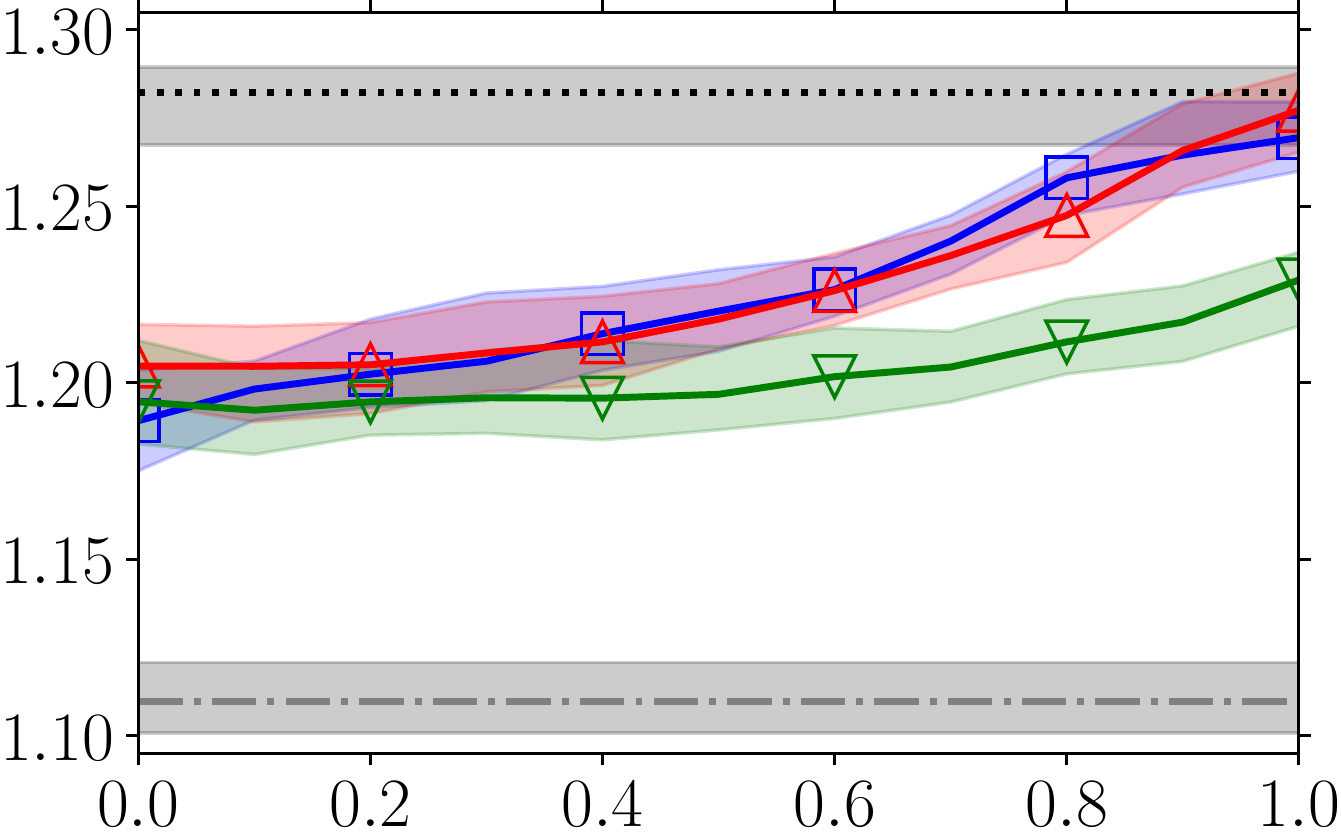}\hspace{1.28mm} &
\includegraphics[scale=0.34]{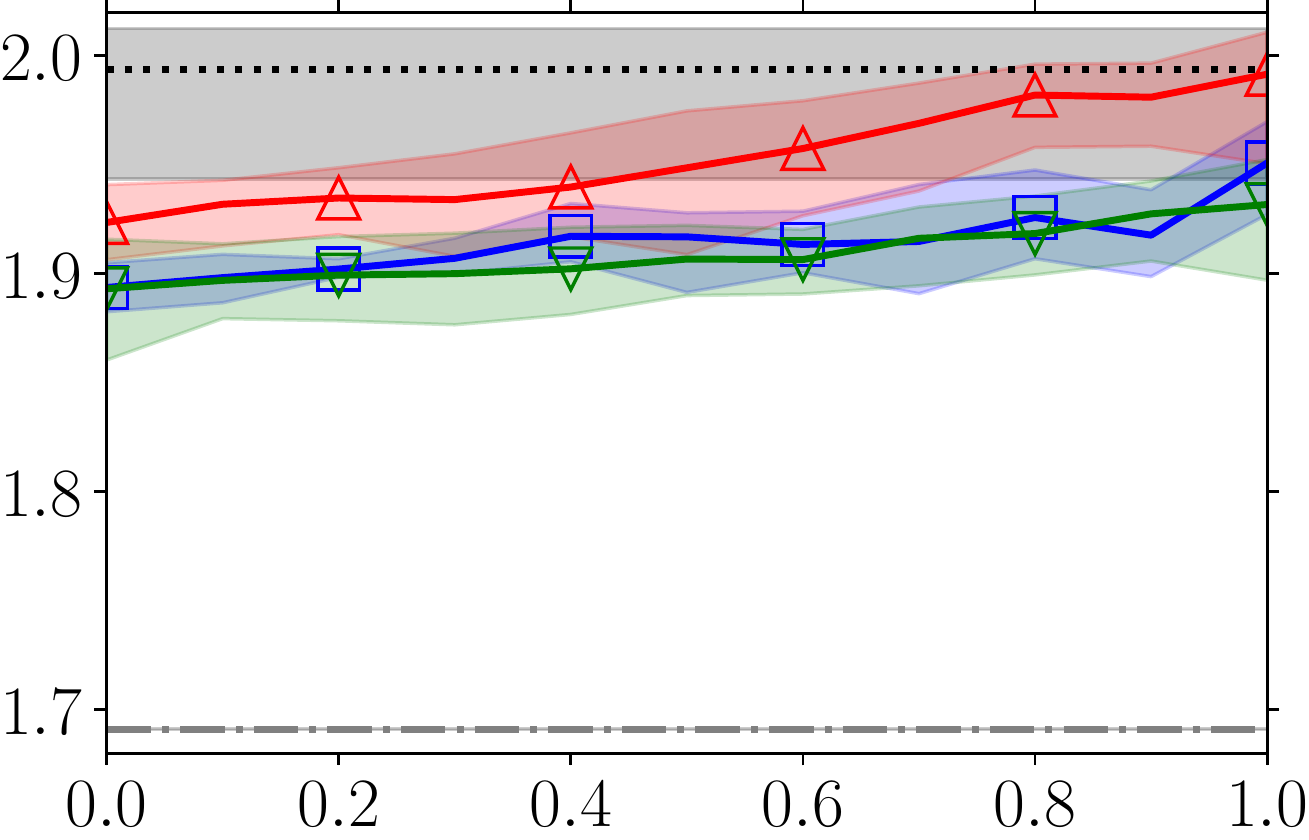}\hspace{1.28mm}
\\
 \multicolumn{1}{c}{\footnotesize Bias-Parameter Interpolation ($z$)}
 &
 \multicolumn{1}{c}{\footnotesize Bias-Parameter Interpolation ($z$)}
 &
 \multicolumn{1}{c}{\footnotesize Bias-Parameter Interpolation ($z$)}
\\
\multicolumn{3}{c}{
\includegraphics[width=0.98\textwidth]{figures/top5clipvary_legend}
} 
\end{tabular}
\caption{
Effect of varying misestimations of the bias parameters $\hat{\alpha}$ and $\hat{\beta}$ on the ECP (Eq.~\ref{eq:reward}) of three estimators in the top-5 known-bias setting when $N=10^8$.
Bias parameters are interpolations between the true values and their mean over positions with the interpolation parameter $z$: $\hat{\alpha}_k = z \cdot \alpha_k + (1-z) \sum_{i=1}^5 \alpha_i/ 5$ and $\hat{\beta}_k = z \cdot \beta_k + (1-z)  \sum_{i=1}^5 \beta_i / 5$.
Results are means over 20 independent runs, shaded areas indicate the 90\% confidence intervals; y-axis: policy performance in terms of ECP (Eq.~\ref{eq:reward}) on the held-out test-set; x-axis: $z$ the interpolation parameter.
}
\label{fig:biasvary}
\end{figure*}

\subsection{Correcting to Bias Introduced by the Clipping Strategy}
As discussed in Section~\ref{sec:drtheory}, one of the advantages of \ac{DR} estimation is that, in contrast with \ac{IPS}, it can potentially correct for some of the bias introduced by clipping.
To experimentally verify whether these corrections can lead to observable advantages in practice, we ran additional experiments with varying clipping strategies applied to the \ac{IPS}, \ac{DM} and \ac{DR} estimators in the top-5 setting with known bias parameters.

Figure~\ref{fig:clipN} shows the learning curves of these estimators with our standard clipping strategy: $\tau^\text{1x} = 10/\sqrt{N}$ (cf.~Eq.~\ref{eq:estimatedvalues}), a strategy with 1000 times less clipping: $\tau^\text{0.001x} = 10^{-2}/\sqrt{N}$, and a heavy clipping strategy: $\tau^\text{1000x} = 10^4/\sqrt{N}$.
In addition, the effect of individual threshold values are visualized in Figure~\ref{fig:cliptau}, where ECP with $N=10^8$ is displayed for values of the clipping threshold $\tau$ ranging from $10^{-6}$ to $1$.

Clearly, \ac{IPS} is the most sensitive to the clipping threshold as its ECP drops dramatically when heavy clipping is applied.
In contrast, while there is a noticeable effect from varying $\tau$ on the \ac{DM} and \ac{DR} estimators, the differences between light, standard and heavy clipping are relatively small on the Yahoo!\ and MSLR datasets.
On the Istella dataset, there is a larger decrease in ECP for the \ac{DM} and \ac{DR} with heavy clipping, but it is still much smaller than that of \ac{IPS}.
Importantly, we see that, regardless of what clipping is applied, \ac{DR} always has a higher ECP than \ac{DM} and \ac{IPS}.
This indicates that the performance advantage of \ac{DR} over \ac{DM} remains stable w.r.t.\ the clipping strategy, where the differences with \ac{IPS} become especially large under heavy clipping.

Therefore, we conclude that the \ac{DM} and \ac{DR} are less sensitive to propensity clipping than \ac{IPS} and can better correct for the bias introduced by clipping strategies.
Where the performance of \ac{IPS} considerably varies for different clipping strategies, \ac{DM} and \ac{DR} are only affected by very heavy clipping.
We can thus infer that the use of regression by \ac{DM} and \ac{DR} makes them more robust to propensity clipping.
Moreover, the advantage of \ac{DR} over both \ac{DM} and \ac{IPS} is consistent across all our tested clipping strategies, indicating it is the optimal choice regardless of what clipping strategy is applied.

\subsection{Robustness to Incorrect Bias Specification}

The main results presented in Figure~\ref{fig:results} and Table~\ref{tab:results} reveal that there is very little difference in performance between the top-5 setting where the bias parameters are known and where they have to be estimated.
While this shows that good performance is maintained when bias has to be estimated, it is unclear whether this also means that the estimators are robust to misspecified bias, since it is possible the estimated bias parameters are actually quite accurate.
Furthermore, most of our theoretical results assume that bias is correctly estimated, it is thus valuable to empirically verify whether the advantages of the \ac{DR} remain when its bias parameters are incorrect.

To better understand how robust the \ac{DR} estimator is to bias misspecification, the ECP of the \ac{DR}, \ac{DM} and \ac{IPS} estimators were measured in the top-5 setting with intentionally misspecified bias parameters.
For the incorrect bias parameters, we choose the mean values across positions: $\hat{\alpha}_k =  \sum_{i=1}^5 \alpha_i/5$ and $\hat{\beta}_k = \sum_{i=1}^5 \beta_i / 5$.
These mean values represent a naive approach that ignores the effect of the position on the examination and trust of users, i.e.\ it assumes that any document that is displayed in the top-5 is treated equally by the user, regardless of its exact position.

Figure~\ref{fig:biasN} displays the learning curves with these incorrect bias parameters.
Clearly, the ECP reached with all three estimators drops dramatically when the bias is heavily misspecified.
While \ac{IPS} and \ac{DM} converge on similar performance, \ac{DR} provides noticeably higher ECP when $N\geq10^5$ on all three datasets.
This strongly indicates that \ac{DR} is more robust to heavily misspecified bias than \ac{DM} and \ac{IPS}.

We further investigate how the degree of misspecification affects the estimators, by measuring ECP in the top-5 setting when $N=10^8$ and bias is interpolated between the true values and the mean with the parameter $z \in [0,1]$:
$\hat{\alpha}_k = z \cdot \alpha_k + (1-z) \sum_{i=1}^5 \alpha_i/ 5$ and $\hat{\beta}_k = z \cdot \beta_k + (1-z)  \sum_{i=1}^5 \beta_i / 5$.
The results are displayed in Figure~\ref{fig:biasvary}.

In line with our previous observations, Figure~\ref{fig:biasvary} reveals \ac{IPS} to have the lowest ECP, regardless of bias misspecification.
Interestingly, the differences between \ac{DR} and \ac{DM} vary: on Istella, \ac{DR} considerably outperforms \ac{DM}, but on Yahoo!\ and MSLR, the difference is only clear when $z<0.1$ and $z>0.9$.
When the interpolation is more in between the extreme values, \ac{DM} and \ac{DR} have comparable ECP where sometimes \ac{DM} has slightly higher performance.
As a result, we cannot conclude whether \ac{DR} better deals with bias misspecification than \ac{DM}.
Nevertheless, the differences between \ac{DR} and \ac{DM} are relatively small, thus the choice does not seem very consequential.
Conversely, our results clearly indicate that \ac{IPS} provides worse ECP than \ac{DR} and \ac{DM} whether bias is misspecified or not.

In summary, our results show that \ac{DR} estimation is much more robust to bias misspecification than \ac{IPS}.
Moreover, it appears to outperform \ac{DM} under heavy or light misspecification, but results are mixed when the misspecification is moderate.
Overall, our results indicate that the advantages of \ac{DR} over \ac{IPS} and \ac{DM} are mostly still applicable when bias is incorrectly estimated or misspecified.

\section{Conclusion}
\label{sec:conclusion}

This paper has introduced the first unbiased \ac{DR} estimator that is specifically designed to correct for position-bias in click feedback.
Our estimator differs from existing \ac{DR} estimators by using the expected correlation between clicks and preference per rank, instead of the unobservable examination variable or corrections solely based on action probabilities.
Additionally, we also proposed a novel \ac{DM} estimator and a novel cross-entropy loss estimator.
In terms of theory, this work has contributed the most robust estimator for \ac{LTR} yet:
our \ac{DR} estimator is the only method that corrects for position-bias, trust-bias and item-selection bias and has less strict unbiasedness criteria than the prevalent \ac{IPS} approach.
Moreover, our experimental results show that it can provide enormous increases in data-efficiency compared to \ac{IPS} and better overall performance w.r.t.\  other existing state-of-the-art approaches.
Therefore, both our theoretical and empirical results indicate that our \ac{DR} estimator is the most reliable and effective way to correct for position-bias.
Consequently, we think there is large potential in replacing \ac{IPS} with \ac{DR} as the new basis for the unbiased \ac{LTR} field.

Future work hopefully finds similar gains in related tasks, e.g.\ exposure-based ranking fairness~\citep{singh2018fairness} or ranking display advertisements~\citep{lagree2016multiple}.
Overall, we expect the improvements in efficiency and robustness to make unbiased \ac{LTR} even more attractive for real-world applications.

\subsection*{Code, Resources and Data}
To facilitate the reproducibility of the reported results, this work only made use of publicly available data and our experimental implementation is publicly available at \url{https://github.com/HarrieO/2022-doubly-robust-LTR}.
Additionally, a video presentation with accompanying slides is available at \url{https://harrieo.github.io//publication/2023-doubly-robust}.

\subsection*{Acknowledgments}
We thank the reviewers of previous versions of this work for their valuable comments and suggestions.
This research was partially supported by the Google Research Scholar Program.
All content represents the opinion of the author, which is not necessarily shared or endorsed by their respective employers and/or sponsors.

\appendix

\section*{Appendices}

\section{BIAS and Variance of IPS}
\label{appendix:proofipsbias}

\begin{theorem}
\label{theorem:ipsbias}
The \ac{IPS} estimator (Eq.~\ref{eq:ips}) has the following bias:
\begin{equation}
\mathds{E}_{c,y \sim \pi_0}\mleft[ \hat{\mathcal{R}}_\text{\normalfont IPS}(\pi) \mright] - \mathcal{R}(\pi)
=
\sum_{d \in D} \frac{\hat{\omega}_{d}}{\hat{\rho}_{d}}\mleft(\mleft(\rho_{d} - \hat{\rho}_{d}\frac{\omega_{d}}{\hat{\omega}_{d}} \mright) R_d + \mathds{E}_{y \sim \pi_0}\mleft[ \beta_{k(d)} - \hat{\beta}_{k(d)} \mright]\mright).
\end{equation}

\end{theorem}
\begin{proof}
Using Eq.~\ref{eq:clickprob}, \ref{eq:reward}, \ref{eq:truerho} and \ref{eq:ips} we get the following derivation:
\begin{align}
\mathds{E}_{c,y \sim \pi_0}\mleft[ \hat{\mathcal{R}}_\text{\normalfont IPS}(\pi) \mright]
&= \sum_{d \in D} \frac{\hat{\omega}_{d}}{\hat{\rho}_{d}} \mleft(\rho_{d}R_d + \mathds{E}_{y \sim \pi_0}\mleft[ \beta_{k(d)} - \hat{\beta}_{k(d)} \mright]\mright)
 \\
&=
 \mathcal{R}(\pi) + 
\sum_{d \in D} \frac{\hat{\omega}_{d}}{\hat{\rho}_{d}}\mleft(\mleft(\rho_{d} - \hat{\rho}_{d}\frac{\omega_{d}}{\hat{\omega}_{d}} \mright) R_d + \mathds{E}_{y \sim \pi_0}\mleft[ \beta_{k(d)} - \hat{\beta}_{k(d)} \mright]\mright).
\qedhere
\end{align}
\end{proof}

\begin{lemma}
\label{lemma:omega}
By the definitions of $\omega$ (Eq.~\ref{eq:trueomega}) and $\hat{\omega}$ (Eq.~\ref{eq:estimatedvalues}):
\begin{equation}
\big(\hat{\alpha} = \alpha \land \hat{\beta} = \beta \big) \longrightarrow \mleft(\forall d \in D, \; \hat{\omega}_{d} = \omega_d \mright).
\end{equation}
\end{lemma}

\begin{lemma}
\label{lemma:rho}
By the definitions of $\rho$ (Eq.~\ref{eq:truerho}) and $\hat{\rho}$ (Eq.~\ref{eq:estimatedvalues}):
\begin{equation}
\mleft(\hat{\alpha} = \alpha \land \mleft(\forall d \in D,\; \hat{\pi}_0(d) = \pi_0(d) \land \rho_d \geq \tau\mright)\mright) \longrightarrow (\forall d \in D, \; \hat{\rho}_{d} = \rho_d).
\end{equation}
\end{lemma}

\begin{lemma}
\label{eq:proof:betaips}
Trivially, if the $\hat{\beta}$ bias parameters are correct then:
\begin{equation}
\hat{\beta} = \beta \longrightarrow
\mleft(\forall d \in D, \;  \mathds{E}_{y \sim \pi_0}\mleft[\hat{\beta}_{k(d)} \mright] = \mathds{E}_{y \sim \pi_0}\mleft[\beta_{k(d)} \mright]  \mright).
\end{equation}
\end{lemma}

\begin{corollary}
\label{theorem:ipssimplebias}
When $\hat{\alpha}$ and $\hat{\beta}$ are correct \ac{IPS} has the bias:
\begin{equation}
( \hat{\alpha} = \alpha \land \hat{\beta} = \beta ) \longrightarrow
\mathbb{E}_{c,y \sim \pi_0}\mleft[ \widehat{\mathcal{R}}_\text{IPS}(\pi) \mright] - \mathcal{R}(\pi)
= 
\sum_{d \in D} \frac{{\omega}_{d}}{\hat{\rho}_{d}}\mleft(\rho_{d} - \hat{\rho}_{d} \mright) R_d.
\end{equation}
\end{corollary}
\begin{proof}
Follows from Theorem~\ref{theorem:ipsbias} and Lemmas~\ref{lemma:omega} and~\ref{eq:proof:betaips}.
\end{proof}

\begin{theorem}
\label{eq:proof:ipsbiastheorem}
The \ac{IPS} estimator (Eq.~\ref{eq:ips})  is unbiased when $\hat{\alpha}$, $\hat{\beta}$ and $\hat{\pi}_0$ are correctly estimated and clipping has no effect:
\begin{equation}
\big(\hat{\alpha} = \alpha \land \hat{\beta} = \beta \land
\mleft(\forall d \in D, \;  \hat{\pi}_0(d) = \pi_0(d) \land \rho_d \geq \tau  \mright)\big)
\longrightarrow
\mathds{E}_{c,y \sim \pi_0}\mleft[ \hat{\mathcal{R}}_\text{\normalfont IPS}(\pi) \mright] = \mathcal{R}(\pi).
\end{equation}
\end{theorem}
\begin{proof}
Follows from applying Lemma~\ref{lemma:rho} to Corollary~\ref{theorem:ipssimplebias}.
\end{proof}

\begin{theorem}
The \ac{IPS} estimator (Eq.~\ref{eq:ips}) has the variance:
\begin{equation}
\mathds{V}\big[ \widehat{\mathcal{R}}_\text{IPS}(\pi) \big]
 = \frac{1}{N}\sum_{d \in D} \frac{\hat{\omega}_{d}^2}{\hat{\rho}_{d}^2}
\big(
\mathds{V}\big[ c(d) \big] + \mathds{V}\big[ \hat{\beta}_{k(d)} \big] - 2  \mathds{C}\text{ov}\big[c(d), \hat{\beta}_{k(d)}\big]
\big).
\end{equation}
\end{theorem}
\begin{proof}
Follows from Eq.~\ref{eq:clickprob} and \ref{eq:ips}.
\end{proof}

\section{Bias of \ac{CV} Estimator}
\label{appendix:cv}

\begin{lemma}
\label{lemma:cvexpectedvalue}
The \ac{CV} estimator (Eq.~\ref{eq:cv}) has the following expected value:
\begin{equation}
\mathds{E}_{c,y \sim\pi_0}\mleft[ \widehat{\mathcal{R}}_\text{CV}(\pi) \mright] 
= \sum_{d \in D}
\frac{\hat{\omega}_{d}}{\hat{\rho}_{d}}
\mathds{E}_{y \sim\pi_0}\mleft[\hat{\alpha}_{k(d)}\mright]
 \hat{R}_d
.
 \end{equation}
\end{lemma}
\begin{proof} Follows directly from Eq.~\ref{eq:cv}.
\end{proof}

\begin{theorem}
\label{theorem:cvunbiasreq}
The \ac{CV} estimator (Eq.~\ref{eq:cv}) is an unbiased estimate of \ac{DM} (Eq.~\ref{eq:regression})
if the $\hat{\alpha}$ and $\hat{\beta}$ bias parameters are correctly estimated
and per item either $\hat{\pi}_0(d)$ is correct and clipping has no effect:
\begin{equation}
\big(
\forall d \in D, \; \hat{\pi}_0(d) = \pi_0(d) \land \hat{\rho}_d \geq \tau
\big)
\longrightarrow
\mathds{E}\mleft[
\widehat{\mathcal{R}}_\text{CV}(\pi)
\mright]
 = \widehat{\mathcal{R}}_\text{DM}(\pi)
 .
\end{equation}
\end{theorem}
\begin{proof}
From Lemma~\ref{lemma:cvexpectedvalue} it clearly follows that the expected value of \ac{CV} is equal to \ac{DM} (Eq.~\ref{eq:regression}) when $\hat{\rho}_{d} =
\mathds{E}_{y \sim\pi_0}\mleft[\hat{\alpha}_{k(d)}\mright]$.
The definition of $\hat{\rho}$ (Eq.~\ref{eq:estimatedvalues}) shows that this is the case when $\hat{\pi}_0(d)$ is correct and clipping has no effect:
\begin{equation}
\big(\forall d \in D,\; \hat{\pi}_0(d) = \pi_0(d) \land \hat{\rho}_d \geq \tau\big) \longrightarrow \mleft(\forall d \in D, \; \hat{\rho}_{d} = \mathds{E}_{y \sim\pi_0}\mleft[\hat{\alpha}_{k(d)}\mright]\mright).
\label{eq:cvproofstep}
\end{equation}
Applying Eq.~\ref{eq:cvproofstep} to Lemma~\ref{lemma:cvexpectedvalue} thus proves Theorem~\ref{theorem:cvunbiasreq}.
\end{proof}

\section{Bias and Variance of \ac{DR} Estimator}
\label{appendix:drbias}
\label{appendix:drvariance}

\begin{theorem}
\label{theorem:drlongbias}
The \ac{DR} estimator (Eq.~\ref{eq:dr}) has the following bias:
\begin{equation}
\begin{split}
&
\mathds{E}_{c,y \sim\pi_0}\mleft[ \hat{\mathcal{R}}_\text{\normalfont DR}(\pi) \mright] - \mathcal{R}(\pi)
\\ &\hspace{1.7cm}
= 
 \sum_{d \in D} \frac{\hat{\omega}_{d}}{\hat{\rho}_{d}}
\bigg(
\mleft(\rho_{d} - \frac{\hat{\rho}_{d}}{\hat{\omega}_d}\omega_d \mright)R_d 
+ \mleft(\hat{\rho}_{d} - \mathds{E}_{y\sim\pi_0}\mleft[\hat{\alpha}_{k(d)}\mright]\mright)\hat{R}_d
+ \mathds{E}_{y\sim\pi_0}\mleft[ \beta_{k(d)} - \hat{\beta}_{k(d)} \mright]
 \bigg).
 \end{split}
 \end{equation}
\end{theorem}
\begin{proof} Using Eq.~\ref{eq:reward},~\ref{eq:regression} and \ref{eq:dr} we make the following derivation:
\begin{align}
&\mathds{E}_{c,y \sim\pi_0}\mleft[ \hat{\mathcal{R}}_\text{DR}(\pi) \mright]
=  \hat{\mathcal{R}}_\text{DM}(\pi) + 
\sum_{d \in D} \frac{\hat{\omega}_{d}}{\hat{\rho}_{d}}
\Big(
\rho_{d} R_d 
 - \mathds{E}_{y\sim \pi_0}\mleft[\hat{\alpha}_{k(d)}\mright]\hat{R}_d
+ \mathds{E}_{y\sim \pi_0}\mleft[ \beta_{k(d)} - \hat{\beta}_{k(d)} \mright]
 \Big)
\nonumber \\
&\hspace{0.7cm}
= \sum_{d \in D} \frac{\hat{\omega}_{d}}{\hat{\rho}_{d}}
\mleft(
\rho_{d} R_d
 - \big(\mathds{E}_{y\sim \pi_0}\mleft[\hat{\alpha}_{k(d)}\mright] - \hat{\rho}_{d}\big)\hat{R}_d
 + \mathds{E}_{y\sim \pi_0}\mleft[ \beta_{k(d)} - \hat{\beta}_{k(d)} \mright]
 \mright)
   \label{eq:adrbias1} \\
&\hspace{0.7cm}
=  \mathcal{R}(\pi) + \sum_{d \in D} \frac{\hat{\omega}_{d}}{\hat{\rho}_{d}}
\Big(
\Big(\rho_{d} - \frac{\hat{\rho}_{d}}{\hat{\omega}_d}\omega_d \Big) R_d 
 - \mleft(\mathds{E}_{y\sim\pi_0}\mleft[\hat{\alpha}_{k(d)}\mright] - \hat{\rho}_{d}\mright)\hat{R}_d
 + \mathds{E}_{y\sim\pi_0}\mleft[ \beta_{k(d)} - \hat{\beta}_{k(d)} \mright]
 \Big).
 \qedhere
 \end{align}
\end{proof}

\begin{lemma}
\label{lemma:expalpha1}
By the definition of $\rho$ (Eq.~\ref{eq:truerho}):
\begin{equation}
\hat{\alpha} = \alpha  \longrightarrow \big(\forall d \in D, \;  \mathds{E}_{y\sim \pi_0}\big[\hat{\alpha}_{k(d)}\big] = \rho_d \big).
\end{equation}
\end{lemma}

\begin{corollary}
\label{theorem:drsimplebias}
The bias of the \ac{DR} estimator (Eq.~\ref{eq:dr}) can be simplified when $\hat{\alpha}$ and $\hat{\beta}$ are correctly estimated:
\begin{equation}
( \hat{\alpha} = \alpha \land \hat{\beta} = \beta ) \longrightarrow
\mathbb{E}_{c,y \sim\pi_0}\mleft[ \hat{\mathcal{R}}_\text{DR}(\pi) \mright]
-
\mathcal{R}(\pi) = \sum_{d \in D} \frac{\omega_{d}}{\hat{\rho}_{d}}
\mleft(\rho_{d} - \hat{\rho}_{d} \mright)\mleft( R_d - \hat{R}_d \mright).
\end{equation}
\end{corollary}
\begin{proof}
Apply Lemmas~\ref{lemma:omega}, \ref{eq:proof:betaips} and~\ref{lemma:expalpha1} to Theorem~\ref{theorem:drlongbias}.
\end{proof}

\begin{theorem}
\label{theorem:drbias}
The \ac{DR} estimator (Eq.~\ref{eq:dr}) is unbiased
if the $\hat{\alpha}$ and $\hat{\beta}$ bias parameters are correctly estimated
and per item either $\hat{\pi}_0(d)$ is correct and clipping has no effect or $\hat{R}_d$ is correct:
 \begin{equation}
\big(\hat{\alpha} = \alpha \land \hat{\beta} = \beta \land
\mleft(\forall d \in D, \,  (\hat{\pi}_0(d) = \pi_0(d)  \land \rho_d \geq \tau \mright) \lor \hat{R}_d = R_d )
\big)
\longrightarrow
\mathbb{E}_{c,y \sim \pi_0}\mleft[ \hat{\mathcal{R}}_\text{DR}(\pi) \mright] = \mathcal{R}_\pi.
\end{equation}
\end{theorem}
\begin{proof}
From Corollary~\ref{theorem:drsimplebias} it clearly follows that the \ac{DR} estimator is unbiased when the $\hat{\alpha}$ and $\hat{\beta}$ bias parameters are correct and per item $d$ either $\hat{\rho}_d$ or $\hat{R}_d$ is correct:
\begin{equation}
\mleft( \hat{\alpha} = \alpha \land \hat{\beta} = \beta \land \mleft(\forall d \in D, \,  \hat{\rho}_d = {\rho}_d \lor \hat{R}_d = R_d \mright) \mright) 
\longrightarrow \mathbb{E}_{c,y \sim \pi_0}\mleft[ \hat{\mathcal{R}}_\text{DR}(\pi) \mright] = \mathcal{R}_\pi.
 \label{eq:drunbiasedstep}
\end{equation}
Applying Lemma~\ref{lemma:rho} to Eq.~\ref{eq:drunbiasedstep} provides proof for Theorem~\ref{theorem:drbias}.
\end{proof}

\begin{theorem}
\label{theorem:dripsbiascomp}
If $\hat{\alpha}$ and $\hat{\beta}$ are correct and the regression model predicts each preference $\hat{R}_d$ between $0$ and twice the true $R_d$ value then the bias of the \ac{DR} estimator (Eq.~\ref{eq:dr}) is less or equal to that of the \ac{IPS} estimator (Eq.~\ref{eq:ips}):
\begin{equation}
\begin{split}
&
\mleft(\hat{\alpha} = \alpha \land \hat{\beta} = \beta
\land
\mleft(\forall d \in D, \; 0 \leq \hat{R}_d \leq 2 R_d \mright)\mright)
\\&\hspace{4cm}
\longrightarrow
| \mathbb{E}_{c,y \sim \pi_0}\mleft[ \mathcal{R}(\pi)\mright] - \hat{\mathcal{R}}_\text{DR}(\pi)| 
\leq
| \mathbb{E}_{c,y \sim \pi_0}\mleft[  \mathcal{R}(\pi) \mright] - \hat{\mathcal{R}}_\text{IPS}(\pi)|.
\end{split}
\end{equation}
\end{theorem}
\begin{proof}
This follows from comparing
Corollary~\ref{theorem:ipssimplebias}
with~\ref{theorem:drsimplebias}.
\end{proof}

\begin{theorem}
\label{theorem:drvariance}
The \ac{DR} estimator (Eq.~\ref{eq:dr}) has the variance:
\begin{align}
\mathds{V}\big[ \widehat{\mathcal{R}}_\text{DR}(\pi) \big]
= \frac{1}{N} \sum_{d \in D} \frac{\hat{\omega}_{d}^2}{\hat{\rho}_{d}^2}
\Big(&
\mathds{V}\big[ c(d) \big] + \mathds{V}\big[ \hat{\beta}_{k(d)} \big]
 + \hat{R}_d^2 \cdot \mathds{V}\mleft[ \hat{\alpha}_{k(d)} \mright]
\label{eq:drvariance}
\\&
- 2 \Big(
\mathds{C}\text{\normalfont ov}\big(c(d), \hat{\beta}_{k(d)}\big)
+ \hat{R}_d
\big(
\mathds{C}\text{\normalfont ov}\big(c(d), \hat{\alpha}_{k(d)}\big)
-
\mathds{C}\text{\normalfont ov}\big(\hat{\beta}_{k(d)}, \hat{\alpha}_{k(d)}\big)
\big)
\Big)
\Big).
\nonumber
\end{align}
\end{theorem}
\begin{proof}
This follows from Eq.~\ref{eq:clickprob} and~\ref{eq:dr}.
\end{proof}

\begin{lemma}
\label{lemma:covariance} 
The covariance between clicks on an item $c(d)$ and $\alpha_{k(d)}$ is:
\begin{equation}
\mathds{C}\text{\normalfont ov}\big(c(d), {\alpha}_{k(d)}\big)
=
R_d \mathds{V}\mleft[ {\alpha}_{k(d)} \mright] + \mathds{C}\text{\normalfont ov}\big({\alpha}_{k(d)}, {\beta}_{k(d)}\big).
\end{equation}
\end{lemma}
\begin{proof}
\begin{align}
\hspace{1cm}&\hspace{-1cm}
\mathds{C}\text{\normalfont ov}\mleft(c(d), {\alpha}_{k(d)}\mright)
= \mathds{E}_{c,y\sim\pi_0}\mleft[\Big(
c(d) - \mathds{E}_{c,y\sim\pi_0}\big[c(d)\big]
\Big)
\mleft(
{\alpha}_{k(d)} - \mathds{E}_{y\sim\pi_0}\mleft[{\alpha}_{k(d)}\mright]
\mright)
\mright]
\nonumber\\
&= \mathds{E}_{c,y\sim\pi_0}\mleft[\mleft(
c(d) - \mathds{E}_{y\sim\pi_0}\mleft[{\alpha}_{k(d)}\mright]  R_d - \mathds{E}_{y\sim\pi_0}\mleft[{\beta}_{k(d)}\mright]
\mright)
\mleft(
{\alpha}_{k(d)} - \mathds{E}_{y\sim\pi_0}\mleft[{\alpha}_{k(d)}\mright]
\mright)
\mright]
\nonumber\\
&= \mathds{E}_{c,y\sim\pi_0}\Big[
{\alpha}_{k(d)}c(d) - \mathds{E}_{y\sim\pi_0}\mleft[{\alpha}_{k(d)}\mright]c(d)
- {\alpha}_{k(d)}\mathds{E}_{y\sim\pi_0}\mleft[{\alpha}_{k(d)}\mright]  R_d
\nonumber\\& \hspace{0.3cm}
+ \mathds{E}_{y\sim\pi_0}\mleft[{\alpha}_{k(d)}\mright]^2  R_d
- {\alpha}_{k(d)}\mathds{E}_{y\sim\pi_0}\mleft[{\beta}_{k(d)}\mright]
+ \mathds{E}_{y\sim\pi_0}\mleft[{\alpha}_{k(d)}\mright]\mathds{E}_{y\sim\pi_0}\mleft[{\beta}_{k(d)}\mright]
\Big]
\nonumber\\
&= 
\mathds{E}_{y\sim\pi_0}\mleft[{\alpha}_{k(d)}^2\mright]R_d + \mathds{E}_{y\sim\pi_0}\mleft[{\alpha}_{k(d)}{\beta}_{k(d)}\mright]
- \mathds{E}_{y\sim\pi_0}\mleft[{\alpha}_{k(d)}\mright]^2R_d
\\& \hspace{0.3cm}
- \mathds{E}_{y\sim\pi_0}\mleft[{\alpha}_{k(d)}\mright]\mathds{E}_{y\sim\pi_0}\mleft[{\beta}_{k(d)}\mright]
- \mathds{E}_{y\sim\pi_0}\mleft[{\alpha}_{k(d)}\mright]^2  R_d
+ \mathds{E}_{y\sim\pi_0}\mleft[{\alpha}_{k(d)}\mright]^2R_d
\nonumber\\& \hspace{0.3cm}
 - \mathds{E}_{y\sim\pi_0}\mleft[{\alpha}_{k(d)}\mright]\mathds{E}_{y\sim\pi_0}\mleft[{\beta}_{k(d)}\mright]
+ \mathds{E}_{y\sim\pi_0}\mleft[{\alpha}_{k(d)}\mright]\mathds{E}_{y\sim\pi_0}\mleft[{\beta}_{k(d)}\mright]
\nonumber\\
&= 
R_d\mleft(
\mathds{E}_{y\sim\pi_0}\mleft[{\alpha}_{k(d)}^2\mright]
- \mathds{E}_{y\sim\pi_0}\mleft[{\alpha}_{k(d)}\mright]^2
\mright)
 + \mathds{E}_{y\sim\pi_0}\mleft[{\alpha}_{k(d)}{\beta}_{k(d)}\mright]
 - \mathds{E}_{y\sim\pi_0}\mleft[{\alpha}_{k(d)}\mright]\mathds{E}_{y\sim\pi_0}\mleft[{\beta}_{k(d)}\mright]
\nonumber\\
&
= R_d \mathds{V}\mleft[ {\alpha}_{k(d)} \mright] + \mathds{C}\text{\normalfont ov}\big({\alpha}_{k(d)}, {\beta}_{k(d)}\big).
\nonumber
\end{align}
\end{proof}

\begin{corollary}
\label{lemma:variance}
If $\hat{\alpha}$ and $\hat{\beta}$ are correct then the variance of the \ac{DR} estimator (Eq.~\ref{eq:dr}) is:
\begin{align}
&
(\hat{\alpha} = \alpha \land \hat{\beta} = \beta) \longrightarrow
\\
&\hspace{0.5cm}
\mathds{V}\mleft[ \hat{\mathcal{R}}_\text{\normalfont DR}(\pi) \mright]
= \frac{1}{N} \sum_{d \in D} \frac{{\omega}_{d}^2}{\hat{\rho}_{d}^2}
\Big(
\mathds{V}\mleft[ c(d) \mright] 
+ \mathds{V}\mleft[ {\beta}_{k(d)} \mright]
- 2 \mathds{C}\text{\normalfont ov}\mleft(c(d), {\beta}_{k(d)}\mright) + \mathds{V}\mleft[ {\alpha}_{k(d)} \mright]
\mleft(
\hat{R}_d^2 - 2 \hat{R}_d R_d
\mright)
\Big).
\nonumber
\end{align}
\end{corollary}
\begin{proof}
In Theorem~\ref{theorem:drvariance} replace $\hat{\alpha}$ and $\hat{\beta}$ with $\alpha$ and $\beta$ respectively and then use Lemma~\ref{lemma:covariance} to replace $\mathds{C}\text{\normalfont ov}\big(c(d), {\alpha}_{k(d)}\big)$.
\end{proof}

\begin{theorem}
\label{eq:proof:variancedecrease}
If $\hat{\alpha}$ and $\hat{\beta}$ are correct and the regression model predicts each preference $\hat{R}_d$ between $0$ and twice the true $R_d$ value then the variance of the \ac{DR} estimator (Eq.~\ref{eq:dr}) is less or equal to that of the \ac{IPS} estimator (Eq.~\ref{eq:ips}):
\begin{equation}
\mleft(\hat{\alpha} = \alpha \land \hat{\beta} = \beta
\land
\mleft(\forall d \in D, \; 0 \leq \hat{R}_d \leq 2R_d \mright)\mright)
\longrightarrow
\mathds{V}\mleft[ \hat{\mathcal{R}}_\text{\normalfont DR}(\pi) \mright] \leq \mathds{V}\mleft[ \hat{\mathcal{R}}_\text{\normalfont IPS}(\pi) \mright].
\end{equation}
\end{theorem}
\begin{proof}
Comparing Corollary~\ref{theorem:ipssimplebias} and~\ref{lemma:variance} reveals that:
\begin{equation}
\mleft( \hat{\alpha} = \alpha \land \hat{\beta} = \beta
\land
\mleft(\forall d \in D, \;  \hat{R}_d^2 - 2 \hat{R}_d R_d \leq 0  \mright) \mright)
\longrightarrow
\mathds{V}\mleft[ \hat{\mathcal{R}}_\text{DR}(\pi) \mright] \leq \mathds{V}\mleft[ \hat{\mathcal{R}}_\text{IPS}(\pi) \mright].
\label{eq:proof:allmd}
\end{equation}
For a single $\hat{R}_d$ the following holds:
\begin{equation}
0 \leq
\hat{R}_d \leq 2 R_d \longrightarrow
\hat{R}_d^2 - 2 \hat{R}_d R_d \leq 0
.
\label{eq:proof:singlemd}
\end{equation}
Theorem~\ref{theorem:drvariance} follows directly from Eq.~\ref{eq:proof:allmd} and~\ref{eq:proof:singlemd}.
\end{proof}

\section{Bias of The Cross-Entropy Estimator}
\label{appendix:loglikelihoodbias}

\begin{theorem}
\label{theorem:CEbias}
The $\widehat{\mathcal{L}}$ estimator (Eq.~\ref{eq:CEestimator}) has the following bias:
\begin{equation}
\begin{split}
&\mathbb{E}_{c,y\sim \pi_0}\big[\widehat{\mathcal{L}}(\hat{R})\big]
- \mathcal{L}(\hat{R})
= \sum_{d \in D}\frac{1}{\hat{\rho}_d}\big(
\big(
(\hat{\rho}_d - \rho_d) R_d + \mathbb{E}_{y\sim \pi_0}\big[\hat{\beta}_{k(d)} - \beta_{k(d)}\big]\big)
\log(\hat{R}_d)
\\&\hspace{3.6cm} +
\big(\mathbb{E}_{y\sim \pi_0}\big[\beta_{k(d)} - \hat{\beta}_{k(d)} - \hat{\alpha}_{k(d)}\big] + \hat{\rho}_d + (\rho_d - \hat{\rho}_d) R_d\big)\log(1 -\hat{R}_d)
\big).
\end{split}
\end{equation}
\end{theorem}
\begin{proof}
First, we consider the expected value of $\widehat{\mathcal{L}}(\hat{R})$:
\begin{align}
\underset{c,y\sim \pi_0}{\mathbb{E}}\big[\widehat{\mathcal{L}}(\hat{R}) \big]
 &= - \sum_{d \in D} \frac{1}{\hat{\rho}_d}\big( \underset{c,y\sim \pi_0}{\mathbb{E}}\big[ c(d) - \hat{\beta}_{k(d)}\big]\log(\hat{R}_d)
+ \underset{c,y\sim \pi_0}{\mathbb{E}}\big[\hat{\alpha}_{k(d)} + \hat{\beta}_{k(d)} - c(d) \big]\log(1 -\hat{R}_d)\big)
 \label{eq:CEbiasstep} \nonumber  \\ &
  = - \sum_{d \in D} \frac{1}{\hat{\rho}_d}\Big( \big(\rho_d R_d + \mathbb{E}_{y\sim \pi_0}\big[  \beta_{k(d)} - \hat{\beta}_{k(d)}\big]\big)\log(\hat{R}_d)
  \\& \hspace{1.8cm}
 + \big(\mathbb{E}_{y\sim \pi_0}\big[\hat{\alpha}_{k(d)} + \hat{\beta}_{k(d)} - \beta_{k(d)} \big] - \rho_d R_d\big)\log(1 -\hat{R}_d)\Big).
 \nonumber
\end{align}
Subtract Eq.~\ref{eq:trueCEloss} from the result of Eq.~\ref{eq:CEbiasstep} to prove Theorem~\ref{theorem:CEbias}.
\end{proof}

\begin{lemma}
\label{lemma:expalpha2}
Following Lemma~\ref{lemma:rho} and Lemma~\ref{lemma:expalpha1}:
\begin{equation}
(\hat{\alpha} = \alpha \land (\forall d \in D, \;  \hat{\pi}_0(d) = \pi_0(d) \land \rho_d \geq \tau ))
 \longrightarrow (\forall d \in D, \; \mathds{E}_{y\sim \pi_0}[\hat{\alpha}_{k(d)}] = \hat{\rho}_d ).
\end{equation}
\end{lemma}

\begin{theorem}
\label{eq:proof:nlltheorem}
$\widehat{\mathcal{L}}(\hat{R})$ is unbiased when $\hat{\alpha}$, $\hat{\beta}$ and $\hat{\pi}_0$ are correctly estimated and clipping has no effect:
\begin{equation}
(\hat{\alpha} = \alpha \land \hat{\beta} = \beta \land
(\forall d \in D, \; \hat{\pi}_0(d) = \pi_0(d)  \land \rho_d \geq \tau ))
\longrightarrow
\mathbb{E}_{y\sim \pi_0}\big[\widehat{\mathcal{L}}(d)\big] = \mathcal{L}(d).
\end{equation}
\end{theorem}
\begin{proof}
Theorem~\ref{theorem:CEbias} reveals an unbiasedness condition:
\begin{equation}
\Big(\forall d \in D, \;
\hat{\rho}_d = \rho_d
\land
\mathds{E}_{y\sim \pi_0}\big[\hat{\beta}_{k(d)}\big] = \mathds{E}_{y\sim \pi_0}\big[\beta_{k(d)}\big]
 \land
 \mathds{E}_{y\sim \pi_0}\mleft[\hat{\alpha}_{k(d)}\mright] = \hat{\rho}_d
 \Big)
\rightarrow
\mathds{E}_{y\sim \pi_0}\big[\widehat{\mathcal{L}}(\hat{R})\big] = \mathcal{L}(\hat{R}).
\label{eq:proof:nll1}
\end{equation}
From Lemma~\ref{lemma:rho}, \ref{eq:proof:betaips} and~\ref{lemma:expalpha2} it follows that:
\begin{align}
&(\hat{\alpha} = \alpha \land \hat{\beta} = \beta \land
(\forall d \in D, \;  \hat{\pi}_0(d) = \pi_0(d) \land \rho_d \geq \tau ))
\label{eq:proof:nll2}
\\&\hspace{2.5cm}
\longrightarrow
\big(\forall d \in D, \;
\hat{\rho}_d = \rho_d
\land
\mathds{E}_{y\sim \pi_0}\big[\hat{\beta}_{k(d)}\big] = \mathds{E}_{y\sim \pi_0}\mleft[\beta_{k(d)}\mright]
 \land
 \mathds{E}_{y\sim \pi_0}\mleft[\hat{\alpha}_{k(d)}\mright] = \hat{\rho}_d
 \big).
 \nonumber
\end{align}
Combining Eq.~\ref{eq:proof:nll1} and~\ref{eq:proof:nll2} directly proves Theorem~\ref{eq:proof:nlltheorem}.
\end{proof}

\section{Main results evaluated with Discounted cumulative gain}
\label{appendix:ndcg}

The main results presented in Section~\ref{sec:results} used ECP (Eq.~\ref{eq:reward}) as the metric of performance.
As argued in Section~\ref{sec:ltrgoal}, we think ECP is the most appropriate metric as it is based on the actual user model in the simulation, i.e.\ it utilizes the true $\alpha$ and $\beta$ values.
Nevertheless, this makes it harder to compare our results with previous work that relies on more traditional metrics.

To better enable such comparisons, and to verify whether our conclusions translate to other metrics, Table~\ref{tab:dcgresults} reports the same results as Table~\ref{tab:results} but in terms of normalized \acl{DCG} (NDCG)~\citep{jarvelin2002cumulated}:
\begin{equation}
DCG@K(y) = \sum^K_{k=1} \frac{R_{y_k}}{\log_2(k + 1)},
\qquad
NDCG@K(y) = \frac{DCG@K(y)}{\max_{y'} DCG@K(y')}.
\end{equation}
Comparing Table~\ref{tab:results} with Table~\ref{tab:dcgresults} reveals that both tables show the same trends and relative differences between the different methods.
This confirms that the conclusions that were made from comparisons in Section~\ref{sec:results} are still valid when measuring with NDCG instead of ECP.
In other words, Table~\ref{tab:dcgresults} shows that even when evaluating with NDCG, the performance improvements of \ac{DR} and \ac{DM} over \ac{IPS} and other baselines remain very clear.

{
\setlength{\tabcolsep}{0.00cm}
\begin{table*}[h]
\centering
\caption{
NDCG reached using different estimators in three different settings and three datasets for several $N$ values: the number of displayed rankings in the simulated training set.
In addition, the NDCG of the logging policy and a model trained on the ground-truth (Full-Information) are included to indicate estimated lower and upper bounds on possible performance respectively.
Top part: NDCG@5 in the top-5 setting with known $\alpha$ and $\beta$ bias parameters; middle part:
NDCG@5 in the top-5 setting with estimated $\hat{\alpha}$ and $\hat{\beta}$;
bottom part: NDCG in the full-ranking setting (no cutoff) with known $\alpha$ and $\beta$ bias parameters.
Reported numbers are averages over 20 independent runs evaluated on held-out test-sets, brackets display the standard deviation (logging policy deviation is ommitted since it did not vary between runs).
Bold numbers indicate the highest performance per setting, dataset and $N$ combination.
Statistical differences with our \ac{DR} estimator are measured via a two-sided student-t test,
$^{\tiny \blacktriangledown}$ and $^{\tiny \blacktriangle}$ indicate methods with significantly lower or higher NDCG with $p<0.01$ respectively; additionally, $^{\tiny \triangledown}$ and $^{\tiny \triangle}$ indicate significant differences with $p<0.05$.
}
\label{tab:dcgresults}
\vspace{0.5\baselineskip}
\resizebox{\textwidth}{!}{
\begin{tabular}{ l c c c | c c c | c c c}
&\multicolumn{3}{c}{\footnotesize Yahoo! Webscope}
&\multicolumn{3}{c}{\footnotesize MSLR-WEB30k}
&\multicolumn{3}{c}{\footnotesize Istella}
\\ 
&\footnotesize $N = 10^4$& \footnotesize $N = 10^6$ & \footnotesize $N = 10^9$& \footnotesize $N = 10^4$ & \footnotesize $N = 10^6$ & \footnotesize $N = 10^9$& \footnotesize $N = 10^4$ & \footnotesize $N = 10^6$ & \footnotesize $N = 10^9$\\
 \toprule 
&\multicolumn{9}{c}{ \footnotesize \emph{Top-5 Setting with Known Bias Parameters} } \\ \midrule
\multicolumn{1}{l}{\footnotesize Logging } & \footnotesize 0.700 \phantom{\tiny(0.000)}\phantom{$^{\tiny -}$} & \footnotesize 0.700 \phantom{\tiny(0.000)}\phantom{$^{\tiny -}$} & \footnotesize 0.700 \phantom{\tiny(0.000)}\phantom{$^{\tiny -}$} & \footnotesize 0.465 \phantom{\tiny(0.002)}\phantom{$^{\tiny -}$} & \footnotesize 0.465 \phantom{\tiny(0.002)}\phantom{$^{\tiny -}$} & \footnotesize 0.465 \phantom{\tiny(0.002)}\phantom{$^{\tiny -}$} & \footnotesize 0.539 \phantom{\tiny(0.000)}\phantom{$^{\tiny -}$} & \footnotesize 0.539 \phantom{\tiny(0.000)}\phantom{$^{\tiny -}$} & \footnotesize 0.539 \phantom{\tiny(0.000)}\phantom{$^{\tiny -}$} \\ 
\multicolumn{1}{l}{\footnotesize Full-Info. } & \footnotesize 0.767 {\tiny(0.002)}$^{\tiny \blacktriangle}$ & \footnotesize 0.767 {\tiny(0.002)}$^{\tiny \blacktriangle}$ & \footnotesize 0.767 {\tiny(0.002)}$^{\tiny \blacktriangle}$ & \footnotesize 0.541 {\tiny(0.002)}$^{\tiny \blacktriangle}$ & \footnotesize 0.541 {\tiny(0.002)}$^{\tiny \blacktriangle}$ & \footnotesize 0.541 {\tiny(0.002)}\phantom{$^{\tiny -}$} & \footnotesize 0.639 {\tiny(0.007)}$^{\tiny \blacktriangle}$ & \footnotesize 0.639 {\tiny(0.007)}$^{\tiny \blacktriangle}$ & \footnotesize 0.639 {\tiny(0.007)}$^{\tiny \triangledown}$ \\  \cline{2-10} 
\multicolumn{1}{l}{\footnotesize Naive } &\footnotesize 0.706 {\tiny(0.004)}$^{\tiny \blacktriangledown}$ &\footnotesize 0.708 {\tiny(0.001)}$^{\tiny \blacktriangledown}$ &\footnotesize 0.709 {\tiny(0.001)}$^{\tiny \blacktriangledown}$ &\footnotesize 0.467 {\tiny(0.003)}$^{\tiny \blacktriangledown}$ &\footnotesize 0.470 {\tiny(0.003)}$^{\tiny \blacktriangledown}$ &\footnotesize 0.470 {\tiny(0.002)}$^{\tiny \blacktriangledown}$ &\footnotesize 0.544 {\tiny(0.006)}$^{\tiny \blacktriangledown}$ &\footnotesize 0.554 {\tiny(0.002)}$^{\tiny \blacktriangledown}$ &\footnotesize 0.555 {\tiny(0.001)}$^{\tiny \blacktriangledown}$ \\
\multicolumn{1}{l}{\footnotesize DM (prev) } &\footnotesize 0.716 {\tiny(0.005)}$^{\tiny \blacktriangledown}$ &\footnotesize 0.727 {\tiny(0.004)}$^{\tiny \blacktriangledown}$ &\footnotesize 0.740 {\tiny(0.004)}$^{\tiny \blacktriangledown}$ &\footnotesize 0.475 {\tiny(0.011)}$^{\tiny \triangledown}$ &\footnotesize 0.510 {\tiny(0.008)}$^{\tiny \blacktriangledown}$ &\footnotesize 0.506 {\tiny(0.004)}$^{\tiny \blacktriangledown}$ &\footnotesize 0.493 {\tiny(0.023)}$^{\tiny \blacktriangledown}$ &\footnotesize 0.543 {\tiny(0.014)}$^{\tiny \blacktriangledown}$ &\footnotesize 0.609 {\tiny(0.003)}$^{\tiny \blacktriangledown}$ \\
\multicolumn{1}{l}{\footnotesize RPS } &\footnotesize 0.727 {\tiny(0.003)}\phantom{$^{\tiny -}$} &\footnotesize 0.747 {\tiny(0.002)}$^{\tiny \blacktriangledown}$ &\footnotesize 0.748 {\tiny(0.003)}$^{\tiny \blacktriangledown}$ &\footnotesize \textbf{0.500 {\tiny(0.004)}}$^{\tiny \blacktriangle}$ &\footnotesize 0.527 {\tiny(0.003)}$^{\tiny \blacktriangledown}$ &\footnotesize 0.528 {\tiny(0.003)}$^{\tiny \blacktriangledown}$ &\footnotesize \textbf{0.577 {\tiny(0.002)}}$^{\tiny \blacktriangle}$ &\footnotesize 0.600 {\tiny(0.003)}$^{\tiny \blacktriangledown}$ &\footnotesize 0.601 {\tiny(0.003)}$^{\tiny \blacktriangledown}$ \\
\multicolumn{1}{l}{\footnotesize IPS } &\footnotesize 0.702 {\tiny(0.004)}$^{\tiny \blacktriangledown}$ &\footnotesize 0.734 {\tiny(0.003)}$^{\tiny \blacktriangledown}$ &\footnotesize 0.752 {\tiny(0.002)}$^{\tiny \blacktriangledown}$ &\footnotesize 0.465 {\tiny(0.002)}$^{\tiny \blacktriangledown}$ &\footnotesize 0.500 {\tiny(0.003)}$^{\tiny \blacktriangledown}$ &\footnotesize 0.522 {\tiny(0.004)}$^{\tiny \blacktriangledown}$ &\footnotesize 0.545 {\tiny(0.011)}$^{\tiny \blacktriangledown}$ &\footnotesize 0.602 {\tiny(0.004)}$^{\tiny \blacktriangledown}$ &\footnotesize 0.625 {\tiny(0.004)}$^{\tiny \blacktriangledown}$ \\
\multicolumn{1}{l}{\footnotesize DM (ours) } &\footnotesize 0.727 {\tiny(0.003)}\phantom{$^{\tiny -}$} &\footnotesize 0.752 {\tiny(0.002)}$^{\tiny \blacktriangledown}$ &\footnotesize 0.762 {\tiny(0.001)}$^{\tiny \blacktriangledown}$ &\footnotesize 0.477 {\tiny(0.016)}\phantom{$^{\tiny -}$} &\footnotesize 0.529 {\tiny(0.003)}$^{\tiny \blacktriangledown}$ &\footnotesize 0.536 {\tiny(0.003)}$^{\tiny \blacktriangledown}$ &\footnotesize 0.551 {\tiny(0.011)}$^{\tiny \blacktriangledown}$ &\footnotesize 0.609 {\tiny(0.007)}$^{\tiny \blacktriangledown}$ &\footnotesize 0.630 {\tiny(0.005)}$^{\tiny \blacktriangledown}$ \\
\multicolumn{1}{l}{\footnotesize DR (ours) } &\footnotesize \textbf{0.730 {\tiny(0.005)}}\phantom{$^{\tiny -}$} &\footnotesize \textbf{0.755 {\tiny(0.002)}}\phantom{$^{\tiny -}$} &\footnotesize \textbf{0.765 {\tiny(0.001)}}\phantom{$^{\tiny -}$} &\footnotesize 0.484 {\tiny(0.012)}\phantom{$^{\tiny -}$} &\footnotesize \textbf{0.534 {\tiny(0.003)}}\phantom{$^{\tiny -}$} &\footnotesize \textbf{0.541 {\tiny(0.002)}}\phantom{$^{\tiny -}$} &\footnotesize 0.566 {\tiny(0.012)}\phantom{$^{\tiny -}$} &\footnotesize \textbf{0.626 {\tiny(0.008)}}\phantom{$^{\tiny -}$} &\footnotesize \textbf{0.642 {\tiny(0.003)}}\phantom{$^{\tiny -}$} \\
\midrule
&\multicolumn{9}{c}{ \footnotesize \emph{Top-5 Setting with Estimated Bias Parameters} } \\ \midrule
\multicolumn{1}{l}{\footnotesize Logging } & \footnotesize 0.700 \phantom{\tiny(0.000)}\phantom{$^{\tiny -}$} & \footnotesize 0.700 \phantom{\tiny(0.000)}\phantom{$^{\tiny -}$} & \footnotesize 0.700 \phantom{\tiny(0.000)}\phantom{$^{\tiny -}$} & \footnotesize 0.465 \phantom{\tiny(0.002)}\phantom{$^{\tiny -}$} & \footnotesize 0.465 \phantom{\tiny(0.002)}\phantom{$^{\tiny -}$} & \footnotesize 0.465 \phantom{\tiny(0.002)}\phantom{$^{\tiny -}$} & \footnotesize 0.539 \phantom{\tiny(0.000)}\phantom{$^{\tiny -}$} & \footnotesize 0.539 \phantom{\tiny(0.000)}\phantom{$^{\tiny -}$} & \footnotesize 0.539 \phantom{\tiny(0.000)}\phantom{$^{\tiny -}$} \\ 
\multicolumn{1}{l}{\footnotesize Full-Info. } & \footnotesize 0.767 {\tiny(0.002)}$^{\tiny \blacktriangle}$ & \footnotesize 0.767 {\tiny(0.002)}$^{\tiny \blacktriangle}$ & \footnotesize 0.767 {\tiny(0.002)}$^{\tiny \blacktriangle}$ & \footnotesize 0.541 {\tiny(0.002)}$^{\tiny \blacktriangle}$ & \footnotesize 0.541 {\tiny(0.002)}$^{\tiny \blacktriangle}$ & \footnotesize 0.541 {\tiny(0.002)}\phantom{$^{\tiny -}$} & \footnotesize 0.639 {\tiny(0.007)}$^{\tiny \blacktriangle}$ & \footnotesize 0.639 {\tiny(0.007)}$^{\tiny \blacktriangle}$ & \footnotesize 0.639 {\tiny(0.007)}\phantom{$^{\tiny -}$} \\  \cline{2-10} 
\multicolumn{1}{l}{\footnotesize Naive } &\footnotesize 0.705 {\tiny(0.003)}$^{\tiny \blacktriangledown}$ &\footnotesize 0.708 {\tiny(0.001)}$^{\tiny \blacktriangledown}$ &\footnotesize 0.709 {\tiny(0.001)}$^{\tiny \blacktriangledown}$ &\footnotesize 0.466 {\tiny(0.003)}$^{\tiny \blacktriangledown}$ &\footnotesize 0.470 {\tiny(0.003)}$^{\tiny \blacktriangledown}$ &\footnotesize 0.470 {\tiny(0.002)}$^{\tiny \blacktriangledown}$ &\footnotesize 0.543 {\tiny(0.006)}$^{\tiny \triangledown}$ &\footnotesize 0.555 {\tiny(0.002)}$^{\tiny \blacktriangledown}$ &\footnotesize 0.554 {\tiny(0.004)}$^{\tiny \blacktriangledown}$ \\
\multicolumn{1}{l}{\footnotesize DM (prev) } &\footnotesize 0.727 {\tiny(0.004)}$^{\tiny \triangledown}$ &\footnotesize 0.742 {\tiny(0.004)}$^{\tiny \blacktriangledown}$ &\footnotesize 0.737 {\tiny(0.002)}$^{\tiny \blacktriangledown}$ &\footnotesize 0.481 {\tiny(0.012)}\phantom{$^{\tiny -}$} &\footnotesize 0.507 {\tiny(0.009)}$^{\tiny \blacktriangledown}$ &\footnotesize 0.501 {\tiny(0.005)}$^{\tiny \blacktriangledown}$ &\footnotesize 0.532 {\tiny(0.013)}$^{\tiny \blacktriangledown}$ &\footnotesize 0.566 {\tiny(0.016)}$^{\tiny \blacktriangledown}$ &\footnotesize 0.612 {\tiny(0.005)}$^{\tiny \blacktriangledown}$ \\
\multicolumn{1}{l}{\footnotesize RPS } &\footnotesize 0.726 {\tiny(0.003)}$^{\tiny \blacktriangledown}$ &\footnotesize 0.748 {\tiny(0.002)}$^{\tiny \blacktriangledown}$ &\footnotesize 0.748 {\tiny(0.003)}$^{\tiny \blacktriangledown}$ &\footnotesize \textbf{0.500 {\tiny(0.004)}}$^{\tiny \blacktriangle}$ &\footnotesize 0.526 {\tiny(0.003)}$^{\tiny \blacktriangledown}$ &\footnotesize 0.528 {\tiny(0.003)}$^{\tiny \blacktriangledown}$ &\footnotesize \textbf{0.579 {\tiny(0.003)}}$^{\tiny \blacktriangle}$ &\footnotesize 0.600 {\tiny(0.002)}$^{\tiny \blacktriangledown}$ &\footnotesize 0.601 {\tiny(0.003)}$^{\tiny \blacktriangledown}$ \\
\multicolumn{1}{l}{\footnotesize IPS } &\footnotesize 0.703 {\tiny(0.004)}$^{\tiny \blacktriangledown}$ &\footnotesize 0.735 {\tiny(0.003)}$^{\tiny \blacktriangledown}$ &\footnotesize 0.753 {\tiny(0.001)}$^{\tiny \blacktriangledown}$ &\footnotesize 0.465 {\tiny(0.002)}$^{\tiny \blacktriangledown}$ &\footnotesize 0.498 {\tiny(0.004)}$^{\tiny \blacktriangledown}$ &\footnotesize 0.520 {\tiny(0.004)}$^{\tiny \blacktriangledown}$ &\footnotesize 0.545 {\tiny(0.011)}\phantom{$^{\tiny -}$} &\footnotesize 0.600 {\tiny(0.004)}$^{\tiny \blacktriangledown}$ &\footnotesize 0.623 {\tiny(0.006)}$^{\tiny \blacktriangledown}$ \\
\multicolumn{1}{l}{\footnotesize DM (ours) } &\footnotesize \textbf{0.730 {\tiny(0.004)}}\phantom{$^{\tiny -}$} &\footnotesize 0.750 {\tiny(0.005)}$^{\tiny \blacktriangledown}$ &\footnotesize 0.763 {\tiny(0.002)}$^{\tiny \triangledown}$ &\footnotesize 0.473 {\tiny(0.010)}\phantom{$^{\tiny -}$} &\footnotesize 0.528 {\tiny(0.006)}$^{\tiny \blacktriangledown}$ &\footnotesize 0.537 {\tiny(0.003)}$^{\tiny \blacktriangledown}$ &\footnotesize 0.532 {\tiny(0.015)}$^{\tiny \blacktriangledown}$ &\footnotesize 0.603 {\tiny(0.014)}$^{\tiny \blacktriangledown}$ &\footnotesize 0.629 {\tiny(0.003)}$^{\tiny \blacktriangledown}$ \\
\multicolumn{1}{l}{\footnotesize DR (ours) } &\footnotesize \textbf{0.730 {\tiny(0.003)}}\phantom{$^{\tiny -}$} &\footnotesize \textbf{0.756 {\tiny(0.001)}}\phantom{$^{\tiny -}$} &\footnotesize \textbf{0.765 {\tiny(0.002)}}\phantom{$^{\tiny -}$} &\footnotesize 0.479 {\tiny(0.012)}\phantom{$^{\tiny -}$} &\footnotesize \textbf{0.532 {\tiny(0.004)}}\phantom{$^{\tiny -}$} &\footnotesize \textbf{0.541 {\tiny(0.002)}}\phantom{$^{\tiny -}$} &\footnotesize 0.552 {\tiny(0.019)}\phantom{$^{\tiny -}$} &\footnotesize \textbf{0.624 {\tiny(0.005)}}\phantom{$^{\tiny -}$} &\footnotesize \textbf{0.640 {\tiny(0.004)}}\phantom{$^{\tiny -}$} \\
\midrule
&\multicolumn{9}{c}{ \footnotesize \emph{Full-Ranking Setting with Known Bias Parameters} } \\ \midrule
\multicolumn{1}{l}{\footnotesize Logging } & \footnotesize 0.858 \phantom{\tiny(0.000)}\phantom{$^{\tiny -}$} & \footnotesize 0.858 \phantom{\tiny(0.000)}\phantom{$^{\tiny -}$} & \footnotesize 0.858 \phantom{\tiny(0.000)}\phantom{$^{\tiny -}$} & \footnotesize 0.746 \phantom{\tiny(0.001)}\phantom{$^{\tiny -}$} & \footnotesize 0.746 \phantom{\tiny(0.001)}\phantom{$^{\tiny -}$} & \footnotesize 0.746 \phantom{\tiny(0.001)}\phantom{$^{\tiny -}$} & \footnotesize 0.728 \phantom{\tiny(0.000)}\phantom{$^{\tiny -}$} & \footnotesize 0.728 \phantom{\tiny(0.000)}\phantom{$^{\tiny -}$} & \footnotesize 0.728 \phantom{\tiny(0.000)}\phantom{$^{\tiny -}$} \\ 
\multicolumn{1}{l}{\footnotesize Full-Info. } & \footnotesize 0.888 {\tiny(0.001)}$^{\tiny \blacktriangle}$ & \footnotesize 0.888 {\tiny(0.001)}$^{\tiny \blacktriangle}$ & \footnotesize 0.888 {\tiny(0.001)}\phantom{$^{\tiny -}$} & \footnotesize 0.775 {\tiny(0.002)}$^{\tiny \blacktriangle}$ & \footnotesize 0.775 {\tiny(0.002)}$^{\tiny \blacktriangle}$ & \footnotesize 0.775 {\tiny(0.002)}\phantom{$^{\tiny -}$} & \footnotesize 0.785 {\tiny(0.003)}$^{\tiny \blacktriangle}$ & \footnotesize 0.785 {\tiny(0.003)}$^{\tiny \blacktriangle}$ & \footnotesize 0.785 {\tiny(0.003)}$^{\tiny \triangledown}$ \\  \cline{2-10} 
\multicolumn{1}{l}{\footnotesize Naive } &\footnotesize 0.859 {\tiny(0.002)}$^{\tiny \blacktriangledown}$ &\footnotesize 0.859 {\tiny(0.000)}$^{\tiny \blacktriangledown}$ &\footnotesize 0.859 {\tiny(0.000)}$^{\tiny \blacktriangledown}$ &\footnotesize 0.745 {\tiny(0.001)}$^{\tiny \blacktriangledown}$ &\footnotesize 0.746 {\tiny(0.001)}$^{\tiny \blacktriangledown}$ &\footnotesize 0.746 {\tiny(0.001)}$^{\tiny \blacktriangledown}$ &\footnotesize 0.724 {\tiny(0.003)}$^{\tiny \blacktriangledown}$ &\footnotesize 0.729 {\tiny(0.001)}$^{\tiny \blacktriangledown}$ &\footnotesize 0.729 {\tiny(0.001)}$^{\tiny \blacktriangledown}$ \\
\multicolumn{1}{l}{\footnotesize DM (prev) } &\footnotesize 0.858 {\tiny(0.001)}$^{\tiny \blacktriangledown}$ &\footnotesize 0.856 {\tiny(0.001)}$^{\tiny \blacktriangledown}$ &\footnotesize 0.849 {\tiny(0.008)}$^{\tiny \blacktriangledown}$ &\footnotesize 0.743 {\tiny(0.002)}$^{\tiny \blacktriangledown}$ &\footnotesize 0.743 {\tiny(0.002)}$^{\tiny \blacktriangledown}$ &\footnotesize 0.744 {\tiny(0.002)}$^{\tiny \blacktriangledown}$ &\footnotesize 0.732 {\tiny(0.001)}$^{\tiny \blacktriangledown}$ &\footnotesize 0.737 {\tiny(0.001)}$^{\tiny \blacktriangledown}$ &\footnotesize 0.736 {\tiny(0.002)}$^{\tiny \blacktriangledown}$ \\
\multicolumn{1}{l}{\footnotesize RPS } &\footnotesize 0.861 {\tiny(0.001)}$^{\tiny \blacktriangledown}$ &\footnotesize 0.861 {\tiny(0.000)}$^{\tiny \blacktriangledown}$ &\footnotesize 0.861 {\tiny(0.000)}$^{\tiny \blacktriangledown}$ &\footnotesize 0.748 {\tiny(0.002)}$^{\tiny \blacktriangledown}$ &\footnotesize 0.748 {\tiny(0.001)}$^{\tiny \blacktriangledown}$ &\footnotesize 0.610 {\tiny(0.001)}$^{\tiny \blacktriangledown}$ &\footnotesize 0.737 {\tiny(0.004)}$^{\tiny \triangledown}$ &\footnotesize 0.741 {\tiny(0.000)}$^{\tiny \blacktriangledown}$ &\footnotesize 0.701 {\tiny(0.014)}$^{\tiny \blacktriangledown}$ \\
\multicolumn{1}{l}{\footnotesize IPS } &\footnotesize 0.859 {\tiny(0.003)}$^{\tiny \blacktriangledown}$ &\footnotesize 0.877 {\tiny(0.001)}$^{\tiny \blacktriangledown}$ &\footnotesize \textbf{0.888 {\tiny(0.001)}}\phantom{$^{\tiny -}$} &\footnotesize 0.745 {\tiny(0.001)}$^{\tiny \blacktriangledown}$ &\footnotesize 0.762 {\tiny(0.002)}$^{\tiny \blacktriangledown}$ &\footnotesize \textbf{0.774 {\tiny(0.002)}}\phantom{$^{\tiny -}$} &\footnotesize 0.725 {\tiny(0.005)}$^{\tiny \blacktriangledown}$ &\footnotesize 0.769 {\tiny(0.004)}$^{\tiny \blacktriangledown}$ &\footnotesize 0.787 {\tiny(0.003)}\phantom{$^{\tiny -}$} \\
\multicolumn{1}{l}{\footnotesize DM (ours) } &\footnotesize \textbf{0.866 {\tiny(0.002)}}\phantom{$^{\tiny -}$} &\footnotesize 0.881 {\tiny(0.001)}$^{\tiny \blacktriangledown}$ &\footnotesize 0.885 {\tiny(0.002)}$^{\tiny \blacktriangledown}$ &\footnotesize 0.752 {\tiny(0.006)}$^{\tiny \triangledown}$ &\footnotesize \textbf{0.772 {\tiny(0.002)}}\phantom{$^{\tiny -}$} &\footnotesize \textbf{0.774 {\tiny(0.002)}}\phantom{$^{\tiny -}$} &\footnotesize 0.738 {\tiny(0.007)}\phantom{$^{\tiny -}$} &\footnotesize \textbf{0.783 {\tiny(0.003)}}$^{\tiny \blacktriangle}$ &\footnotesize \textbf{0.793 {\tiny(0.002)}}$^{\tiny \blacktriangle}$ \\
\multicolumn{1}{l}{\footnotesize DR (ours) } &\footnotesize \textbf{0.866 {\tiny(0.003)}}\phantom{$^{\tiny -}$} &\footnotesize \textbf{0.882 {\tiny(0.001)}}\phantom{$^{\tiny -}$} &\footnotesize \textbf{0.888 {\tiny(0.001)}}\phantom{$^{\tiny -}$} &\footnotesize \textbf{0.755 {\tiny(0.004)}}\phantom{$^{\tiny -}$} &\footnotesize 0.771 {\tiny(0.002)}\phantom{$^{\tiny -}$} &\footnotesize \textbf{0.774 {\tiny(0.002)}}\phantom{$^{\tiny -}$} &\footnotesize \textbf{0.741 {\tiny(0.005)}}\phantom{$^{\tiny -}$} &\footnotesize 0.779 {\tiny(0.004)}\phantom{$^{\tiny -}$} &\footnotesize 0.787 {\tiny(0.003)}\phantom{$^{\tiny -}$} \\
\bottomrule

\end{tabular}
}
\end{table*}
}

\clearpage

\bibliographystyle{ACM-Reference-Format}
\bibliography{references}


\begin{thebibliography}{57}


\ifx \showCODEN    \undefined \def \showCODEN     #1{\unskip}     \fi
\ifx \showDOI      \undefined \def \showDOI       #1{#1}\fi
\ifx \showISBNx    \undefined \def \showISBNx     #1{\unskip}     \fi
\ifx \showISBNxiii \undefined \def \showISBNxiii  #1{\unskip}     \fi
\ifx \showISSN     \undefined \def \showISSN      #1{\unskip}     \fi
\ifx \showLCCN     \undefined \def \showLCCN      #1{\unskip}     \fi
\ifx \shownote     \undefined \def \shownote      #1{#1}          \fi
\ifx \showarticletitle \undefined \def \showarticletitle #1{#1}   \fi
\ifx \showURL      \undefined \def \showURL       {\relax}        \fi
\providecommand\bibfield[2]{#2}
\providecommand\bibinfo[2]{#2}
\providecommand\natexlab[1]{#1}
\providecommand\showeprint[2][]{arXiv:#2}

\bibitem[Agarwal et~al\mbox{.}(2019a)]%
        {agarwal2019counterfactual}
\bibfield{author}{\bibinfo{person}{Aman Agarwal}, \bibinfo{person}{Kenta
  Takatsu}, \bibinfo{person}{Ivan Zaitsev}, {and} \bibinfo{person}{Thorsten
  Joachims}.} \bibinfo{year}{2019}\natexlab{a}.
\newblock \showarticletitle{A General Framework for Counterfactual
  Learning-to-Rank}. In \bibinfo{booktitle}{\emph{Proceedings of the 42nd
  International ACM SIGIR Conference on Research \& Development in Information
  Retrieval}}. ACM, \bibinfo{pages}{5--14}.
\newblock


\bibitem[Agarwal et~al\mbox{.}(2019b)]%
        {agarwal2019addressing}
\bibfield{author}{\bibinfo{person}{Aman Agarwal}, \bibinfo{person}{Xuanhui
  Wang}, \bibinfo{person}{Cheng Li}, \bibinfo{person}{Michael Bendersky}, {and}
  \bibinfo{person}{Marc Najork}.} \bibinfo{year}{2019}\natexlab{b}.
\newblock \showarticletitle{Addressing Trust Bias for Unbiased
  Learning-to-Rank}. In \bibinfo{booktitle}{\emph{The World Wide Web
  Conference}}. ACM, \bibinfo{pages}{4--14}.
\newblock


\bibitem[Agarwal et~al\mbox{.}(2019c)]%
        {agarwal2019estimating}
\bibfield{author}{\bibinfo{person}{Aman Agarwal}, \bibinfo{person}{Ivan
  Zaitsev}, \bibinfo{person}{Xuanhui Wang}, \bibinfo{person}{Cheng Li},
  \bibinfo{person}{Marc Najork}, {and} \bibinfo{person}{Thorsten Joachims}.}
  \bibinfo{year}{2019}\natexlab{c}.
\newblock \showarticletitle{Estimating Position Bias without Intrusive
  Interventions}. In \bibinfo{booktitle}{\emph{Proceedings of the Twelfth ACM
  International Conference on Web Search and Data Mining}}. ACM,
  \bibinfo{pages}{474--482}.
\newblock


\bibitem[Ai et~al\mbox{.}(2021)]%
        {ai2020unbiased}
\bibfield{author}{\bibinfo{person}{Qingyao Ai}, \bibinfo{person}{Tao Yang},
  \bibinfo{person}{Huazheng Wang}, {and} \bibinfo{person}{Jiaxin Mao}.}
  \bibinfo{year}{2021}\natexlab{}.
\newblock \showarticletitle{Unbiased Learning to Rank: Online or Offline?}
\newblock \bibinfo{journal}{\emph{ACM Transactions on Information Systems
  (TOIS)}} \bibinfo{volume}{39}, \bibinfo{number}{2} (\bibinfo{year}{2021}),
  \bibinfo{pages}{1--29}.
\newblock


\bibitem[Bekker et~al\mbox{.}(2019)]%
        {bekker2019beyond}
\bibfield{author}{\bibinfo{person}{Jessa Bekker}, \bibinfo{person}{Pieter
  Robberechts}, {and} \bibinfo{person}{Jesse Davis}.}
  \bibinfo{year}{2019}\natexlab{}.
\newblock \showarticletitle{Beyond the Selected Completely at Random Assumption
  for Learning from Positive and Unlabeled Data}. In
  \bibinfo{booktitle}{\emph{Joint European Conference on Machine Learning and
  Knowledge Discovery in Databases}}. Springer, \bibinfo{pages}{71--85}.
\newblock


\bibitem[Burges(2010)]%
        {burges2010ranknet}
\bibfield{author}{\bibinfo{person}{Christopher~J.C. Burges}.}
  \bibinfo{year}{2010}\natexlab{}.
\newblock \bibinfo{booktitle}{\emph{From RankNet to LambdaRank to LambdaMART:
  An Overview}}.
\newblock \bibinfo{type}{{T}echnical {R}eport} MSR-TR-2010-82.
  \bibinfo{institution}{Microsoft}.
\newblock


\bibitem[Chapelle and Chang(2011)]%
        {Chapelle2011}
\bibfield{author}{\bibinfo{person}{Olivier Chapelle} {and} \bibinfo{person}{Yi
  Chang}.} \bibinfo{year}{2011}\natexlab{}.
\newblock \showarticletitle{{Yahoo! Learning to Rank Challenge Overview}}.
\newblock \bibinfo{journal}{\emph{Journal of Machine Learning Research}}
  \bibinfo{volume}{14} (\bibinfo{year}{2011}), \bibinfo{pages}{1--24}.
\newblock


\bibitem[Craswell et~al\mbox{.}(2008)]%
        {craswell2008experimental}
\bibfield{author}{\bibinfo{person}{Nick Craswell}, \bibinfo{person}{Onno
  Zoeter}, \bibinfo{person}{Michael Taylor}, {and} \bibinfo{person}{Bill
  Ramsey}.} \bibinfo{year}{2008}\natexlab{}.
\newblock \showarticletitle{An Experimental Comparison of Click Position-Bias
  Models}. In \bibinfo{booktitle}{\emph{Proceedings of the 2008 International
  Conference on Web Search and Data Mining}}. \bibinfo{pages}{87--94}.
\newblock


\bibitem[Dato et~al\mbox{.}(2016)]%
        {dato2016fast}
\bibfield{author}{\bibinfo{person}{Domenico Dato}, \bibinfo{person}{Claudio
  Lucchese}, \bibinfo{person}{Franco~Maria Nardini}, \bibinfo{person}{Salvatore
  Orlando}, \bibinfo{person}{Raffaele Perego}, \bibinfo{person}{Nicola
  Tonellotto}, {and} \bibinfo{person}{Rossano Venturini}.}
  \bibinfo{year}{2016}\natexlab{}.
\newblock \showarticletitle{Fast Ranking with Additive Ensembles of Oblivious
  and Non-Oblivious Regression Trees}.
\newblock \bibinfo{journal}{\emph{ACM Transactions on Information Systems
  (TOIS)}} \bibinfo{volume}{35}, \bibinfo{number}{2} (\bibinfo{year}{2016}),
  \bibinfo{pages}{Article 15}.
\newblock


\bibitem[Dud{\'\i}k et~al\mbox{.}(2014)]%
        {dudik2014doubly}
\bibfield{author}{\bibinfo{person}{Miroslav Dud{\'\i}k},
  \bibinfo{person}{Dumitru Erhan}, \bibinfo{person}{John Langford}, {and}
  \bibinfo{person}{Lihong Li}.} \bibinfo{year}{2014}\natexlab{}.
\newblock \showarticletitle{Doubly Robust Policy Evaluation and Optimization}.
\newblock \bibinfo{journal}{\emph{Statist. Sci.}} \bibinfo{volume}{29},
  \bibinfo{number}{4} (\bibinfo{year}{2014}), \bibinfo{pages}{485--511}.
\newblock


\bibitem[Fang et~al\mbox{.}(2019)]%
        {fang2019intervention}
\bibfield{author}{\bibinfo{person}{Zhichong Fang}, \bibinfo{person}{Aman
  Agarwal}, {and} \bibinfo{person}{Thorsten Joachims}.}
  \bibinfo{year}{2019}\natexlab{}.
\newblock \showarticletitle{Intervention Harvesting for Context-Dependent
  Examination-Bias estimation}. In \bibinfo{booktitle}{\emph{Proceedings of the
  42nd International ACM SIGIR Conference on Research and Development in
  Information Retrieval}}. \bibinfo{pages}{825--834}.
\newblock


\bibitem[Hofmann et~al\mbox{.}(2013)]%
        {hofmann2013reusing}
\bibfield{author}{\bibinfo{person}{Katja Hofmann}, \bibinfo{person}{Anne
  Schuth}, \bibinfo{person}{Shimon Whiteson}, {and} \bibinfo{person}{Maarten de
  Rijke}.} \bibinfo{year}{2013}\natexlab{}.
\newblock \showarticletitle{Reusing Historical Interaction Data for Faster
  Online Learning to Rank for IR}. In \bibinfo{booktitle}{\emph{Proceedings of
  the Sixth ACM International Conference on Web Search and Data Mining}}. ACM,
  \bibinfo{pages}{183--192}.
\newblock


\bibitem[Horvitz and Thompson(1952)]%
        {horvitz1952generalization}
\bibfield{author}{\bibinfo{person}{Daniel~G Horvitz} {and}
  \bibinfo{person}{Donovan~J Thompson}.} \bibinfo{year}{1952}\natexlab{}.
\newblock \showarticletitle{A Generalization of Sampling Without Replacement
  from a Finite Universe}.
\newblock \bibinfo{journal}{\emph{Journal of the American statistical
  Association}} \bibinfo{volume}{47}, \bibinfo{number}{260}
  (\bibinfo{year}{1952}), \bibinfo{pages}{663--685}.
\newblock


\bibitem[Jagerman et~al\mbox{.}(2020)]%
        {jagerman2020safe}
\bibfield{author}{\bibinfo{person}{Rolf Jagerman}, \bibinfo{person}{Ilya
  Markov}, {and} \bibinfo{person}{Maarten~De Rijke}.}
  \bibinfo{year}{2020}\natexlab{}.
\newblock \showarticletitle{Safe Exploration for Optimizing Contextual
  Bandits}.
\newblock \bibinfo{journal}{\emph{ACM Transactions on Information Systems
  (TOIS)}} \bibinfo{volume}{38}, \bibinfo{number}{3} (\bibinfo{year}{2020}),
  \bibinfo{pages}{1--23}.
\newblock


\bibitem[J{\"a}rvelin and Kek{\"a}l{\"a}inen(2002)]%
        {jarvelin2002cumulated}
\bibfield{author}{\bibinfo{person}{Kalervo J{\"a}rvelin} {and}
  \bibinfo{person}{Jaana Kek{\"a}l{\"a}inen}.} \bibinfo{year}{2002}\natexlab{}.
\newblock \showarticletitle{Cumulated Gain-Based Evaluation of IR Techniques}.
\newblock \bibinfo{journal}{\emph{ACM Transactions on Information Systems
  (TOIS)}} \bibinfo{volume}{20}, \bibinfo{number}{4} (\bibinfo{year}{2002}),
  \bibinfo{pages}{422--446}.
\newblock


\bibitem[Joachims(2002)]%
        {joachims2002optimizing}
\bibfield{author}{\bibinfo{person}{Thorsten Joachims}.}
  \bibinfo{year}{2002}\natexlab{}.
\newblock \showarticletitle{Optimizing Search Engines Using Clickthrough Data}.
  In \bibinfo{booktitle}{\emph{Proceedings of the Eighth ACM SIGKDD
  International Conference on Knowledge Discovery and Data Mining}}. ACM,
  \bibinfo{pages}{133--142}.
\newblock


\bibitem[Joachims et~al\mbox{.}(2017a)]%
        {joachims2017accurately}
\bibfield{author}{\bibinfo{person}{Thorsten Joachims}, \bibinfo{person}{Laura
  Granka}, \bibinfo{person}{Bing Pan}, \bibinfo{person}{Helene Hembrooke},
  {and} \bibinfo{person}{Geri Gay}.} \bibinfo{year}{2017}\natexlab{a}.
\newblock \showarticletitle{Accurately Interpreting Clickthrough Data as
  Implicit Feedback}. In \bibinfo{booktitle}{\emph{ACM SIGIR Forum}},
  Vol.~\bibinfo{volume}{51}. Acm New York, NY, USA, \bibinfo{pages}{4--11}.
\newblock


\bibitem[Joachims et~al\mbox{.}(2017b)]%
        {joachims2017unbiased}
\bibfield{author}{\bibinfo{person}{Thorsten Joachims}, \bibinfo{person}{Adith
  Swaminathan}, {and} \bibinfo{person}{Tobias Schnabel}.}
  \bibinfo{year}{2017}\natexlab{b}.
\newblock \showarticletitle{Unbiased Learning-to-Rank with Biased Feedback}. In
  \bibinfo{booktitle}{\emph{Proceedings of the Tenth ACM International
  Conference on Web Search and Data Mining}}. ACM, \bibinfo{pages}{781--789}.
\newblock


\bibitem[Kang et~al\mbox{.}(2007)]%
        {kang2007demystifying}
\bibfield{author}{\bibinfo{person}{Joseph~DY Kang}, \bibinfo{person}{Joseph~L
  Schafer}, {et~al\mbox{.}}} \bibinfo{year}{2007}\natexlab{}.
\newblock \showarticletitle{Demystifying Double Robustness: A Comparison of
  Alternative Strategies for Estimating a Population Mean from Incomplete
  Data}.
\newblock \bibinfo{journal}{\emph{Statistical science}} \bibinfo{volume}{22},
  \bibinfo{number}{4} (\bibinfo{year}{2007}), \bibinfo{pages}{523--539}.
\newblock


\bibitem[Kiyohara et~al\mbox{.}(2022)]%
        {kiyohara2022doubly}
\bibfield{author}{\bibinfo{person}{Haruka Kiyohara}, \bibinfo{person}{Yuta
  Saito}, \bibinfo{person}{Tatsuya Matsuhiro}, \bibinfo{person}{Yusuke Narita},
  \bibinfo{person}{Nobuyuki Shimizu}, {and} \bibinfo{person}{Yasuo Yamamoto}.}
  \bibinfo{year}{2022}\natexlab{}.
\newblock \showarticletitle{Doubly Robust Off-Policy Evaluation for Ranking
  Policies Under the Cascade Behavior Model}. In
  \bibinfo{booktitle}{\emph{Proceedings of the Fifteenth ACM International
  Conference on Web Search and Data Mining}}. \bibinfo{pages}{487--497}.
\newblock


\bibitem[Komiya\-ma et~al\mbox{.}(2015)]%
        {Komiyama2015}
\bibfield{author}{\bibinfo{person}{Junpei Komiya\-ma}, \bibinfo{person}{Junya
  Honda}, {and} \bibinfo{person}{Hiroshi Nakagawa}.}
  \bibinfo{year}{2015}\natexlab{}.
\newblock \showarticletitle{Optimal Regret Analysis of Thompson Sampling in
  Stochastic Multi-armed Bandit Problem with Multiple Plays}. In
  \bibinfo{booktitle}{\emph{Proceedings of the 32Nd International Conference on
  International Conference on Machine Learning - Volume 37}} (Lille, France)
  \emph{(\bibinfo{series}{ICML'15})}. \bibinfo{publisher}{JMLR.org},
  \bibinfo{pages}{1152--1161}.
\newblock


\bibitem[Lagr{\'e}e et~al\mbox{.}(2016)]%
        {lagree2016multiple}
\bibfield{author}{\bibinfo{person}{Paul Lagr{\'e}e}, \bibinfo{person}{Claire
  Vernade}, {and} \bibinfo{person}{Olivier Capp{\'e}}.}
  \bibinfo{year}{2016}\natexlab{}.
\newblock \showarticletitle{Multiple-Play Bandits in the Position-Based Model}.
  In \bibinfo{booktitle}{\emph{Advances in Neural Information Processing
  Systems}}. \bibinfo{pages}{1597--1605}.
\newblock


\bibitem[Li et~al\mbox{.}(2018)]%
        {li2018offline}
\bibfield{author}{\bibinfo{person}{Shuai Li}, \bibinfo{person}{Yasin
  Abbasi-Yadkori}, \bibinfo{person}{Branislav Kveton}, \bibinfo{person}{S
  Muthukrishnan}, \bibinfo{person}{Vishwa Vinay}, {and} \bibinfo{person}{Zheng
  Wen}.} \bibinfo{year}{2018}\natexlab{}.
\newblock \showarticletitle{Offline Evaluation of Ranking Policies with Click
  Models}. In \bibinfo{booktitle}{\emph{Proceedings of the 24th ACM SIGKDD
  International Conference on Knowledge Discovery \& Data Mining}}. ACM,
  \bibinfo{pages}{1685--1694}.
\newblock


\bibitem[Liu(2009)]%
        {liu2009learning}
\bibfield{author}{\bibinfo{person}{Tie-Yan Liu}.}
  \bibinfo{year}{2009}\natexlab{}.
\newblock \showarticletitle{Learning to Rank for Information Retrieval}.
\newblock \bibinfo{journal}{\emph{Foundations and Trends in Information
  Retrieval}} \bibinfo{volume}{3}, \bibinfo{number}{3} (\bibinfo{year}{2009}),
  \bibinfo{pages}{225--331}.
\newblock


\bibitem[Oosterhuis(2021)]%
        {oosterhuis2021computationally}
\bibfield{author}{\bibinfo{person}{Harrie Oosterhuis}.}
  \bibinfo{year}{2021}\natexlab{}.
\newblock \showarticletitle{Computationally Efficient Optimization of
  Plackett-Luce Ranking Models for Relevance and Fairness}. In
  \bibinfo{booktitle}{\emph{Proceedings of the 44th International ACM SIGIR
  Conference on Research and Development in Information Retrieval}} (Virtual
  Event, Canada) \emph{(\bibinfo{series}{SIGIR '21})}.
  \bibinfo{publisher}{ACM}, \bibinfo{pages}{1023–1032}.
\newblock


\bibitem[Oosterhuis and de~Rijke(2018)]%
        {oosterhuis2018differentiable}
\bibfield{author}{\bibinfo{person}{Harrie Oosterhuis} {and}
  \bibinfo{person}{Maarten de Rijke}.} \bibinfo{year}{2018}\natexlab{}.
\newblock \showarticletitle{Differentiable Unbiased Online Learning to Rank}.
  In \bibinfo{booktitle}{\emph{Proceedings of the 27th ACM International
  Conference on Information and Knowledge Management}}. ACM,
  \bibinfo{pages}{1293--1302}.
\newblock


\bibitem[Oosterhuis and de~Rijke(2019)]%
        {oosterhuis2019optimizing}
\bibfield{author}{\bibinfo{person}{Harrie Oosterhuis} {and}
  \bibinfo{person}{Maarten de Rijke}.} \bibinfo{year}{2019}\natexlab{}.
\newblock \showarticletitle{Optimizing Ranking Models in an Online Setting}. In
  \bibinfo{booktitle}{\emph{Advances in Information Retrieval}}.
  \bibinfo{publisher}{Springer International Publishing},
  \bibinfo{address}{Cham}, \bibinfo{pages}{382--396}.
\newblock
\showISBNx{978-3-030-15712-8}


\bibitem[Oosterhuis and de~Rijke(2020)]%
        {oosterhuis2020topkrankings}
\bibfield{author}{\bibinfo{person}{Harrie Oosterhuis} {and}
  \bibinfo{person}{Maarten de Rijke}.} \bibinfo{year}{2020}\natexlab{}.
\newblock \showarticletitle{Policy-Aware Unbiased Learning to Rank for Top-k
  Rankings}. In \bibinfo{booktitle}{\emph{Proceedings of the 43rd International
  ACM SIGIR Conference on Research and Development in Information Retrieval}}.
  \bibinfo{publisher}{ACM}, \bibinfo{pages}{489--498}.
\newblock


\bibitem[Oosterhuis and de~Rijke(2021a)]%
        {oosterhuis2021onlinecounterltr}
\bibfield{author}{\bibinfo{person}{Harrie Oosterhuis} {and}
  \bibinfo{person}{Maarten de Rijke}.} \bibinfo{year}{2021}\natexlab{a}.
\newblock \showarticletitle{Unifying Online and Counterfactual Learning to
  Rank}. In \bibinfo{booktitle}{\emph{Proceedings of the 14th ACM International
  Conference on Web Search and Data Mining (WSDM'21)}}. ACM.
\newblock


\bibitem[Oosterhuis and de~Rijke(2021b)]%
        {oosterhuis2021robust}
\bibfield{author}{\bibinfo{person}{Harrie Oosterhuis} {and}
  \bibinfo{person}{Maarten~de de Rijke}.} \bibinfo{year}{2021}\natexlab{b}.
\newblock \showarticletitle{Robust Generalization and Safe Query-Specialization
  in Counterfactual Learning to Rank}. In \bibinfo{booktitle}{\emph{Proceedings
  of the Web Conference 2021}}. \bibinfo{pages}{158--170}.
\newblock


\bibitem[Ovaisi et~al\mbox{.}(2020)]%
        {ovaisi2020correcting}
\bibfield{author}{\bibinfo{person}{Zohreh Ovaisi}, \bibinfo{person}{Ragib
  Ahsan}, \bibinfo{person}{Yifan Zhang}, \bibinfo{person}{Kathryn Vasilaky},
  {and} \bibinfo{person}{Elena Zheleva}.} \bibinfo{year}{2020}\natexlab{}.
\newblock \showarticletitle{Correcting for Selection Bias in Learning-to-rank
  Systems}. In \bibinfo{booktitle}{\emph{Proceedings of The Web Conference
  2020}}. \bibinfo{pages}{1863--1873}.
\newblock


\bibitem[Qin and Liu(2013)]%
        {qin2013introducing}
\bibfield{author}{\bibinfo{person}{Tao Qin} {and} \bibinfo{person}{Tie-Yan
  Liu}.} \bibinfo{year}{2013}\natexlab{}.
\newblock \showarticletitle{Introducing LETOR 4.0 datasets}.
\newblock \bibinfo{journal}{\emph{arXiv preprint arXiv:1306.2597}}
  (\bibinfo{year}{2013}).
\newblock


\bibitem[Radlinski et~al\mbox{.}(2008)]%
        {radlinski2008does}
\bibfield{author}{\bibinfo{person}{Filip Radlinski}, \bibinfo{person}{Madhu
  Kurup}, {and} \bibinfo{person}{Thorsten Joachims}.}
  \bibinfo{year}{2008}\natexlab{}.
\newblock \showarticletitle{How Does Clickthrough Data Reflect Retrieval
  Quality?}. In \bibinfo{booktitle}{\emph{Proceedings of the 17th ACM
  Conference on Information and Knowledge Management}}. ACM,
  \bibinfo{pages}{43--52}.
\newblock


\bibitem[Richardson et~al\mbox{.}(2007)]%
        {richardson2007predicting}
\bibfield{author}{\bibinfo{person}{Matthew Richardson}, \bibinfo{person}{Ewa
  Dominowska}, {and} \bibinfo{person}{Robert Ragno}.}
  \bibinfo{year}{2007}\natexlab{}.
\newblock \showarticletitle{Predicting Clicks: Estimating the Click-Through
  Rate for New Ads}. In \bibinfo{booktitle}{\emph{Proceedings of the 16th
  International Conference on World Wide Web}}. \bibinfo{pages}{521--530}.
\newblock


\bibitem[Robins et~al\mbox{.}(1994)]%
        {robins1994estimation}
\bibfield{author}{\bibinfo{person}{James~M Robins}, \bibinfo{person}{Andrea
  Rotnitzky}, {and} \bibinfo{person}{Lue~Ping Zhao}.}
  \bibinfo{year}{1994}\natexlab{}.
\newblock \showarticletitle{Estimation of Regression Coefficients when Some
  Regressors are not Always Observed}.
\newblock \bibinfo{journal}{\emph{Journal of the American statistical
  Association}} \bibinfo{volume}{89}, \bibinfo{number}{427}
  (\bibinfo{year}{1994}), \bibinfo{pages}{846--866}.
\newblock


\bibitem[Saito(2020)]%
        {saito2020doubly}
\bibfield{author}{\bibinfo{person}{Yuta Saito}.}
  \bibinfo{year}{2020}\natexlab{}.
\newblock \showarticletitle{Doubly Robust Estimator for Ranking Metrics with
  Post-Click Conversions}. In \bibinfo{booktitle}{\emph{Fourteenth ACM
  Conference on Recommender Systems}}. \bibinfo{pages}{92--100}.
\newblock


\bibitem[Saito and Joachims(2021)]%
        {saito2021counterfactual}
\bibfield{author}{\bibinfo{person}{Yuta Saito} {and} \bibinfo{person}{Thorsten
  Joachims}.} \bibinfo{year}{2021}\natexlab{}.
\newblock \showarticletitle{Counterfactual Learning and Evaluation for
  Recommender Systems: Foundations, Implementations, and Recent Advances}. In
  \bibinfo{booktitle}{\emph{Fifteenth ACM Conference on Recommender Systems}}.
  \bibinfo{pages}{828--830}.
\newblock


\bibitem[Saito et~al\mbox{.}(2020)]%
        {saito2020unbiased}
\bibfield{author}{\bibinfo{person}{Yuta Saito}, \bibinfo{person}{Suguru
  Yaginuma}, \bibinfo{person}{Yuta Nishino}, \bibinfo{person}{Hayato Sakata},
  {and} \bibinfo{person}{Kazuhide Nakata}.} \bibinfo{year}{2020}\natexlab{}.
\newblock \showarticletitle{Unbiased Recommender Learning from
  Missing-not-at-Random Implicit Feedback}. In
  \bibinfo{booktitle}{\emph{Proceedings of the 13th International Conference on
  Web Search and Data Mining}}. \bibinfo{pages}{501--509}.
\newblock


\bibitem[Schuth et~al\mbox{.}(2016)]%
        {schuth2016mgd}
\bibfield{author}{\bibinfo{person}{Anne Schuth}, \bibinfo{person}{Harrie
  Oosterhuis}, \bibinfo{person}{Shimon Whiteson}, {and}
  \bibinfo{person}{Maarten de Rijke}.} \bibinfo{year}{2016}\natexlab{}.
\newblock \showarticletitle{Multileave Gradient Descent for Fast Online
  Learning to Rank}. In \bibinfo{booktitle}{\emph{Proceedings of the Ninth ACM
  International Conference on Web Search and Data Mining}}.
  \bibinfo{pages}{457--466}.
\newblock


\bibitem[Singh and Joachims(2018)]%
        {singh2018fairness}
\bibfield{author}{\bibinfo{person}{Ashudeep Singh} {and}
  \bibinfo{person}{Thorsten Joachims}.} \bibinfo{year}{2018}\natexlab{}.
\newblock \showarticletitle{Fairness of Exposure in Rankings}. In
  \bibinfo{booktitle}{\emph{Proceedings of the 24th ACM SIGKDD International
  Conference on Knowledge Discovery \& Data Mining}}.
  \bibinfo{pages}{2219--2228}.
\newblock


\bibitem[Singh and Joachims(2019)]%
        {singh2019policy}
\bibfield{author}{\bibinfo{person}{Ashudeep Singh} {and}
  \bibinfo{person}{Thorsten Joachims}.} \bibinfo{year}{2019}\natexlab{}.
\newblock \showarticletitle{Policy Learning for Fairness in Ranking}. In
  \bibinfo{booktitle}{\emph{Advances in Neural Information Processing
  Systems}}. \bibinfo{pages}{5426--5436}.
\newblock


\bibitem[Strehl et~al\mbox{.}(2010)]%
        {strehl2010logged}
\bibfield{author}{\bibinfo{person}{Alex Strehl}, \bibinfo{person}{John
  Langford}, \bibinfo{person}{Lihong Li}, {and} \bibinfo{person}{Sham~M
  Kakade}.} \bibinfo{year}{2010}\natexlab{}.
\newblock \showarticletitle{Learning from Logged Implicit Exploration Data}. In
  \bibinfo{booktitle}{\emph{Advances in Neural Information Processing
  Systems}}, \bibfield{editor}{\bibinfo{person}{J.~Lafferty},
  \bibinfo{person}{C.~Williams}, \bibinfo{person}{J.~Shawe-Taylor},
  \bibinfo{person}{R.~Zemel}, {and} \bibinfo{person}{A.~Culotta}} (Eds.),
  Vol.~\bibinfo{volume}{23}. \bibinfo{publisher}{Curran Associates, Inc.}
\newblock


\bibitem[Student(1908)]%
        {student1908probable}
\bibfield{author}{\bibinfo{person}{Student}.} \bibinfo{year}{1908}\natexlab{}.
\newblock \showarticletitle{The Probable Error of a Mean}.
\newblock \bibinfo{journal}{\emph{Biometrika}} (\bibinfo{year}{1908}),
  \bibinfo{pages}{1--25}.
\newblock


\bibitem[Sutton et~al\mbox{.}(1998)]%
        {sutton1998introduction}
\bibfield{author}{\bibinfo{person}{Richard~S Sutton}, \bibinfo{person}{Andrew~G
  Barto}, {et~al\mbox{.}}} \bibinfo{year}{1998}\natexlab{}.
\newblock \bibinfo{booktitle}{\emph{Introduction to Reinforcement Learning}}.
  Vol.~\bibinfo{volume}{135}.
\newblock \bibinfo{publisher}{MIT press Cambridge}.
\newblock


\bibitem[Ustimenko and Prokhorenkova(2020)]%
        {ustimenko2020stochasticrank}
\bibfield{author}{\bibinfo{person}{Aleksei Ustimenko} {and}
  \bibinfo{person}{Liudmila Prokhorenkova}.} \bibinfo{year}{2020}\natexlab{}.
\newblock \showarticletitle{StochasticRank: Global Optimization of Scale-Free
  Discrete Functions}. In \bibinfo{booktitle}{\emph{International Conference on
  Machine Learning}}. PMLR, \bibinfo{pages}{9669--9679}.
\newblock


\bibitem[Vardasbi et~al\mbox{.}(2020a)]%
        {vardasbi2020cascade}
\bibfield{author}{\bibinfo{person}{Ali Vardasbi}, \bibinfo{person}{Maarten de
  Rijke}, {and} \bibinfo{person}{Ilya Markov}.}
  \bibinfo{year}{2020}\natexlab{a}.
\newblock \showarticletitle{Cascade Model-Based Propensity Estimation for
  Counterfactual Learning to Rank}. In \bibinfo{booktitle}{\emph{Proceedings of
  the 43rd International ACM SIGIR Conference on Research and Development in
  Information Retrieval}}. \bibinfo{pages}{2089--2092}.
\newblock


\bibitem[Vardasbi et~al\mbox{.}(2020b)]%
        {vardasbi2020trust}
\bibfield{author}{\bibinfo{person}{Ali Vardasbi}, \bibinfo{person}{Harrie
  Oosterhuis}, {and} \bibinfo{person}{Maarten de Rijke}.}
  \bibinfo{year}{2020}\natexlab{b}.
\newblock \showarticletitle{When Inverse Propensity Scoring does not Work:
  Affine Corrections for Unbiased Learning to Rank}. In
  \bibinfo{booktitle}{\emph{Proceedings of the 28th ACM International
  Conference on Information and Knowledge Management}}.
\newblock


\bibitem[Wang et~al\mbox{.}(2021)]%
        {wang2021non}
\bibfield{author}{\bibinfo{person}{Nan Wang}, \bibinfo{person}{Zhen Qin},
  \bibinfo{person}{Xuanhui Wang}, {and} \bibinfo{person}{Hongning Wang}.}
  \bibinfo{year}{2021}\natexlab{}.
\newblock \showarticletitle{Non-Clicks Mean Irrelevant? Propensity Ratio
  Scoring As a Correction}. In \bibinfo{booktitle}{\emph{Proceedings of the
  14th ACM International Conference on Web Search and Data Mining}}.
  \bibinfo{pages}{481--489}.
\newblock


\bibitem[Wang et~al\mbox{.}(2016)]%
        {wang2016learning}
\bibfield{author}{\bibinfo{person}{Xuanhui Wang}, \bibinfo{person}{Michael
  Bendersky}, \bibinfo{person}{Donald Metzler}, {and} \bibinfo{person}{Marc
  Najork}.} \bibinfo{year}{2016}\natexlab{}.
\newblock \showarticletitle{Learning to Rank with Selection Bias in Personal
  Search}. In \bibinfo{booktitle}{\emph{Proceedings of the 39th International
  ACM SIGIR conference on Research and Development in Information Retrieval}}.
  ACM, \bibinfo{pages}{115--124}.
\newblock


\bibitem[Wang et~al\mbox{.}(2018a)]%
        {wang2018position}
\bibfield{author}{\bibinfo{person}{Xuanhui Wang}, \bibinfo{person}{Nadav
  Golbandi}, \bibinfo{person}{Michael Bendersky}, \bibinfo{person}{Donald
  Metzler}, {and} \bibinfo{person}{Marc Najork}.}
  \bibinfo{year}{2018}\natexlab{a}.
\newblock \showarticletitle{Position Bias Estimation for Unbiased Learning to
  Rank in Personal Search}. In \bibinfo{booktitle}{\emph{Proceedings of the
  Eleventh ACM International Conference on Web Search and Data Mining}}. ACM,
  \bibinfo{pages}{610--618}.
\newblock


\bibitem[Wang et~al\mbox{.}(2018b)]%
        {wang2018lambdaloss}
\bibfield{author}{\bibinfo{person}{Xuanhui Wang}, \bibinfo{person}{Cheng Li},
  \bibinfo{person}{Nadav Golbandi}, \bibinfo{person}{Michael Bendersky}, {and}
  \bibinfo{person}{Marc Najork}.} \bibinfo{year}{2018}\natexlab{b}.
\newblock \showarticletitle{The LambdaLoss Framework for Ranking Metric
  Optimization}. In \bibinfo{booktitle}{\emph{Proceedings of the 27th ACM
  International Conference on Information and Knowledge Management}}. ACM,
  \bibinfo{pages}{1313--1322}.
\newblock


\bibitem[Williams(1992)]%
        {williams1992simple}
\bibfield{author}{\bibinfo{person}{Ronald~J Williams}.}
  \bibinfo{year}{1992}\natexlab{}.
\newblock \showarticletitle{Simple Statistical Gradient-Following Algorithms
  for Connectionist Reinforcement Learning}.
\newblock \bibinfo{journal}{\emph{Machine Learning}} \bibinfo{volume}{8},
  \bibinfo{number}{3-4} (\bibinfo{year}{1992}), \bibinfo{pages}{229--256}.
\newblock


\bibitem[Yan et~al\mbox{.}(2022)]%
        {yan2022twotowers}
\bibfield{author}{\bibinfo{person}{Le Yan}, \bibinfo{person}{Zhen Qin},
  \bibinfo{person}{Honglei Zhuang}, \bibinfo{person}{Xuanhui Wang},
  \bibinfo{person}{Mike Bendersky}, {and} \bibinfo{person}{Marc Najork}.}
  \bibinfo{year}{2022}\natexlab{}.
\newblock \showarticletitle{Revisiting Two Tower Models for Unbiased Learning
  to Rank}. In \bibinfo{booktitle}{\emph{Proceedings of the 45th International
  ACM SIGIR Conference on Research and Development in Information Retrieval}}
  (Madrid, Spain) \emph{(\bibinfo{series}{SIGIR '22})}.
  \bibinfo{publisher}{ACM}.
\newblock


\bibitem[Yuan et~al\mbox{.}(2020)]%
        {yuan2020unbiased}
\bibfield{author}{\bibinfo{person}{Bowen Yuan}, \bibinfo{person}{Yaxu Liu},
  \bibinfo{person}{Jui-Yang Hsia}, \bibinfo{person}{Zhenhua Dong}, {and}
  \bibinfo{person}{Chih-Jen Lin}.} \bibinfo{year}{2020}\natexlab{}.
\newblock \showarticletitle{Unbiased Ad Click Prediction for Position-Aware
  Advertising Systems}. In \bibinfo{booktitle}{\emph{Fourteenth ACM Conference
  on Recommender Systems}}. \bibinfo{pages}{368--377}.
\newblock


\bibitem[Yue and Joachims(2009)]%
        {yue2009interactively}
\bibfield{author}{\bibinfo{person}{Yisong Yue} {and} \bibinfo{person}{Thorsten
  Joachims}.} \bibinfo{year}{2009}\natexlab{}.
\newblock \showarticletitle{Interactively Optimizing Information Retrieval
  Systems as a Dueling Bandits Problem}. In
  \bibinfo{booktitle}{\emph{Proceedings of the 26th Annual International
  Conference on Machine Learning}}. ACM, \bibinfo{pages}{1201--1208}.
\newblock


\bibitem[Zhuang et~al\mbox{.}(2021)]%
        {zhuang2021cross}
\bibfield{author}{\bibinfo{person}{Honglei Zhuang}, \bibinfo{person}{Zhen Qin},
  \bibinfo{person}{Xuanhui Wang}, \bibinfo{person}{Michael Bendersky},
  \bibinfo{person}{Xinyu Qian}, \bibinfo{person}{Po Hu}, {and}
  \bibinfo{person}{Dan~Chary Chen}.} \bibinfo{year}{2021}\natexlab{}.
\newblock \showarticletitle{Cross-Positional Attention for Debiasing Clicks}.
  In \bibinfo{booktitle}{\emph{Proceedings of the Web Conference 2021}}.
  \bibinfo{pages}{788--797}.
\newblock


\bibitem[Zhuang and Zuccon(2020)]%
        {zhuang2020counterfactual}
\bibfield{author}{\bibinfo{person}{Shengyao Zhuang} {and}
  \bibinfo{person}{Guido Zuccon}.} \bibinfo{year}{2020}\natexlab{}.
\newblock \showarticletitle{Counterfactual Online Learning to Rank}. In
  \bibinfo{booktitle}{\emph{European Conference on Information Retrieval}}.
  Springer, \bibinfo{pages}{415--430}.
\newblock


\end{thebibliography}

\end{document}